\documentclass{svjour3}                     
\smartqed  
\usepackage{natbib}
\usepackage{dsfont}
\usepackage{amsmath}
\usepackage{amssymb}
\usepackage{longtable}
\usepackage{tikz}
\usepackage{pgflibraryshapes}
\usepackage{pgflibrarysnakes}
\usepackage{multirow}
\usepackage{subfigure}
\usepackage{proof}
\usepackage{floatfig}
\usetikzlibrary{snakes}
\usetikzlibrary{patterns}
\usepackage{graphicx}
\def \stub {\mathit{stub}}
\newcommand{\flex}[1] {\mathit{flex}^{#1}}
\newcommand{\trad}[2]{\tau_{#1,#2}}
\def \agset {\mathrm{Ag}}
\def \sign {\Sigma}
\newcommand{\struct}[1]{\mathcal{#1}}
\newcommand{\agstruct}[2]{\struct{A}_{#1}^{#2}}
\newcommand{\dstr}[2]{\struct{D}_{#1}^{#2}}
\newcommand{\istr}[2]{\struct{I}_{#1}^{#2}}
\newcommand{\agstructdef}[2]{\agstruct{#1}{#2}=\langle\dstr{#1}{#2},\istr{#1}{#2}\rangle}
\newcommand{\negstruct}[1]{\mathcal{N}^{#1}}
\def \P {\mathcal{P}}
\def \L {\mathcal{L}}
\newcommand{\LS}[1] {\L_{\mathcal{S}_{#1}}}
\newcommand{\LF}[1] {\L_{\mathcal{F}_{#1}}}

\newcommand{\MND}{\mathit{MND}}

\newcommand{\fix}[2]{} 

\journalname{}

\begin{document}

\title{Meaning Negotiation as Inference\thanks{The work presented in this paper
was partially supported by the FP7-ICT-2007-1 Project no.~216471,
``AVANTSSAR: Automated Validation of Trust and Security of
Service-oriented Architectures".}
}


\author{Elisa Burato         \and
        Matteo Cristani      \and
        Luca Vigan\`{o}
}


\institute{Dipartimento di Informatica, Universit\`a di Verona \at
              Strada Le Grazie 15, I-37134 Italy \\
              \email{elisa.burato@univr.it\,; matteo.cristani@univr.it\,; luca.vigano@univr.it}            
}

\date{Received: date / Accepted: date}

\maketitle

\begin{abstract}
Meaning negotiation (MN) is the general process with which agents reach
an agreement about the meaning of a set of terms. Artificial Intelligence scholars have dealt with the problem of MN by means of argumentations schemes, beliefs merging and information fusion operators, and ontology alignment but the proposed approaches depend upon the number of participants. In this paper, we give a general
model of MN for an arbitrary number of agents, in which
each participant discusses with
the others her viewpoint by exhibiting it in an actual set of
constraints on the meaning of the negotiated terms. We call this
presentation of individual viewpoints an angle. The agents do
not aim at forming a common viewpoint but, instead, at agreeing about an
acceptable common angle.
We analyze separately the process of MN by two agents (\emph{bilateral} or \emph{pairwise} MN) and by more than two agents (\emph{multiparty} MN), and we use game theoretic models to understand how the process develops in both cases: the models are Bargaining Game for bilateral MN and English Auction for multiparty MN.
We formalize the process of reaching such an agreement by giving a
deduction system that comprises of rules that are
consistent and adequate for representing MN.
\end{abstract}

\section{Introduction and Motivations}
In recent years, it has become clear that computer systems do not work
in isolation. Rather, computer systems are increasingly acting as
elements in a complex, distributed community of people and systems,
which, in order to fulfill their tasks, must cooperate, coordinate their
activities and communicate with each other. In fact, cooperation and
coordination are needed almost everywhere computers are used. Relevant
examples include health institutions, electricity networks, electronic
commerce, robotic systems, digital libraries, military units etc.

Problems of coordination and cooperation are not a novelty due to the
birth of automated systems. They exist at multiple levels of activity in
a wide range of human agents as well. People achieve their own goals
through communication and cooperation with other people; and, in
industrial systems, with machines as well.

The main difficulty in agent cooperation and communication is to
understand each other. People and, in general, \textit{intelligent
agents} come from different organizations and individuals and thus they
have different backgrounds and, maybe, different expression languages.
However, natural agents get to agreements as a means for solving
conflicts. Consequently, artificial agents, in order to be reasonably
similar to human agents, to an extent that grants their usefulness, are
to be designed as agents that discuss to reach an agreement by starting
from distinct viewpoints.


Intelligent agents have been considered in a wide number of reasons and
applications, that is in all the situations in which people can delegate
their interests to somebody else. In fact the word \emph{intelligent}
refers to the ability to behave, to reason and to perceive situations
and the environment the agents are in like humans do. In all the
applications of intelligent agents, a basic mechanism of agreement is
required: information agents, electronic commerce agents, agents in
e-learning systems and legal reasoning have to know the meaning of all
the information they receive from others.\fix{Luca}{Why do we speak of
the user? I would avoid speaking of user (and user of what?)} In all the
situations in which a misunderstanding arises, the system does not work
as the user's expectations and it produces wrong outcomes.

To achieve an agreement there are fundamentally four possibilities:
\begin{itemize}
\item \emph{delegate decision}: consisting in the choice of an external
agent that decides for all the other agents involved;
\item \emph{judge decision}: consisting in the choice of an external
agent that chooses among the proposals of the agents;
\item \emph{merging}: consisting in generating a new theory starting
from the merged ones;
\item \emph{negotiation}: consisting in a sequence of actions aiming at
the definition of a novel, shared position, emerged from the discussion
itself, by means of a group of mechanisms, established as rules of the
negotiation.
\end{itemize}

In general, negotiation is a dialogue between two or more agents by
which they try to reach an agreement about something starting from
different viewpoints about the shared object. A negotiation process is
\emph{quantitative} when the agents discuss about how to share a set of
countable objects, whereas in \emph{Meaning Negotiation}, on which we
focus in this paper, the proposals are pieces of knowledge represented
by terms, i.e. the expressions of what an agent knows about the
negotiated terms. These pieces of knowledge may be accepted or rejected
by the other discussants.

More specifically, Meaning Negotiation (henceforth MN) is a negotiation
process in which the sharing object is the meaning of a set of terms. A common {Luca}{Why ``acceptable''? A definition is a definition.
Maybe replace with ``common'', ``standard'',...} definition of MN is:
the process that takes place when the involved agents have some
knowledge (some data or information) to share but do not agree on what
knowledge agents share and (possibly) how they reach an agreement about
it.

In this paper, we focus on the processes that take place when the agents
who negotiate agree about the mechanisms to reach an agreement and
disagree about the meaning of the negotiated terms. In this case, agents
\emph{know} how to reach an agreement, and see the mechanisms themselves
as a protocol. We aim at designing a model of an inference engine whose
derivations are indeed agreement processes. This approach views meaning
negotiation as an inference process.

To clarify what a negotiation process indeed is, let us introduce an
example that we will also employ in the rest of the paper as a running
example.

\begin{example}\label{es:1}
Consider two agents Alice and Bob that the negotiate the meaning of the
term ``vehicle''. Suppose that Alice thinks that a vehicle always has
two, three, four or six wheels; a handlebar or a steering wheel; a
motor, or two or four bicycle pedals, or a tow bar. On the other hand,
Bob thinks that it always has two, three or four wheels; a handlebar or
a steering wheel; a motor, or two or four bicycle pedals. Alice and Bob
are in disagreement because Bob does not know if a vehicle has a tow bar
or not. \hfill$\Box$
\end{example}

In this example, the MN depends upon the relevance of the terms the
agents use. Alice and Bob define ``vehicle'' in different ways and with
different terms. In fact, Alice uses ``tow bar'' and Bob does not. Bob
does not say anything about the tow bar maybe he does not know what a
tow bar is, or maybe he does not consider as relevant the properties
about the tow bar. In this paper, we assume that the agents make
assertions only about the properties they consider as relevant.

One of the main parameters of the MN is the number of the involved
agents. In Section~\ref{blmn}, we discuss in detail how a negotiation
process takes place between two parties, whilst in Section~\ref{mpmn} we
discuss the situation arising when the number of involved partners is
higher than two.

Before\fix{Luca}{This paragraph should be rewritten, as it is quite
confused. In particular the sentence ``the notion of term, and
definitions, in particular, of term, is part of the structure itself of
the Logic'' which I don't really understand} we do so, we further
exemplify the specific problems of the definition of an inference
engine, by considering a case taken from a common situation. We consider
the Description Logic framework, where the notion of term, and
definitions, in particular, of term, is part of the structure itself of
the Logic. More specifically, in description logic, the
\emph{acceptability} of a concept is tested by the subsumption relation.
Having two concepts $A$ and $B$, $A$ subsumes $B$ ($A \sqsubseteq B$)
when the definition for $B$ is also a definition for $A$ but not vice
versa. Two concepts are equivalently defined when $A \sqsubseteq B$ and
$B \sqsubseteq A$. Suppose that Alice considers $A$ and $B$, where $A
\sqsubseteq B$, as two plausible definitions for a concept $X$. Alice
can accept a new definition $C$ for $X$ iff $A \sqsubseteq C \sqsubseteq
B$ because Alice has a pair of concepts $(A,B)$ describing the same
things that $C$ describes.

In the same way of the above mentioned Description Logic Framework, a
logical formula $\varphi$ is acceptable with respect to an agent when
she shares the interpretation of all its terms. Therefore Alice always
accepts $\varphi$ when it is equivalent to her current
angle\fix{Luca}{We cannot talk about angle here, as it is not defined
yet. All this should come later (or the definition anticipated, but I
don't advise doing so)} $\alpha$, that is $\alpha \leftrightarrow
\varphi$ because $\mathcal{I}(\alpha) = \mathcal{I}(\varphi)$ where
$\mathcal{I}$ is an interpretation function, but Alice will accept
$\varphi$ also when she has a pair of feasible angles $\alpha$ and
$\beta$ such that $\alpha \rightarrow \varphi \rightarrow \beta$ because
$\mathcal{I}(\alpha) \subseteq \mathcal{I}(\varphi) \subseteq
\mathcal{I}(\beta)$. In the last case Alice shares the interpretation of
all the terms in $\varphi$ by means of two angles $\alpha$ and $\beta$
that are partial representations of her knowledge she considers as
acceptable.

An agent always accepts a logical theory when it is equivalent to her
own one, i.e. when the two theories have the same set of semantical
models. The logical equivalence is always a condition of agreement but
there are situations in which it is not a necessary condition and some
weaker conditions are sufficient to claim the agreement. In fact, a
logical theory may be considered acceptable when it is a good compromise
for an agent, i.e the logical theory expresses an acceptable part of
what the agent wants to express. In this sense, the set of semantical
models of the proposal is not equal to the set of semantical models of
the theory of the agent, but the agent has a superset and a subset of
models bounding it.

Conclusively, the participants to a discussion may disagree, in fact, in three different ways:
\begin{itemize}
\item The properties used to define the terms are inconsistent and
contradictory.
\item The relevant properties for an agent are more/less than those
expressed by one or more other agents.
\item Some agents do not know the properties used by someone in the
multiple agent system.
\end{itemize}

The idea of our framework is that the knowledge of an agent represents
her viewpoint about the world and in order to negotiate with the other
agents, an agent possibly has a set of acceptable portions of her
knowledge that she may consider as good compromises with respect to her
viewpoint.

We call \emph{angle} any partial representation of a viewpoint. The
knowledge of a negotiation is built by a single viewpoint and many
angles, i.e. many partial representations of it. Moreover, in this paper
we assume that angles are presented as \emph{logical theories}, and in
particular \emph{propositional} ones. At the beginning of a MN process,
agents are in disagreement, i.e. they have mutually inconsistent
knowledge. By MN, they try to reach a common angle representing a shared
acceptable knowledge, where the MN ends in positive way when the agents
have a common knowledge, and it ends in negative way otherwise: agents
are in \emph{agreement} when they have found a set of constraints on the
meaning of the negotiated terms that is accepted by both agents (this
new theory is named, here, a \emph{common angle}); \emph{disagreement}
when they are not in agreement. To negotiate the meaning of a set of
terms means to propose definitions, properties, typical memberships of
the terms' definitions, and/or to accept or to reject definitions.

\subsection{Bilateral Meaning Negotiation}
\label{blmn}

When the negotiation involves two agents, each proposal has one sender
and one receiver. Game Theory scholars have dealt with several bilateral
negotiation protocols like \emph{Divorce}~\citep{wurman01},
\emph{Pleadings}~\citep{gordon93} and the \emph{Bargaining
Game}~\citep{kambe}. When the negotiation is bilateral the agents are
called buyer and seller (which is typically the first proposing agent).
Both the buyer and the seller have the same feasible actions: they make
a proposal or accept or reject an offer.

As a matter of fact, it is the Bargaining game that offers the most
natural framework for meaning negotiation, due to the need for
negotiators to avoid the meaning of terms to be negotiated to be a
\emph{compromise} between the definitions preferred by the two agents,
that implies that the negotiators play by moving themselves to the
other's viewpoint with the maximally possible flexibility.\fix{Luca}{Inglese incerto ``moving themselves to the
other's viewpoint'' e anche il resto della frase}

In the Bargaining Game, two agents have to share, say, one dollar and do
this by each making a proposal. If the sum of their demands is less than
one, they share the dollar, otherwise they have to make a new demand.
The Bargaining Game is built by two stages:
\begin{itemize}
\item \emph{Demand stage}: agents make a proposal and if the proposals
are compatible, the negotiation ends in positive way; otherwise the
second stage begins.
\item \emph{War of attrition}: agents have incompatible viewpoints and
perform new demands. If the demands are compatible, the process ends
positively, otherwise they make new ones.
\end{itemize}
In the Bargaining Game, players have a \textit{negotiation power} that
represents how often an agent cedes during the negotiation and how much
she resists about her current angle. The negotiation power of an agent
is captured by a set of partially ordered angles of her viewpoint. The
partial order among the angles allows an agent to choose the next
proposal to perform, and to evaluate the acceptability of the received
offers. Moreover, the set of partially ordered angles has a minimum that
identifies the last offer an agent proposes in a negotiation. We say
that each agent has 
\begin{itemize}
\item one single \emph{stubborn} and 
\item many \emph{flexible} angles
\end{itemize} 
that are respectively the limit proposal (i.e., the last offer) and the
acceptable ones, where each flexible angle is consistent with the
stubborn knowledge.

The Meaning Negotiation process ends in a positive way (agreement) when
both agents agree about a common definition of the meaning of the set of
terms, i.e. they propose the same thing, or in a negative way
(disagreement) when they are not in agreement and they have no more
proposal to perform.

\subsection{Multiple parties' Meaning Negotiation}
\label{mpmn}

When the number of agents is more than two, the negotiation is
multiparty and each proposal has one sender and many receivers. A
proposal may be accepted or rejected by all the agents or by only some
of them, and the receivers may answer in different ways. The negotiation
process for multiparty scenarios is computationally harder than the
bilateral one and it needs the organization of the order in which agents
make assertions (proposals, acceptance or rejection of offers) during
the process. The modeling of the Meaning Negotiation in this case
depends also on the role of the involved agents. Having $n+1$ agents in
the negotiation, the possible role distinctions are:
\begin{itemize}
\item $1$-$n$: one seller and many buyers;
\item $n$-$1$: many sellers and one buyer;
\item $n_1$-$n_2$: many sellers and many buyers.
\end{itemize}
In the first case, the agents behave like in an auction. Before entering the auction, the seller establishes a maximal price for the item. The seller begins the game by making the initial request that is the \emph{reservation price}. The auction develops by \emph{beats}.
A beat\index{beat} consists of:
\begin{enumerate}
\item the seller makes a request;
\item each buyer proposes a counteroffer or accepts the seller's proposal.
\end{enumerate}
No more beat begins if the maximal price is reached or if the buying
agents do not make new proposals. In an auction scenario, a proposal is
also called a \emph{bid}. The end of the auction is established by the
seller, i.e. by the auctioneer. In general, in an auction there is only
one winner, i.e. only one agent buys the item in the auction.

The second case, $n$-$1$, is similar to the first one. The sellers have
to convince the buyer to accept the price they propose and when the buyer is not convinced she has to respond with another offer.\fix{Luca}{``and the buyer make her one''?} In the Meaning
Negotiation perspective, a buyer is not different from a seller because
they have the same feasible action: accept an offer, reject an offer or
make a proposal. Even if the agents generally have different strategies
depending on their role, i.e. typically a buyer enhances instead the
seller fall\fix{Luca}{``the seller falls the last offer''? ``falls''
\`e la parola giusta?} the last offer, the purpose of the seller and of
the buyer is the same: to meet the opponent's request. Therefore, buyers
and sellers make new proposals in the same way, that is by \emph{ceding}
their last one.

The third case, $n_1$-$n_2$, is called \emph{fish market}. For a
reference on the different ways to perform negotiation,
see~\cite{lomuscio03}. It is not possible to make a modeling of the MN
of this multiple-agent system structure because there is no agent
monitoring the process and no behavioral guidelines for the players. In
the first two cases, the auctioneers, the seller in the first and the
buyer in the second, are the agents who control the Meaning Negotiation
process and check whether an agreement is reached between the involved
agents. As in the auction game, in the fish market each agent makes a
proposal or accepts/rejects the opponents' one but there is no
coordination among the agents. It may be the case that two or more
agents make proposals simultaneously so each agent is a buyer, i.e. she
makes an offer, and a seller, i.e. she evaluates the received offers, at
the same time. The result is that a common proposal is difficult to
find. In the worst case, where there are $n$ agents involved in
negotiation in total, this means there can be up to $n(n-1)/2$
negotiation threads. Clearly, from an analysis point of view, this makes
such negotiation hard to handle.

In this paper, as in the main approach in the current Artificial
Intelligence literature, we model the multiparty MN by reserving an
agent, typically the first bidding one, to be the referee of the process
and the game used to represent it is the auction. In Game Theory, there
are several auction types~\citep{benameur02}: English, Dutch, Vickrey,
First-price sealed-bid etc. The types of auction differ on the behavior
of the agents involved and on the number of the proposals the agents
make. In this paper we use the English Auction because the agents behave
as in the Bargaining Game. The English Auction Game begins by the
proposal of the auctioneer that is called \emph{reservation price} and
it is the minimum price the agents have to pay to win the auction. In
the next step of the English Auction, each player makes her offer by
incrementing the last bidden one, i.e. the auctioneer's proposal. There
is not a fixed number of turns for agents' bidding, instead the game
continues until no more bids are performed. The game ends with a winner
that is the agent who bids the highest offer.

In a MN perspective, the English Auction game is slightly different in
the outcome. The goal of the negotiation is agents in sharing a
viewpoint. Therefore the positive ending condition of the game is that
all the agents make the same bid and the bidden proposal is the
representation of their viewpoints.

There are MN contexts in which it is sufficient to have a ``major'' part
of agreeing agents to consider positive the negotiation. In general
``major part'' means that a number of agents, typically more then 50\% ,
but it may mean that a part of the most trustworthy agents are in
agreement. In the former case, the minimum number of agreeing agents is
a parameter of the game: suppose $\alpha$ is the chosen number for
``major part'', the MN continues until at least $\alpha$ agents agree
about a common angle. The minimum number of agreeing agents is called
\emph{degree of sharing}. A MN process for more than two agents, say $n$
agents, has two positive ending conditions and two types of positive
outcomes, if a positive outcome exists:
\begin{description}
\item[\emph{partially positive}]: when the degree of sharing is less than the number of the participants ($\alpha< n$);
\item[\emph{totally positive}]: when the degree of sharing is equal to the number of the negotiating agents ($\alpha = n$).
\end{description}
The latter case prevails when there are specialists about the
negotiation subject into the multiple agent system and their opinion is
more relevant then the opponents' ones. When participating to a
negotiation process, the agents assume a viewpoint and many admissible
angles of it. A specialist knows more about the negotiation subject than
a less expert agent and her negotiation behavior will be to make
concessions as few as possible. Conversely, if a no expert agent knows
that an agent in the MAS is a specialist, then she trusts the specialist
and probably makes concessions with respect to the proposals of the
specialist. The degree of knowledge of an agent translates into the
trustworthiness with respect to herself. In this paper, the
trustworthiness of the agent is not specifically considered because it
is out of the scope of the paper and it is left as a future work.

The role of the auctioneer is to monitor the game in order to understand
when it ends and whether in a positive or a negative way. In general,
the auctioneer is the first bidding agent but in a negotiation
perspective she may play in two ways: active or passive. An active
referee is a participant of the negotiation and the reservation price is
her viewpoint. Moreover, an active referee makes herself proposals
during the auction as all the other agents and she is considered in the
agreement test. A passive auctioneer does not affect the negotiation.
She only tests the process and makes only one bid, the first one for the
reservation price.

\subsection{Aims of proposed approach}

The aim of the paper is to give a general model to represent the process
of MN by means of a deduction system. Our formalization is based upon
Game Theory notions of behavior of the agents during a
negotiation/litigation. The negotiation process has already been dealt
in terms of games but, to the best of our knowledge, only quantitative
negotiation were studied. MN is not quantitative thus one of the main
problem in dealing with it is the identification of the agreement and
disagreement situations, i.e. the mutual evaluation of the proposals of
the players. The purpose of the paper is to extend the current
literature with the formalization of the MN problem by means of a
deduction system that is independent of the number of the involved
agents and of their expression languages. Our work begins with the study
of the representation of the knowledge of the agents in a MN, and in
particular the representation of the properties the agents consider as
necessary and unforgivable in defining the meaning of the set of terms
they are negotiating and, vice versa, which are the facultative ones,
because these properties identify the negotiation space between agents.
We call the first one the stubborn knowledge of the agent and flexible
knowledge the second one.

The \textit{first contribution} of the paper is the definition of the meaning negotiating agent in terms of her stubborn and flexible knowledge.

The \textit{second contribution} of the paper is the study of the agreement and disagreement situations between the agents and the definition of the different ways in which they may be in disagreement (absolute, relative, essence and compatibility).

As said above, one important issue in MN is the evaluation of a received proposal. Agents make proposal and evaluate the opponents' one. The evaluation mechanism is not trivial when the negotiation is not quantitative. When is one definition of a set of terms better than another one? When are two or more definitions equivalent? Here, we study how a proposal is evaluated with respect to the knowledge of an agent. In our model, the types of
disagreement depend upon the relation among the proposal $p$ and the
stubborn and the flexible knowledge of the agent $i$
who receives and evaluates $p$:
\begin{itemize}
\item \emph{Call-away} occurs when $p$ is a generalization\footnote{A theory $A$ is a \emph{generalization} of a theory $B$ when the models of $A$ are a superset of the models of $B$.} of the
stubborn knowledge of $i$, thus it would correspond to dropping out some unquestionable knowledge.
\item \emph{Absolute disagreement} occurs when the stubborn knowledge of
$i$ is inconsistent with respect to $p$.
\item \emph{Essence disagreement} occurs when the flexible knowledge of
$i$ is inconsistent with respect to $p$.
\item \emph{Compatibility} occurs when $p$ is consistent with the
flexible knowledge of $i$ but it is not a generalization or a
restriction of $i$'s viewpoint.
\item \emph{Relative disagreement} occurs when $p$ is a generalization
of the flexible knowledge of $i$.
\end{itemize}
The call-away situations arise when an agent does not accept all the
necessary requests of the other one and thus exits the MN so that the MN
ends negatively.

An important point in MN as well as of the Multiple Agent Systems (MAS), is the strategical component in the definition of negotiating agent. In this paper we do not give any definition of strategy of agents but we assume that whenever an agent has to choose the next move, she has a way to do it. In general, in MAS literature there are two main ways in which the agents behave: collaborative and competitive. A collaborative agent always chooses the move that improves the welfare of the MAS she is in, whereas a competitive agent moves in order to achieve her goals and, possibly, to prevent the other ones. The study of the strategies in MN process needs the definition of MAS welfare and goals, and of the attainment of a goal.

The rest of the paper is organized as follows: in
Section~\ref{sec:formalization} we formalize the negotiating agents in
terms of their knowledge and language, Section~\ref{sec:process} defines
the agreement and disagreement relations between agents and gives the
deduction rules for bilateral and multiparty MN and
Section~\ref{sec:relWork} discusses the current approaches of Artificial
Intelligence community for MN. The paper ends with the summary of the
contributions of our work and with a discussion of future work
(Section~\ref{sec:conclusions}).

\section{A Formalization of Negotiating Agents}\label{sec:formalization}
We consider here a general MN process, so we abstract away from the
particular terms whose meaning the agents are negotiating. We first
consider the knowledge of negotiating agents
(Section~\ref{sec:knowledge}), i.e. what agents know about the meaning
of the set of terms they are negotiating, and then their language
(Section~\ref{sec:language}), i.e. how they represent their knowledge
and how they make proposals during the MN.

\subsection{The Knowledge of Negotiating Agents}\label{sec:knowledge}
When agents give the definition of a concept, they:
\begin{itemize}
\item give the necessary (\emph{stubborn}) properties and the
characterizing (\emph{flexible}) ones;
\item give the properties that necessarily have not to hold and the ones
that plausibly (flexibly) have not to hold; and
\item give the formulas asserting what has not (stubbornly), or may not
(flexibly), be used in the definition. 
\end{itemize}

The notion of relevance of a formula is interesting at this stage
of the definition, but instead of introducing a novel
operator, 
we simply consider a formula as not relevant to an agent if she does not assert it.
When $i$ \emph{asserts} a formula $\varphi$, she has a way to evaluate
it: she thinks $\varphi$ as positive or negative. If $i$ does
not assert $\varphi$ then either
$i$ does not know $\varphi$, i.e., she is not able to evaluate it
or $i$ does not think $\varphi$ is relevant in defining the
negotiated meaning.
So, we assume that \emph{whenever $i$ thinks $\varphi$ as not relevant for the negotiation, $i$ never asserts $\varphi$ during the negotiation}.

\begin{example}\label{es:2}
As in Example~\ref{es:1}, consider the definition of the term
``vehicle''. Alice (stubbornly) thinks that it always has two, three,
four or six wheels; a handlebar or a steering wheel; a motor, or two or
four bicycle pedals, or a tow bar. Moreover, Alice (flexibly) thinks
that a ``vehicle'' may be defined only as a car, then having four
wheels, a steering wheel, and a motor; or only as a bicycle, then having
two wheels, a handlebar and two bicycle pedals.
In other words, Alice has two acceptable ways to define a vehicle
(namely, a car or a bicycle as particular ``vehicles'') but she has only
one general description of a ``vehicle''. \
\hfill$\Box$
\end{example}

The necessary and the characterizing properties of a concept definition
are closely related to \textit{EGG/YOLK} objects, introduced by
Lehmann and Cohn (\citeyear{lehmann94}) 
as a way to represent class
membership based on typicality of the members: the egg is the set of the
class members and the yolk is the set of the \emph{typical} ones. For
instance, the class of ``employees'' of a company $A$ may be defined as
``the set of people that receive money from the company in exchange for
carrying out the instructions of a person who is an employee of that
company'', thus excluding, e.g., the head of the company (who
has no boss), and the typical employee would include regular workers
like secretaries and foremen. Another company $B$ might have a different
definition, e.g., including the head of the company, resulting in a
mismatch. Nevertheless, if both companies provide some typical examples
of ``employees'' it is possible that all of $A$'s typical employees fit
$B$'s definition, and all of $B$'s typical employees fit $A$'s
definition: $\mathit{YOLK}_B \leq \mathit{EGG}_A$ and $\mathit{YOLK}_A
\leq \mathit{EGG}_B$, in the terminology of \citep{lehmann94}.

In this paper, we use the same idea to express that negotiating
agents have a \emph{preference} over their knowledge: the
properties an agent thinks as necessary are the typical ones,
and the characterizing properties are those that are not typical but
plausible.
We focus on the models of the knowledge of an agent. The stubborn
properties of a concept definition are the most acceptable ones,
therefore they thus have more elements satisfying them than the flexible
properties have. Hence, we represent the elements satisfying the
stubborn properties in the egg and those satisfying the flexible ones in
the yolk.
Differently from the original model, concept definitions are here
restricted by stubborn properties to the largest acceptable set of
models, hence represented by the egg, whilst the yolk is employed to
denote the most restricted knowledge, that is, the one on which the
agents are flexible.

The stubborn properties never change during the negotiation; therefore,
the egg is fixed at the beginning of the MN. Instead, the flexible part
of the definition of a concept is the core of the proposal of a
negotiating agent. Each proposal differs from the further ones in two
possible ways: it may give a definition of the negotiated object that is
more descriptive than the next ones, or the given definition specifies
properties that the next ones do not and vice versa. In the former case,
we say that the agent carries out a \emph{weakening action}, in the
latter the agent carries out a \emph{changing theory action}. In this paper, we do not consider how and why an agent chooses the next action to perform, but a general approach for dealing with agency in \emph{multiple agent system (MAS)} is based on the representation of the choice of the action to perform by
\emph{attitudes}, which ``are driving forces behind the actions of
agents''~\citep{meyer99}. In other words, attitudes are \emph{the
representation of the reasons that guide the agents in their behavior}.
They are preferences between the criteria used to evaluate the feasible
actions. 
%
In general, the main criteria for evaluating an action are:
\begin{enumerate}
\item The MAS welfare: is the action positive for all the agents in the MAS?
\item The personal advantage: is the action individually positive for the agent in choosing an action?
\end{enumerate}
%
By attitude, we mean the \emph{preference order of the evaluation criteria}.
Following the enumeration in the list above, the main attitudes in agency are:
\begin{itemize}
\item\textit{collaborative}: the main goal of the agent is the welfare of the MAS: $1$ is preferred to $2$;
\item\textit{competitive}: the action performed by a competitive agent are advantageous or not damaging herself: $2$ is preferred to $1$.
\end{itemize}
In a MN perspective, a collaborative agent aims at ending the process as soon as possible, whilst a competitive agent tends to stay as close as possible to her initial viewpoint.
The collaborative and the competitive attitudes are dual.
%

However, none of the weakening or changing theory actions can be carried out with respect to a proposal if
the proposal describes the necessary properties of the object in the MN.
We say that in such a situation the agents always make a
\emph{stubbornness action} that is equivalent to \emph{no more change}.

\subsection{The Language of Negotiating Agents}\label{sec:language}

Each agent $i$ is represented by her language $\L_i$, which is composed
of two disjoint sublanguages (where we intend, with ``language'', the
set of well-formed formulae of a logical language):
\begin{itemize}
\item a \emph{stubbornness} language containing the properties $i$ deems
as necessary in defining the negotiated meaning and
\item a \emph{flexible} language containing the properties $i$ deems as
not necessary in the MN.
\end{itemize}

\begin{definition}[$\sign_i$ and $\L_i$]\label{agSignature}
Consider an abstract set of terms and let $\agset$ be the set of the
\emph{negotiating agents}. The \emph{signature} $\sign_i$ of an agent
$i\in\agset$ is the pair $\langle \mathcal{P}_i,\alpha_i\rangle$ where
\begin{itemize}
\item $\P_i$ is the set of the predicate symbols;
\item $\alpha_i:\mathcal{P}_i\rightarrow \mathbb{N}$ is the arity
function for predicate symbols.
\end{itemize}
The \emph{language} $\L_i$ of $i\in\agset$ comprises of
$\sign_i$-formulas defined inductively as follows:
\begin{itemize}
\item If $P \in \P_i$, $\alpha_i(P)=n$ and $t_1,\dots,t_n$ are terms
then $P(t_1,\dots,t_n)$ is a $\sign_i$-formula.
\item If $\varphi$ and $\psi$ are $\sign_i$-formulas then $\neg
\varphi$, $\varphi\wedge \psi$, $\varphi \vee \psi$, and $\varphi
\rightarrow \psi$ are $\sign_i$-formulas.
\end{itemize}
\end{definition}


\begin{definition}[Stubbornness and Flexibility of an agent]
The agent $i$ considers the formulas in $\L_i$ in two ways:
\emph{stubborn} or \emph{flexible}. That is, the language $\L_i$ is
divided in two disjoint sets: $\L_i=\LS{i}\cup\LF{i}$, where
\begin{description}
\item[$\LS{i}$] is the set of stubborn formulas;
\item[$\LF{i}$] is the set of flexible formulas.
\end{description}
We further define
$$
\stub_i = \bigwedge_{\varphi \in \LS{i}}\varphi
$$
and 
$$
\flex{}_i = \bigwedge_{\varphi \in \LF{i}}\varphi
$$
\end{definition}

\begin{table}[t]\centering
\begin{displaymath}	
\renewcommand{\arraystretch}{2.5}
\begin{array}{lll}
{\footnotesize
\vcenter{
\infer[(W)]{\flex{k+1}_i}
{
\flex{k}_i\rightarrow\flex{k+1}_i &
\neg( \stub_i \leftrightarrow \flex{k}_i)
}}
}
& \quad
\begin{minipage}{1.2cm}
\begin{tikzpicture}[level distance=1mm]
\draw (0,0) ellipse (6.3mm and 3.3mm);
\draw[fill = gray!60] (0,0) ellipse (3.5mm and 2mm);
\draw[fill = gray!20] (0,0) ellipse (1.5mm and 1mm);
\end{tikzpicture}
\end{minipage} \\
{\footnotesize
\vcenter{\infer[(C)]{\flex{k+1}_i}{\flex{k}_i & \neg( \stub_i \leftrightarrow \flex{k}_i)& \neg( \flex{k}_i \rightarrow \flex{k+1}_i) & \neg( \flex{k+1}_i \rightarrow \flex{k}_i)}}
}
& \quad
\begin{minipage}{1.2cm}
\begin{tikzpicture}[level distance=1mm]
\draw (0.1,0) ellipse (5mm and 2.5mm);
\draw[fill = gray!80] (-0.1,0) ellipse (1.5mm and 1mm);
\draw[fill = gray!30] (0.3,0) ellipse (1.5mm and 1mm);
\end{tikzpicture}
\end{minipage}
\
\begin{minipage}{1cm}
\begin{tikzpicture}[level distance=1mm]
\draw (0.1,0) ellipse (5mm and 2.5mm);
\draw[fill = gray!30] (0.2,0) ellipse (1.7mm and 1.1mm);
\draw[fill = gray!80] (0,0) ellipse (1.7mm and 1.1mm);
\draw (0.2,0) ellipse (1.7mm and 1.1mm);
\end{tikzpicture}
\end{minipage} \\
{\footnotesize
\vcenter{\infer[(S)]{\varphi}{\varphi & \stub_i \leftrightarrow \varphi}}
}
& \quad
\begin{minipage}{1.2cm}
\begin{tikzpicture}[level distance=1mm]
\draw[fill = gray!30] (0,0) ellipse (6.3mm and 3.3mm);
\draw[fill = gray!80] (0,0) ellipse (6mm and 3mm);
\end{tikzpicture}
\end{minipage}
\end{array}
\end{displaymath}
\caption{Rules for making new proposals and the corresponding EGG/YOLKs. The dark gray yolk identifies $\flex{k+1}_i$ and the light gray one identifies $\flex{k}_i$.}\label{proposal}
\end{table}

During a negotiation process, the viewpoint of each agent is presented in
a specific \textit{angle}. In other words, a viewpoint is a hierarchy of
theories, related by the partial order relation of weakening, and an
element of this hierarchy is an angle. Each agent presents angles in
sequence during the negotiation. Thus we call \textit{current angle
formula} (\emph{CAF}) the angle presented at the current stage of the
negotiation. A flexible formula $\flex{k}_i$ expresses the
\emph{$k^\mathrm{th}$ angle} discussed in the MN by the agent $i$ and it changes during
the process. We assume here that for each CAF $\flex{k}_i$ there is a
stubborn formula in $\LS{i}$ that is a generalization of it. In
general, during a negotiation of the meaning of a term, the agents relax
their viewpoint in order to meet the opponent's one, and they do this
only if the relaxing formula is not too general. Then, for each
assertion in the MN, the agents have a maximal generalization
of it and this is a formula in the stubbornness set. For instance, if
the object of the negotiation is the meaning of \emph{pen}, an agent is
flexible on the ink color of the object but not on the fact that the
object contains ink; then, the \emph{red ink} predicate is a flexible
one and the \emph{contains ink} predicate is a stubborn one.

$\flex{k}_i$ changes during the MN by applying to it one of the rules
for making new proposals given in Table~\ref{proposal}:
weakening $(W)$, changing theory $(C)$ or stubbornness $(S)$. The
EGG/YOLK representations show the collocation of the new proposal (in
the stubbornness situation the new proposal is the same as the last
one).

Let $\flex{k}_i$ be the last proposal of an agent $i$ during a MN. There
are two ways for $i$ to make a new proposal $\flex{k+1}_i$. The
weakening rule $(W)$ states that $i$ can propose $\flex{k+1}_i$ if
$\flex{k+1}_i$ is entailed by $\flex{k}_i$ (i.e., $\flex{k}_i\rightarrow
\flex{k+1}_i$) and $\flex{k}_i$ is not the most general formula the
agent can negotiate (corresponding to her stubbornness viewpoint, i.e.,
$\flex{k}_i\leftrightarrow \stub_i$). Note that if $i$ weakens, say,
$\flex{0}_i$ to the new CAF $\flex{1}_i$, then $i$ may be no more able
to satisfy $\flex{0}_i$.

The rule $(C)$ states that $i$ can just change angle. Suppose that
$\flex{k}_i$ is the last proposal of an agent $i$ during a MN. There are
two ways for $i$ to make a new proposal $\flex{k+1}_i$. In the first
case, expressed by the weakening rule $(W)$, $i$ proposes $\flex{k+1}_i$
if $\flex{k+1}_i$ is entailed by $\flex{k}_i$ (i.e.,
$\flex{k}_i\rightarrow \flex{k+1}_i$) and $\flex{k}_i$ is not the most
general formula the agent can negotiate (corresponding to her
stubbornness viewpoint, i.e., $\flex{k}_i\leftrightarrow \stub_i$). In
the second case, expressed by the rule $(C)$, $i$ just changes theory.
Although we do not consider MN strategies in detail here, in general, an
agent chooses whether to perform a weakening or a changing theory action
by applying the corresponding rule, but there are situations in which
one action is better than the other. For instance, when an agent checks
the compatibility situation it seems better to weaken the theory rather
than changing it so to try to entail the opponent's viewpoint, while in
essence disagreement situations it seems better to change the theory
rather than weakening it so to try to meet the opponent's viewpoint.

If agent $i$ is in stubbornness does she continue the negotiation or
does she have to exit it? We assume that the agent exits the MN only if
all the agents in the negotiation are stubborn. But an agent does not
know the opponent's stubbornness viewpoint, so the exit condition is
recognized only by the system. However, the stubborn agent always makes
the same proposal during the MN, as expressed by the rule $(S)$. If
$\flex{k}_i\leftrightarrow \stub_i$ then $\flex{k_1}_i = \flex{k_1+1}_i$
for all $k_1>k$.

Let us now go deeply inside the negotiation process constraints. If an
agent $i$ makes a weakening of $\flex{0}_i$ and has $\flex{1}_i$ as the
CAF, then $i$ is no more able to satisfy $\flex{0}_i$. As we show below,
the process of negotiation, means relaxing of individual hierarchies. In
particular, based upon the reasoning above, $\flex{k}_{i}$ is the
$k^{\mathrm{th}}$ angle of agent $i$.

We introduce a set of $\sign_i$-structures as agents change angles
during the negotiation process and these viewpoints have to be satisfied
in different structures. We thus define the semantical structure of a
signature, which is built by a domain set and an interpretation function
mapping predicate symbols into tuples of elements of the domain. We use
a parameter $k$ to denote the $k^{\mathrm{th}}$
structure of the $k^{\mathrm{th}}$ angle.

\begin{definition}\label{struct}
Given a signature $\sign_i = \langle \P_i,\alpha_i\rangle$, a
\emph{$\sign_i$-structure} $\agstruct{i}{}$ is a pair
$\langle\dstr{i}{},\istr{i}{}\rangle$ where the \emph{domain}
$\dstr{i}{}$ is a finite non-empty set and the \emph{interpretation
function} $\istr{i}{}$ is such that $\istr{i}{}(P)\subseteq \dstr{i}{n}$
for all $P \in \P_i$ for which $\alpha(P)=n$.

We define the set of $\sign_i$-structures $\agstruct{i}{k}$ as
$\mathcal{S}_i= \{\agstruct{i}{k} \mid \agstructdef{i}{k}\}$
where
$\dstr{i}{k} \subseteq \dstr{i}{}$ is the domain set with respect to agent $i$ and,
for all pairs
$(\istr{i}{k},\istr{i}{k+1})$, if the $(k+1)^{\mathrm{th}}$ rule that
agent
$i$ applied is:
    \begin{itemize}
    \item $(W)$, then $\istr{i}{k}(P) \subseteq \istr{i}{k+1}(P)$ for all $P \in \P_i$;
    \item $(C)$, then $\istr{i}{k}(P) \neq \istr{i}{k+1}(P)$, $\istr{i}{k}(P) \nsubseteq \istr{i}{k+1}(P)$ and $\istr{i}{k+1}(P) \nsubseteq \istr{i}{k}(P)$ for all $P \in \P_i$;
    \item $(S)$, then $\istr{i}{k}(P) = \istr{i}{k+1}(P)$, for all $P \in \P_i$.
    \end{itemize}
If $\varphi$ and $\psi$ are $\sign_i$-formulas then:
\begin{itemize}
\item $\agstruct{i}{k} \models P(t_1,\dots,t_n)$ iff $(\mathcal{I}_i(t_1),\dots, \mathcal{I}_i(t_n))\in\mathcal{I}_i(P)$, where $P\in\P_i$ and $t_1,\dots,t_n$ are terms;
\item $\agstruct{i}{k} \models \neg \varphi$ iff $\agstruct{i}{k}\not\models \varphi$;
\item $\agstruct{i}{k} \models \varphi \wedge \psi$ iff $\agstruct{i}{k} \models \varphi$ and $\agstruct{i}{k}\models \psi$;
\item $\agstruct{i}{k} \models \varphi \vee \psi$ iff $\agstruct{i}{k}\models \varphi$ or $\agstruct{i}{k}\models \psi$;
\item $\agstruct{i}{k} \models \varphi \rightarrow \psi$ iff $\agstruct{i}{k}\models \psi$ or $\agstruct{i}{k}\models \neg\varphi$.
\end{itemize}
\end{definition}

\begin{example}\label{es:3}
Suppose Alice defines ``vehicle'' as in Example~\ref{es:1}. Then
\begin{eqnarray*}
\stub_{A}&=&(\mathit{has2wheels} \vee \mathit{has3wheels} \vee \mathit{has4wheels}\vee \mathit{has6wheels}) \wedge \\
        &\ &(\mathit{hasHandlebar} \vee \mathit{hasSteeringWheel})\wedge\\
        &\ &(\mathit{hasMotor} \vee \mathit{has2bicyclePedals} \vee \mathit{has4bicyclePedals} \vee \mathit{hasTowBar})
\end{eqnarray*}
is the stubbornness part of Alice's knowledge whose interpretation is
$\mathcal{I}(\stub_{A})=\{\mathsf{bicycle,tandem,motorbike,scooter,truck,car,trailer,chariot}\}$.
Let
$$\flex{k}_{A}= \mathit{has4wheels} \wedge \mathit{hasSteeringWheel} \wedge (\mathit{hasMotor} \vee \mathit{has2bicyclePedals})$$
be the CAF of Alice that it is not equivalent to her stubbornness knowledge and its interpretation is $\mathcal{I}(\flex{k}_{A})=\{\mathsf{car,truck}\}\subset\mathcal{I}(\stub_{A})$. Suppose Alice changes her CAF by means of a weakening action ($W$); then:
\begin{eqnarray*}
\flex{k+1}_{A} &=& (\mathit{has4wheels} \vee \mathit{has2wheels})\wedge (\mathit{hasSteeringWheel}\vee \mathit{hasHandlebar}) \wedge \\
&\ & (\mathit{hasMotor}\vee \mathit{has2bicyclePedals})
\end{eqnarray*}
The interpretation of $\flex{k+1}_{A}$ is $\mathcal{I}(\flex{k+1}_{A})=\{\mathsf{motorbike,scooter,car,truck}\}\subset\mathcal{I}(\flex{k}_{A})$.
Otherwise, suppose Alice changes her CAF by means of a changing theory action ($C$); then:
$$\flex{k+1}_{A}= \mathit{has6wheels}\wedge \mathit{hasSteeringWheel} \wedge (\mathit{hasMotor} \vee \mathit{hasTowBar})$$

The interpretation of $\flex{k+1}_{A}$ is $\mathcal{I}(\flex{k+1}_{A})=\{\mathsf{truck,trailer}\}$ and $\mathcal{I}(\flex{k+1}_{A})\nsubseteq \mathcal{I}(\flex{k}_{A})$. \
\hfill$\Box$
\end{example}


\section{The MN Process}\label{sec:process}
In this section, we formalize the MN process by defining the negotiation
language (Section~\ref{sec:MN-language}), i.e. how the agents send their
proposals to the opponents, and the negotiation rules
(Section~\ref{sec:MN-rules}) governing the development of the MN
process, both for bilateral (Section~\ref{subsubsec:1-1MNRules}) and
1-n negotiation (Section~\ref{subsubsec:1-nMNRules}). We then
(Section~\ref{subsec:MNdevelop}) show how the bilateral MN develops
depending upon the relation between the stubborn knowledge of the
agents. We do not show the development of the 1-n MN because it
can be viewed as $n-1$ instances of bilateral MN between the auctioneer
and the other agents, where $n$ is the number of the involved agents.

During the MN, agents make proposals and say if they are in agreement or
not with respect to the proposals made by the opponents. Proposals are
negotiation formulas like $j:\varphi$, where we assume that the opponent
$i$ is able to recognize the name label $j$ in $j:\varphi$ and remove it
in order to evaluate $\varphi$.

In general, negotiating agents may not share the same language but have
different signatures. Hence, when $i$ evaluates an assertion by
$j$, she first has to \emph{translate} the symbols occurring in it to
symbols belonging to her signature. Such a translation depends, of
course, on the particular terms that are being considered for the
negotiation, so we assume abstractly that for each pair of agents
$(i,j)$ there is the \emph{translation function}
$\trad{i}{j}$ such that:
$$
\trad{i}{j}:\sign_j\rightarrow \sign_i\,.
$$

When $j$ asserts $\varphi$ (i.e., $j:\varphi$), $i$ is not able to find
which part of $\varphi$ is in the stubbornness set of $j$, since she
only knows that $\varphi = \stub_j \wedge \psi^k$ where $\stub_j$ is the
conjunction of all the formulas in $\LS{j}$ and $\psi^k$ is the
$k^{\mathrm{th}}$ angle of $j$.

In the following, we describe the main conditions an agent has to test
in order to evaluate the opponent proposal and to identify the
negotiation condition she is in. We suppose that $j$ is the first
proponent (bidding) agent and that $i$ is the agent evaluating $j$'s
proposal. Figures~\ref{fig:EYcall-away}, ~\ref{fig:EYagr}
~\ref{fig:EYabs}, ~\ref{fig:EYess}, ~\ref{fig:EYrel}, and
~\ref{fig:EYcomp} show the EGG/YOLK representations in which $i$ is
identified by the plain line and $j$ by the dashed line for each
condition that $i$ tests; the numbering is that of \citep{lehmann94}.
Let $\varphi$ be the proposal of $j$.
When the MN begins the agent receiving the first proposal controls that it is not too general and not too restrictive with respect to her viewpoint.
\begin{itemize}
\item
If the received proposal $\varphi$ is too general ($\stub_i \rightarrow \trad{i}{j}(\varphi)$), then the agent $i$ cannot negotiate with $j$ because no generalization of her stubbornness knowledge is acceptable. In this case $i$ thinks they are in \emph{call-away} and the negotiation ends in a negative way.
The corresponding EGG/YOLK representation is shown in
Figure~\ref{fig:EYcall-away}.

\item Otherwise, in the case in which the received proposal $\varphi$ is too restrictive ($\trad{i}{j}(\varphi)\rightarrow \flex{0}_i$) the only action $i$ can perform is to re-initiate the MN by proposing her angle that is a generalization of $\varphi$.
\end{itemize}

\begin{figure}[t]\centering
\begin{tabular}{c}
\begin{minipage}{2.5cm}
22:  \\
\begin{tikzpicture}[level distance=1mm]
\draw[densely dashed] (0,0) ellipse (6.5mm and 3.5mm);
\draw[densely dashed,fill = gray!20] (0,0) ellipse (5mm and 2.8mm);
\draw (0,0) ellipse (3.5mm and 2mm);
\draw[fill = gray!60] (0,0) ellipse (1.5mm and 1mm);
\end{tikzpicture}
\end{minipage}
\end{tabular}
\caption[Call-away Egg/Yolk configurations]{The EGG/YOLK representation of the opponent's offer from agent $i$'s  viewpoint identified by the plain lines, for call-away.}\label{fig:EYcall-away}
\end{figure}
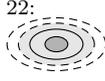

\begin{figure}[t]\centering
\begin{tabular}{c c c c}
\begin{minipage}{1.8cm}
17: \\
\begin{tikzpicture}[level distance=1mm]
\draw[fill = gray!60] (0.1,0) ellipse (3mm and 1.6mm);
\draw[densely dashed] (0.4,0) ellipse (5mm and 2.5mm);
\draw[densely dashed,fill = gray!20] (0.2,0) ellipse (1.4mm and 0.7mm);
\draw (0,0) ellipse (5mm and 2.5mm);
\end{tikzpicture}
\end{minipage}&
\begin{minipage}{1.5cm}
21:  \\
\begin{tikzpicture}[level distance=1mm]
\draw (0,0) ellipse (6.5mm and 3.5mm);
\draw[fill = gray!60] (0,0) ellipse (5mm and 2.8mm);
\draw[densely dashed] (0,0) ellipse (3.5mm and 2mm);
\draw[densely dashed,fill = gray!20] (0,0) ellipse (1.5mm and 1mm);
\end{tikzpicture}
\end{minipage}&
\begin{minipage}{1.8cm}
23: \\
\begin{tikzpicture}[level distance=1mm]
\draw (0.2,0) ellipse (6.5mm and 3.5mm);
\draw[fill = gray!60] (0.1,0) ellipse (3.5mm and 1.5mm);
\draw[densely dashed] (0.3,0) ellipse (3.5mm and 2mm);
\draw[densely dashed,fill = gray!20] (0.15,0) ellipse (1.5mm and 1mm);
\end{tikzpicture}
\end{minipage}&
\begin{minipage}{1.8cm}
29: \\
\begin{tikzpicture}[level distance=1mm]
\draw (0,0) ellipse (5mm and 2.5mm);
\draw[fill = gray!60] (0.2,0) ellipse (2mm and 1.2mm);
\draw[densely dashed] (0.4,0) ellipse (5mm and 2.5mm);
\draw[densely dashed,fill = gray!20] (0.2,0) ellipse (1.4mm and 0.7mm);
\end{tikzpicture}
\end{minipage}\\
\begin{minipage}{1.5cm}
33:  \\
\begin{tikzpicture}[level distance=1mm]
\draw (0,0) ellipse (6.5mm and 3.5mm);
\draw[fill = gray!60] (0,0) ellipse (4mm and 2.5mm);
\draw[densely dashed] (0,0) ellipse (3.8mm and 2.3mm);
\draw[densely dashed,fill = gray!20] (0,0) ellipse (1.5mm and 1mm);
\end{tikzpicture}
\end{minipage}&
\begin{minipage}{1.8cm}
36: \\
\begin{tikzpicture}[level distance=1mm]
\draw (0,0) ellipse (6.5mm and 3.5mm);
\draw[densely dashed] (0,0) ellipse (5mm and 2.8mm);
\draw[fill = gray!60] (0,0) ellipse (3.5mm and 2mm);
\draw[densely dashed,fill = gray!20] (0,0) ellipse (1.5mm and 1mm);
\end{tikzpicture}
\end{minipage}&
\begin{minipage}{1.8cm}
38: \\
\begin{tikzpicture}[level distance=1mm]
\draw[densely dashed] (0,0) ellipse (6.5mm and 3.5mm);
\draw (0,0) ellipse (5mm and 2.8mm);
\draw[fill = gray!60] (0,0) ellipse (3.5mm and 2mm);
\draw[densely dashed,fill = gray!20] (0,0) ellipse (1.5mm and 1mm);
\end{tikzpicture}
\end{minipage}&
\begin{minipage}{1.5cm}
39: \\
\begin{tikzpicture}[level distance=1mm]
\draw (0,0) ellipse (5mm and 2.5mm);
\draw[fill = gray!60] (0.2,0) ellipse (2mm and 1.2mm);
\draw[densely dashed] (0.4,0) ellipse (5mm and 2.5mm);
\draw[densely dashed,fill = gray!20] (0.2,0) ellipse (1.8mm and 1mm);
\end{tikzpicture}
\end{minipage}\\
\begin{minipage}{1.5cm}
40: \\
\begin{tikzpicture}[level distance=1mm]
\draw (0,0) ellipse (6.5mm and 3.5mm);
\draw[densely dashed] (0,0) ellipse (5mm and 2.8mm);
\draw[densely dashed,fill = gray!20] (0,0) ellipse (3.5mm and 2mm);
\draw[fill = gray!60] (0,0) ellipse (3.3mm and 1.8mm);
\end{tikzpicture}
\end{minipage}&
\begin{minipage}{1.8cm}
41: \\
\begin{tikzpicture}[level distance=1mm]
\draw[densely dashed] (0,0) ellipse (6.5mm and 3.5mm);
\draw (0,0) ellipse (5mm and 2.8mm);
\draw[fill = gray!60] (0,0) ellipse (3.5mm and 2mm);
\draw[densely dashed,fill = gray!20] (0,0) ellipse (3.3mm and 1.8mm);
\end{tikzpicture}
\end{minipage}&
\begin{minipage}{1.8cm}
42d: \\
\begin{tikzpicture}[level distance=1mm]
\draw[densely dashed] (0,0) ellipse (6.5mm and 3.5mm);
\draw (0,0) ellipse (6.3mm and 3.3mm);
\draw[fill = gray!60] (0,0) ellipse (3.5mm and 2mm);
\draw[densely dashed,fill = gray!20] (0,0) ellipse (1.5mm and 1mm);
\end{tikzpicture}
\end{minipage}&
\begin{minipage}{1.5cm}
42e: \\
\begin{tikzpicture}[level distance=1mm]
\draw[densely dashed] (0,0) ellipse (6.5mm and 3.5mm);
\draw (0,0) ellipse (6.3mm and 3.3mm);
\draw[fill = gray!60] (0,0) ellipse (3.5mm and 2mm);
\draw[densely dashed,fill = gray!20] (0,0) ellipse (3.3mm and 1.8mm);
\end{tikzpicture}
\end{minipage}\\
\end{tabular}
\caption[Agreement Egg/Yolk configurations]{The EGG/YOLK representation of the opponent's offer from agent $i$'s  viewpoint identified by the plain lines, for agreement.}\label{fig:EYagr}
\end{figure}
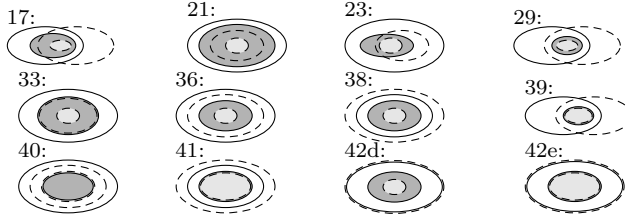

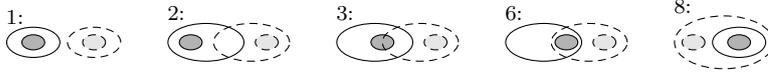
\begin{figure}[t]\centering
\begin{tabular}{c c c c c}
\begin{minipage}{1.7cm}
1: \\
\begin{tikzpicture}[level distance=1mm]
\draw (0,0) ellipse (3.5mm and 2mm);
\draw[fill = gray!60] (0,0) ellipse (1.5mm and 1mm);
\draw[densely dashed] (0.8,0) ellipse (3.5mm and 2mm);
\draw[densely dashed,fill = gray!20] (0.8,0) ellipse (1.5mm and 1mm);
\end{tikzpicture}
\end{minipage} &
\begin{minipage}{1.8cm}
2: \\
\begin{tikzpicture}[level distance=1mm]
\draw (0.2,0) ellipse (5mm and 2.5mm);
\draw[fill = gray!60] (0,0) ellipse (1.5mm and 1mm);
\draw[densely dashed] (0.8,0) ellipse (5mm and 2.5mm);
\draw[densely dashed,fill = gray!20] (1,0) ellipse (1.5mm and 1mm);
\end{tikzpicture}
\end{minipage}&
\begin{minipage}{1.8cm}
3: \\
\begin{tikzpicture}[level distance=1mm]
\draw (0.2,0) ellipse (5mm and 2.5mm);
\draw[fill = gray!60] (0.3,0) ellipse (1.5mm and 1mm);
\draw[densely dashed] (0.8,0) ellipse (5mm and 2.5mm);
\draw[densely dashed,fill = gray!20] (1,0) ellipse (1.5mm and 1mm);
\end{tikzpicture}
\end{minipage}&
\begin{minipage}{1.8cm}
6: \\
\begin{tikzpicture}[level distance=1mm]
\draw (0.2,0) ellipse (5mm and 2.5mm);
\draw[fill = gray!60] (0.5,0) ellipse (1.5mm and 1mm);
\draw[densely dashed] (0.8,0) ellipse (5mm and 2.5mm);
\draw[densely dashed,fill = gray!20] (1,0) ellipse (1.5mm and 1mm);
\end{tikzpicture}
\end{minipage}&
\begin{minipage}{1.8cm}
8: \\
\begin{tikzpicture}[level distance=1mm]
\draw[densely dashed] (0.3,0) ellipse (6.5mm and 3.5mm);
\draw[densely dashed,fill = gray!20] (-0.1,0) ellipse (1.5mm and 1mm);
\draw (0.5,0) ellipse (3.5mm and 2mm);
\draw[fill = gray!60] (0.5,0) ellipse (1.5mm and 1mm);
\end{tikzpicture}
\end{minipage}\\
\end{tabular}
\caption[Absolute Disagreement Egg/Yolk configurations]{The EGG/YOLK representation of the opponent's offer from agent $i$'s  viewpoint identified by the plain lines, for absolute disagreement.}\label{fig:EYabs}
\end{figure}

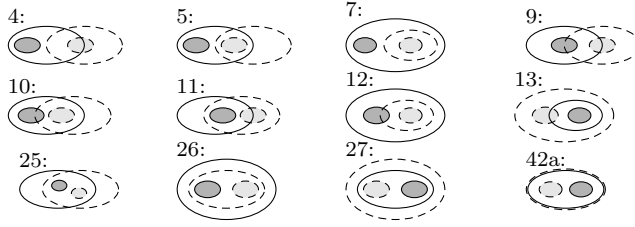
\begin{figure}[t]\centering
\begin{tabular}{c c c c}
\begin{minipage}{1.8cm}
4: \\
\begin{tikzpicture}[level distance=1mm]
\draw[densely dashed,fill = gray!20] (0.65,0) ellipse (1.7mm and 1mm);
\draw (0.2,0) ellipse (5mm and 2.5mm);
\draw[fill = gray!60] (-0.05,0) ellipse (1.7mm and 1mm);
\draw[densely dashed] (0.7,0) ellipse (5mm and 2.5mm);
\end{tikzpicture}
\end{minipage}&
\begin{minipage}{1.8cm}
5: \\
\begin{tikzpicture}[level distance=1mm]
\draw[densely dashed,fill = gray!20] (0.45,0) ellipse (1.7mm and 1mm);
\draw (0.2,0) ellipse (5mm and 2.5mm);
\draw[fill = gray!60] (-0.05,0) ellipse (1.7mm and 1mm);
\draw[densely dashed] (0.7,0) ellipse (5mm and 2.5mm);
\end{tikzpicture}
\end{minipage}&
\begin{minipage}{1.8cm}
7: \\
\begin{tikzpicture}[level distance=1mm]
\draw (0.3,0) ellipse (6.5mm and 3.5mm);
\draw[fill = gray!60] (-0.1,0) ellipse (1.5mm and 1mm);
\draw[densely dashed] (0.5,0) ellipse (3.5mm and 2mm);
\draw[densely dashed,fill = gray!20] (0.5,0) ellipse (1.5mm and 1mm);
\end{tikzpicture}
\end{minipage} &
\begin{minipage}{1.5cm}
9: \\
\begin{tikzpicture}[level distance=1mm]
\draw[densely dashed,fill = gray!20] (0.7,0) ellipse (1.7mm and 1mm);
\draw (0.2,0) ellipse (5mm and 2.5mm);
\draw[fill = gray!60] (0.2,0) ellipse (1.7mm and 1mm);
\draw[densely dashed] (0.7,0) ellipse (5mm and 2.5mm);
\end{tikzpicture}
\end{minipage}\\
\begin{minipage}{1.8cm}
10: \\
\begin{tikzpicture}[level distance=1mm]
\draw[densely dashed,fill = gray!20] (0.4,0) ellipse (1.7mm and 1mm);
\draw (0.2,0) ellipse (5mm and 2.5mm);
\draw[fill = gray!60] (0,0) ellipse (1.7mm and 1mm);
\draw[densely dashed] (0.55,0) ellipse (5mm and 2.5mm);
\end{tikzpicture}
\end{minipage}&
\begin{minipage}{1.8cm}
11: \\
\begin{tikzpicture}[level distance=1mm]
\draw[densely dashed,fill = gray!20] (0.7,0) ellipse (1.7mm and 1mm);
\draw (0.2,0) ellipse (5mm and 2.5mm);
\draw[fill = gray!60] (0.3,0) ellipse (1.7mm and 1mm);
\draw[densely dashed] (0.55,0) ellipse (5mm and 2.5mm);
\end{tikzpicture}
\end{minipage}&
\begin{minipage}{1.8cm}
12: \\
\begin{tikzpicture}[level distance=1mm]
\draw (0.2,0) ellipse (6.5mm and 3.5mm);
\draw[fill = gray!60] (-0.05,0) ellipse (1.7mm and 1.1mm);
\draw[densely dashed] (0.35,0) ellipse (3.5mm and 2mm);
\draw[densely dashed,fill = gray!20] (0.4,0) ellipse (1.5mm and 1mm);
\end{tikzpicture}
\end{minipage}&
\begin{minipage}{1.8cm}
13: \\
\begin{tikzpicture}[level distance=1mm]
\draw[densely dashed] (0.2,0) ellipse (6.5mm and 3.5mm);
\draw[densely dashed,fill = gray!20] (-0.05,0) ellipse (1.7mm and 1.1mm);
\draw (0.35,0) ellipse (3.5mm and 2mm);
\draw[fill = gray!60] (0.4,0) ellipse (1.5mm and 1mm);
\end{tikzpicture}
\end{minipage}\\
\begin{minipage}{1.5cm}
25: \\
\begin{tikzpicture}[level distance=1mm]
\draw[densely dashed,fill = gray!20] (0.48,-0.05) ellipse (1mm and 0.7mm);
\draw (0.2,0) ellipse (5mm and 2.5mm);
\draw[fill = gray!60] (0.22,0.05) ellipse (1mm and 0.7mm);
\draw[densely dashed] (0.5,0) ellipse (5mm and 2.5mm);
\end{tikzpicture}
\end{minipage}&
\begin{minipage}{1.8cm}
26: \\
\begin{tikzpicture}[level distance=1mm]
\draw (0.1,0) ellipse (6.5mm and 4mm);
\draw[fill = gray!60] (-0.15,0) ellipse (1.7mm and 1.1mm);
\draw[densely dashed] (0.1,0) ellipse (5mm and 2.5mm);
\draw[densely dashed,fill = gray!20] (0.35,0) ellipse (1.7mm and 1.1mm);
\end{tikzpicture}
\end{minipage}&
\begin{minipage}{1.8cm}
27: \\
\begin{tikzpicture}[level distance=1mm]
\draw[densely dashed] (0.1,0) ellipse (6.5mm and 4mm);
\draw[densely dashed,fill = gray!20] (-0.15,0) ellipse (1.7mm and 1.1mm);
\draw (0.1,0) ellipse (5mm and 2.5mm);
\draw[fill = gray!60] (0.35,0) ellipse (1.7mm and 1.1mm);
\end{tikzpicture}
\end{minipage}&
\begin{minipage}{1.5cm}
42a: \\
\begin{tikzpicture}[level distance=1mm]
\draw[densely dashed] (0.1,0) ellipse (5.2mm and 2.7mm);
\draw[densely dashed,fill = gray!20] (-0.1,0) ellipse (1.5mm and 1mm);
\draw (0.1,0) ellipse (5mm and 2.5mm);
\draw[fill = gray!60] (0.3,0) ellipse (1.5mm and 1mm);
\end{tikzpicture}
\end{minipage} \\
\end{tabular}
\caption[Essence Disagreement Egg/Yolk configurations]{The EGG/YOLK representation of the opponent's offer from agent $i$'s  viewpoint identified by the plain lines, for essence disagreement.}\label{fig:EYess}
\end{figure}

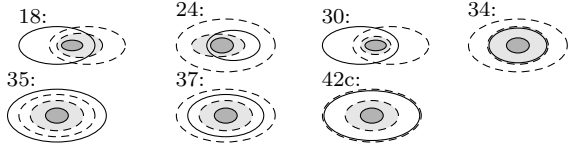
\begin{figure}[t]\centering
\begin{tabular}{c c c c}
\begin{minipage}{1.5cm}
18: \\
\begin{tikzpicture}[level distance=1mm]
\draw[densely dashed,fill = gray!20] (0.3,0) ellipse (3mm and 1.6mm);
\draw[densely dashed] (0.4,0) ellipse (5mm and 2.5mm);
\draw[fill = gray!60] (0.2,0) ellipse (1.4mm and 0.7mm);
\draw (0,0) ellipse (5mm and 2.5mm);
\end{tikzpicture}
\end{minipage}&
\begin{minipage}{1.5cm}
24: \\
\begin{tikzpicture}[level distance=1mm]
\draw[densely dashed] (0.2,0) ellipse (6.5mm and 3.5mm);
\draw[densely dashed,fill = gray!20] (0.1,0) ellipse (3.5mm and 1.5mm);
\draw (0.3,0) ellipse (3.5mm and 2mm);
\draw[fill = gray!60] (0.15,0) ellipse (1.5mm and 1mm);
\end{tikzpicture}
\end{minipage}&
\begin{minipage}{1.5cm}
30: \\
\begin{tikzpicture}[level distance=1mm]
\draw (0,0) ellipse (5mm and 2.5mm);
\draw[densely dashed,fill = gray!20] (0.2,0) ellipse (2mm and 1.2mm);
\draw[densely dashed] (0.4,0) ellipse (5mm and 2.5mm);
\draw[fill = gray!60] (0.2,0) ellipse (1.4mm and 0.7mm);
\end{tikzpicture}
\end{minipage}&
\begin{minipage}{1.5cm}
34:  \\
\begin{tikzpicture}[level distance=1mm]
\draw[densely dashed] (0,0) ellipse (6.5mm and 3.5mm);
\draw[densely dashed,fill = gray!20] (0,0) ellipse (4mm and 2.5mm);
\draw (0,0) ellipse (3.8mm and 2.3mm);
\draw[fill = gray!60] (0,0) ellipse (1.5mm and 1mm);
\end{tikzpicture}
\end{minipage}\\
\begin{minipage}{1.8cm}
35: \\
\begin{tikzpicture}[level distance=1mm]
\draw (0,0) ellipse (6.5mm and 3.5mm);
\draw[densely dashed] (0,0) ellipse (5mm and 2.8mm);
\draw[densely dashed,fill = gray!20] (0,0) ellipse (3.5mm and 2mm);
\draw[fill = gray!60] (0,0) ellipse (1.5mm and 1mm);
\end{tikzpicture}
\end{minipage}&
\begin{minipage}{1.5cm}
37: \\
\begin{tikzpicture}[level distance=1mm]
\draw[densely dashed] (0,0) ellipse (6.5mm and 3.5mm);
\draw (0,0) ellipse (5mm and 2.8mm);
\draw[densely dashed,fill = gray!20] (0,0) ellipse (3.5mm and 2mm);
\draw[fill = gray!60] (0,0) ellipse (1.5mm and 1mm);
\end{tikzpicture}
\end{minipage}&
\begin{minipage}{1.5cm}
42c: \\
\begin{tikzpicture}[level distance=1mm]
\draw[densely dashed] (0,0) ellipse (6.5mm and 3.5mm);
\draw (0,0) ellipse (6.3mm and 3.3mm);
\draw[densely dashed,fill = gray!20] (0,0) ellipse (3.5mm and 2mm);
\draw[fill = gray!60] (0,0) ellipse (1.5mm and 1mm);
\end{tikzpicture}
\end{minipage}\\
\end{tabular}
\caption[Relative Disagreement Egg/Yolk configurations]{The EGG/YOLK representation of the opponent's offer from agent $i$'s  viewpoint identified by the plain lines, for relative disagreement.}\label{fig:EYrel}
\end{figure}

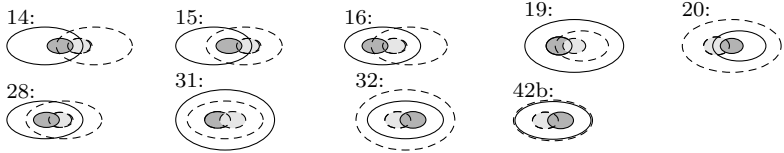
\begin{figure}[t]\centering
\begin{tabular}{c c c c c}
\begin{minipage}{1.8cm}
14: \\
\begin{tikzpicture}[level distance=1mm]
\draw[densely dashed,fill = gray!20] (0.45,0) ellipse (1.7mm and 1mm);
\draw (0,0) ellipse (5mm and 2.5mm);
\draw[fill = gray!60] (0.2,0) ellipse (1.7mm and 1mm);
\draw[densely dashed] (0.65,0) ellipse (5mm and 2.5mm);
\draw[densely dashed] (0.45,0) ellipse (1.5mm and 1mm);
\end{tikzpicture}
\end{minipage}&
\begin{minipage}{1.8cm}
15: \\
\begin{tikzpicture}[level distance=1mm]
\draw[densely dashed,fill = gray!20] (0.45,0) ellipse (1.7mm and 1mm);
\draw (0,0) ellipse (5mm and 2.5mm);
\draw[fill = gray!60] (0.2,0) ellipse (1.7mm and 1mm);
\draw[densely dashed] (0.4,0) ellipse (5mm and 2.5mm);
\draw[densely dashed] (0.45,0) ellipse (1.5mm and 1mm);
\end{tikzpicture}
\end{minipage}&
\begin{minipage}{1.8cm}
16: \\
\begin{tikzpicture}[level distance=1mm]
\draw[densely dashed, fill = gray!20] (0.25,0) ellipse (1.7mm and 1mm);
\draw (0.1,0) ellipse (5mm and 2.5mm);
\draw[fill = gray!60] (0,0) ellipse (1.7mm and 1mm);
\draw[densely dashed] (0.45,0) ellipse (5mm and 2.5mm);
\draw[densely dashed] (0.25,0) ellipse (1.5mm and 1mm);
\end{tikzpicture}
\end{minipage}&
\begin{minipage}{1.5cm}
19: \\
\begin{tikzpicture}[level distance=1mm]
\draw (0.2,0) ellipse (6.5mm and 3.5mm);
\draw[fill = gray!60] (0,0) ellipse (1.7mm and 1.2mm);
\draw[densely dashed] (0.3,0) ellipse (3.5mm and 2mm);
\draw[densely dashed,fill = gray!20] (0.2,0) ellipse (1.5mm and 1mm);
\draw (0,0) ellipse (1.7mm and 1.2mm);
\end{tikzpicture}
\end{minipage} &
\begin{minipage}{1.8cm}
20: \\
\begin{tikzpicture}[level distance=1mm]
\draw[densely dashed] (0.2,0) ellipse (6.5mm and 3.5mm);
\draw[densely dashed,fill = gray!20] (0,0) ellipse (1.7mm and 1.2mm);
\draw (0.3,0) ellipse (3.5mm and 2mm);
\draw[fill = gray!60] (0.2,0) ellipse (1.5mm and 1mm);
\draw[densely dashed] (0,0) ellipse (1.7mm and 1.2mm);
\end{tikzpicture}
\end{minipage}\\
\begin{minipage}{1.8cm}
28: \\
\begin{tikzpicture}[level distance=1mm]
\draw[densely dashed,fill = gray!20] (0.2,0) ellipse (1.7mm and 1mm);
\draw (0,0) ellipse (5mm and 2.5mm);
\draw[fill = gray!60] (0.02,0) ellipse (1.7mm and 1mm);
\draw[densely dashed] (0.25,0) ellipse (5mm and 2.5mm);
\draw[densely dashed] (0.2,0) ellipse (1.5mm and 1mm);
\end{tikzpicture}
\end{minipage}&
\begin{minipage}{1.8cm}
31: \\
\begin{tikzpicture}[level distance=1mm]
\draw (0.1,0) ellipse (6.5mm and 4mm);
\draw[fill = gray!60] (0,0) ellipse (1.7mm and 1.1mm);
\draw[densely dashed] (0.1,0) ellipse (5mm and 2.5mm);
\draw[densely dashed,fill = gray!20] (0.2,0) ellipse (1.7mm and 1.1mm);
\draw[densely dashed] (0,0) ellipse (1.7mm and 1.1mm);
\end{tikzpicture}
\end{minipage}&
\begin{minipage}{1.5cm}
32: \\
\begin{tikzpicture}[level distance=1mm]
\draw[densely dashed] (0.1,0) ellipse (6.5mm and 4mm);
\draw[densely dashed,fill = gray!20] (0,0) ellipse (1.7mm and 1.1mm);
\draw (0.1,0) ellipse (5mm and 2.5mm);
\draw[fill = gray!60] (0.2,0) ellipse (1.7mm and 1.1mm);
\draw[densely dashed] (0,0) ellipse (1.7mm and 1.1mm);
\end{tikzpicture}
\end{minipage}&
\begin{minipage}{1.8cm}
42b: \\
\begin{tikzpicture}[level distance=1mm]
\draw[densely dashed] (0.1,0) ellipse (5.2mm and 2.7mm);
\draw[densely dashed,fill = gray!20] (0,0) ellipse (1.7mm and 1.1mm);
\draw (0.1,0) ellipse (5mm and 2.5mm);
\draw[fill = gray!60] (0.2,0) ellipse (1.7mm and 1.1mm);
\draw[densely dashed] (0,0) ellipse (1.7mm and 1.1mm);
\end{tikzpicture}
\end{minipage}\\
\end{tabular}
\caption[Compatibility Egg/Yolk configurations]{The EGG/YOLK representation of the opponent's offer from agent $i$'s  viewpoint identified by the plain lines, for compatibility.}\label{fig:EYcomp}
\end{figure}

When both of the previous cases are negative, the agent $i$ evaluates how much acceptable is $\varphi$.
\begin{itemize}
\item The ideal\fix{Luca}{Inglese. The ``hopefully situation'' non
ha senso} situation is the agreement. As said before, a proposal is
considered acceptable when it is equivalent to the current angle of $i$
($\flex{k}_i \leftrightarrow \trad{i}{j}(\varphi)$) or when it is
representable by means of a pair of feasible angles, $(\flex{k}_i
\rightarrow \trad{i}{j}(\varphi)) \wedge
(\trad{i}{j}(\varphi)\rightarrow\flex{k+1}_i)$. In our formalism, if
$\flex{k}_i \rightarrow \trad{i}{j}(\varphi)$ then there always exists
$\flex{k+1}_i$ such that the previous condition is true because, as said
in Section~\ref{sec:language}, for each $\flex{k}_i$ there is a stubborn
formula in $\LS{i}$ that is a generalization of it. Thus, $\flex{k}_i
\rightarrow \trad{i}{j}(\varphi)$ and $\flex{k}_i\rightarrow \stub_i$
yield $\trad{i}{j}(\varphi)\rightarrow \stub_i$. In fact, it is not
possible that $\stub_i\rightarrow \trad{i}{j}(\varphi)$ because this is
the call-away condition. Thus, the sufficient condition to reach the
agreement is $\flex{k}_i\rightarrow\trad{i}{j}(\varphi)$. The egg-yolk
configurations for agreement are depicted in Figure~\ref{fig:EYagr}.
When agents are not in agreement, they disagree in many ways and $i$
finds which type of disagreement is between $\varphi$ and her current
angle $\flex{k}_i$.

\item If the proposal of $j$ is not consistent with the stubbornness
knowledge of agent $i$, $\neg(\stub_i \wedge \trad{i}{j}(\varphi))$ then
the agents are in absolute disagreement (Figure~\ref{fig:EYabs}).

\item If $i$ and $j$ are not in absolute disagreement, $i$'s CAF is
consistent with respect to $j$'s proposal, and there is no
generalization/restriction relation between $\flex{k}_i$ and $\varphi$,
$\neg(\flex{k}_i \wedge \trad{i}{j}(\varphi)) \wedge (\stub_i\vee
\trad{i}{j}(\varphi))$, then the agents are in essence disagreement
(Figure~\ref{fig:EYess}).

\item If $i$ and $j$ are neither in essence nor in absolute disagreement
and $\varphi$ is a generalization of $i$'s CAF, $(\flex{k}_i \rightarrow
\trad{i}{j}(\varphi))\wedge \neg(\trad{i}{j}(\varphi)\rightarrow
\flex{k}_i)$, then the agents are in relative disagreement
(Figure~\ref{fig:EYrel}).

\item If $i$ and $j$ are neither in absolute nor in relative
disagreement, $i$'s CAF is consistent with respect to $\varphi$, and
$i$'s CAF is not a weakening of $\varphi$, $(\flex{k}_i \vee
\trad{i}{j}(\varphi))\wedge\neg(\flex{k}_i \rightarrow
\trad{i}{j}(\varphi))\wedge \neg(\trad{i}{j}(\varphi)\rightarrow
\flex{k}_i)$, then the agents are in the compatibility relation
(Figure~\ref{fig:EYcomp}).
\end{itemize}

After evaluating the received proposal, an agent informs the opponent
about the negotiation situation she thinks to be in, in order to give to
the opponent a motivation of the potential disagreement, i.e. the non
acceptability of her proposal. To this end, we extend the formulas in
the agent language:

\begin{definition}[$\L_{i}$ extension]
\label{agLExt}
If $\varphi$ is a proposal received by $i$ in the negotiation process,
then it is a formula asserted by somebody as $j : \varphi$, with $j \neq i$. We extend
the language $\L_{i}$ with the formulas $\mathbf{absDis}(j:\varphi)$,
$\mathbf{essDis}(j:\varphi)$, $\mathbf{relDis}(j:\varphi)$,
$\mathbf{comp}(j:\varphi)$, and $\mathbf{agree}(j:\varphi)$.
For $\agstructdef{i}{k}$ a $\sign_i$-structure, the semantics of these
additional formulas is:
\begin{itemize}
\item $\agstruct{i}{k}\models \mathbf{absDis}(j:\varphi)$ iff
$\agstruct{i}{k}\models \neg(\stub_i \wedge \trad{i}{j}(\varphi))$;
\item $\agstruct{i}{k}\models \mathbf{essDis}(j:\varphi)$ iff
$\agstruct{i}{k}\models (\stub_i \vee \trad{i}{j}(\varphi))\wedge\neg(\flex{k}_i\wedge \trad{i}{j}(\varphi))$;
\item $\agstruct{i}{k}\models \mathbf{relDis}(j:\varphi)$ iff
$\agstruct{i}{k}\models (
\flex{k}_i \rightarrow \trad{i}{j}(\varphi))\wedge\neg(\trad{i}{j}(\varphi)\rightarrow \flex{k}_i)$;
\item $\agstruct{i}{k}\models \mathbf{comp}(j:\varphi)$ iff
$\agstruct{i}{k}\models (\flex{k}_i \vee \trad{i}{j}(\varphi))\wedge \neg(
\flex{k}_i \rightarrow \trad{i}{j}(\varphi))\wedge\neg(\trad{i}{j}(\varphi)\rightarrow \flex{k}_i)$;
\item $\agstruct{i}{k}\models \mathbf{agree}(j:\varphi)$ iff
$\agstruct{i}{k}\models (\flex{k}_i \rightarrow \trad{i}{j}(\varphi))$.
\end{itemize}
\end{definition}

We did not define a sentence $\mathbf{callAway}(j:\varphi)$ as the
call-away condition interrupts the MN. It is also important to remark
that in our system we restrict the evaluation of agent proposals to
formulas in the basic agent language, so no assertion can be made by
agents using extended (and nested) formulas like
$\mathbf{agree}(\mathbf{comp}(j:\varphi))$. This restriction avoids
nested MN processes.


In the two following subsections, we define the negotiation language and
the deductive rules for the MN process.

\subsection{MN language}\label{sec:MN-language}

The negotiation language $\L$ is built by the assertions of the agents
during the negotiation, i.e., labeled formulas $i:\varphi$ meaning that
agent $i \in \agset$ asserts the formula $\varphi \in \L_i$. That is,
$i:\varphi$ represents a proposal the agent $i$ makes in the negotiation
and typically represents her CAF or the evaluation of $\psi$ asserted by
$j$ when $\varphi$ is $\mathbf{R}(j:\psi)$ where $\mathbf{R}$ is one of
the disagreement relations $\mathbf{absDis}$, $\mathbf{essDis}$,
$\mathbf{relDis}$.

\begin{definition}[$\sign$-formula]	\label{negFormulas}
The \emph{signature} of the \emph{MN language} $\L$ is $\sign = \langle
\P,\{\alpha_i\}_{i \in \agset}\rangle$ where $\P=\bigcup_{i\in
\agset}\P_i$ and $\alpha_i:\mathcal{P}_i\rightarrow \mathbb{N}$ is the
arity function for predicate symbols. Let $\varphi$ be a $\L_i$ formula
for some $i\in\agset$; then $\L$ comprises of \emph{$\sign$-formulas}
defined as follows:
\begin{itemize}
\item $i:\varphi$ is a $\sign$-formula;
\item if $\varphi_1$ and $\varphi_2$ are $\sign$-formulas then $\varphi_1 \cap \varphi_2$ is a $\sign$-formula.
\end{itemize}

Let $\negstruct{k} = (\{\agstruct{i}{k}\}_{i \in \agset,
k\in\mathbb{N}},\mathcal{F})$ be a $\sign$-structure where
$\{\agstruct{i}{k}\}_{i \in \agset, k\in\mathbb{N}}$ is the domain
set and $\mathcal{F}$ is an evaluation function that maps name labels into
$\agset$. 
Then:
\begin{itemize}
\item $\negstruct{k}\models i:\varphi$ iff $\agstruct{\mathcal{F}(i)}{k}\models \varphi$;
\item $\negstruct{k} \models \varphi_1 \cap \varphi_2$ iff $\negstruct{k}\models \varphi_1$ and $\negstruct{k}\models \varphi_2$.
\end{itemize}
\end{definition}
We need only the conjunction operation because during the negotiation, the stream of dialog begins with the proposal of the first bidding agent, that is the auctioneer in the $1-n$ MN, and continues with pairs of offer evaluation and proposal of the following proposing agents.

\subsection{MN Rules}\label{sec:MN-rules}

In this section, we provide the deductive rules for the MN process. We
distinguish between pairwise MN and one-to-many MN because in the latter
case an agent behaves differently when she is the
auctioneer.\fix{Luca}{Inglese incerto: ``she is reserved to be the
auctioneer''. La parola reserved suona male. Perch\'e non solo ``she is
the auctioneer''. Reserved viene usato anche dopo. \`E un termine
standard? {\bf Matteo} Non lo e' ho rimosso.} Moreover, in a $1-n$ MN the supervisor system ends the
negotiation in a positive way when all or an acceptable part of the
agents share a common angle: in the former case the negotiation is
totally positive and in the latter it is partially positive.

\begin{figure}[t]\centering
\subfigure[A start.]{
\label{mp}
\begin{tikzpicture}[level distance=4mm]\scriptsize
\draw (0,20) node (alice) {A};
\draw (0,19.2) node (a) {\includegraphics[width=0.5cm]{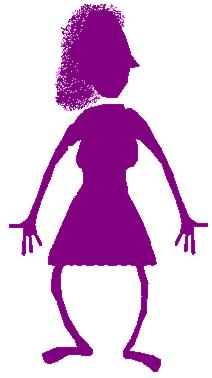}};
\draw (4,20) node (bob) {B};
\draw (4,19.2) node (b) {\includegraphics[width=0.5cm]{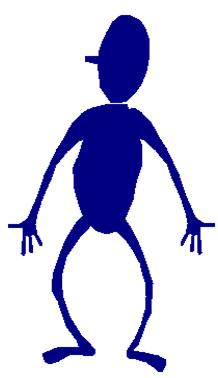}};
\draw[->,thick] (0.5,20) .. controls +(right:0.5cm) and +(left:0.5cm).. node[above,sloped] {$\varphi$} (3.5,20);
\draw[<-,thick] (0.5,19.5) .. controls +(right:0.5cm) and +(left:0.5cm).. node[above,sloped] {$\mathbf{absDis(\varphi),\psi}$} (3.5,19.5);
\draw[->,thick] (0.5,19) .. controls +(right:0.5cm) and +(left:0.5cm).. node[above,sloped] {$\mathbf{absDis(\psi),\varphi'}$} (3.5,19);
\draw (2,18.8) node (c) {$\vdots$};
\end{tikzpicture}
}
\hspace{0.1cm}
\subfigure[A violation.]{
\label{mpViolation}%
\begin{tikzpicture}[level distance=4mm]\scriptsize
\draw (0,20) node (alice) {A};
\draw (0,19.2) node (a) {\includegraphics[width=0.5cm]{alice.jpg}};
\draw (4,20) node (bob) {B};
\draw (4,19.2) node (b) {\includegraphics[width=0.5cm]{bob.jpg}};
\draw[->,thick] (0.5,20) .. controls +(right:0.5cm) and +(left:0.5cm).. node[above,sloped] {$\varphi$} (3.5,20);
\draw[<-,thick] (0.5,19.6) .. controls +(right:0.5cm) and +(left:0.5cm).. node[above,sloped] {$\mathbf{comp(\varphi),\psi}$} (3.5,19.6);
\draw[->,thick] (0.5,19.2) .. controls +(right:0.5cm) and +(left:0.5cm).. node[above,sloped] {$\mathbf{absDis(\psi),\varphi'}$} (3.5,19.2);
\draw (2,18.8) node (c) {$\vdots$};
\end{tikzpicture}}%
\caption{Two bilateral MN scenarios.}
\end{figure}

\renewcommand\arraystretch{1.5}
\begin{table}\centering
\begin{displaymath}
\renewcommand{\arraystretch}{3}
\begin{array}{c}
\infer[(AD)]{i:\mathbf{absDis}(j:\varphi)\cap i:\flex{1}_i}
{j:\varphi & \neg(\stub_i \wedge \trad{i}{j}(\varphi))}
\\
\infer[(ED)]{i:\mathbf{essDis}(j:\varphi)\cap i:\flex{1}_i}
{j:\varphi & \neg(\flex{0}_i\wedge \trad{i}{j}(\varphi))\wedge (\stub_i\vee \trad{i}{j}(\varphi)}
\\
\infer[(I)]{i:\flex{0}_i}
{j:\varphi & \neg(\flex{0}_i \rightarrow \trad{i}{j}(\varphi))\wedge(\trad{i}{j}(\varphi)\rightarrow \flex{0}_i)}
\\
\infer[(Ag)]{i:\mathbf{agree}(j:\varphi)\cap i:\trad{i}{j}(\varphi)}
{j:\varphi & (\flex{0}_i \rightarrow \trad{i}{j}(\varphi))}
\\
\infer[(Co)]{i:\mathbf{comp}(j:\varphi)\cap i:\flex{1}_i}
{j:\varphi & (\flex{0}_i \vee \trad{i}{j}(\varphi))\wedge\neg(\flex{0}_i\rightarrow\trad{i}{j}(\varphi))\wedge\neg(\trad{i}{j}(\varphi)\rightarrow\flex{0}_i)}
\end{array}		
\end{displaymath}
\caption{Rules for the second proposing agent.}
    \label{tab:net1}
\end{table}

\begin{table}[t]
\begin{displaymath}	
	\renewcommand{\arraystretch}{3}
\begin{array}{c}
\infer[(AD\text{-}AD)]{i:\mathbf{absDis}(j:\psi)\cap i:\flex{k+1}_i}
{j:\mathbf{absDis}(i:\flex{k}_i)\cap j:\psi & \neg(\stub_i \wedge \trad{i}{j}(\psi))}
\\
\infer[(AD\text{-}ED)]{i:\mathbf{essDis}(j:\psi)\cap i:\flex{k+1}_i}
{j:\mathbf{absDis}(i:\flex{k}_i)\cap j:\psi & (\stub_i \vee \trad{i}{j}(\psi))\wedge\neg(\flex{k+1}_i \wedge \trad{i}{j}(\psi))}\\

\infer[(AD\text{-}Co)]{i:\mathbf{comp}(j:\psi)\cap i:\flex{k+1}_i}
{j:\mathbf{absDis}(i:\flex{k}_i)\cap j:\psi & (\flex{k+1}_i \vee \trad{i}{j}(\psi))\wedge\neg(\flex{k+1}_i\rightarrow\trad{i}{j}(\psi))\wedge\neg(\trad{i}{j}(\psi)\rightarrow\flex{k+1}_i)}
\\
\infer[(AD\text{-}RD)]{i:\mathbf{relDis}(j:\psi)\cap i:\flex{k+1}_i}
{j:\mathbf{absDis}(i:\flex{k}_i)\cap j:\psi & (\flex{k+1}_i \rightarrow \trad{i}{j}(\psi))\wedge\neg(\trad{i}{j}(\psi)\rightarrow\flex{k+1}_i)}\\

\infer[(AD\text{-}Ag)]{i:\mathbf{agree}(j:\psi)\cap i:\trad{i}{j}(\psi)}
{j:\mathbf{absDis}(i:\flex{k}_i)\cap j:\psi & (\flex{k+1}_i \rightarrow \trad{i}{j}(\psi))}
\\
\infer[(ED\text{-}AD)]{i:\mathbf{absDis}(j:\psi)\cap i:\flex{k+1}_i}
{j:\mathbf{essDis}(i:\flex{k}_i)\cap j:\psi &  \neg(\stub_i \wedge \trad{i}{j}(\psi))}\\

\infer[(ED\text{-}ED)]{i:\mathbf{essDis}(j:\psi)\cap i:\flex{k+1}_i}
{j:\mathbf{essDis}(i:\flex{k}_i)\cap j:\psi & (\stub_i \vee \trad{i}{j}(\psi))\wedge\neg(\flex{k+1}_i \wedge \trad{i}{j}(\psi))}
\\
\infer[(ED\text{-}Co)]{i:\mathbf{comp}(j:\psi)\cap i:\flex{k+1}_i}
{j:\mathbf{essDis}(i:\flex{k}_i)\cap j:\psi & (\flex{k+1}_j\vee \trad{i}{j}(\psi))\wedge\neg(\flex{k+1}_i\rightarrow\trad{i}{j}(\psi))\wedge\neg(\trad{i}{j}(\psi)\rightarrow\flex{k+1}_i)}\\

\infer[(ED\text{-}RD)]{i:\mathbf{relDis}(j:\psi)\cap i:\flex{k+1}_i}
{j:\mathbf{essDis}(i:\flex{k}_i)\cap j:\psi & (\flex{k+1}_i \rightarrow \trad{i}{j}(\psi))\wedge\neg(\trad{i}{j}(\psi)\rightarrow\flex{k+1}_i)}
\\
\infer[(ED\text{-}Ag)]{i:\mathbf{agree}(j:\psi)\cap i:\trad{i}{j}(\psi)}
{j:\mathbf{essDis}(i:\flex{k}_i)\cap j:\psi & (\flex{k+1}_i \rightarrow \trad{i}{j}(\psi))}\\

\infer[(Co\text{-}ED)]{i:\mathbf{essDis}(j:\psi)\cap i:\flex{k+1}_i}
{j:\mathbf{comp}(i:\flex{k}_i)\cap j:\psi & (\stub_i \vee \trad{i}{j}(\psi))\wedge\neg(\flex{k+1}_i \wedge \trad{i}{j}(\psi))}
\\
\infer[(Co\text{-}Co)]{i:\mathbf{comp}(j:\psi)\cap i:\flex{k+1}_i}
{j:\mathbf{comp}(i:\flex{k}_i)\cap j:\psi & (\flex{k+1}_j\vee \trad{i}{j}(\psi))\wedge\neg(\flex{k+1}_i\rightarrow\trad{i}{j}(\psi))\wedge\neg(\trad{i}{j}(\psi)\rightarrow\flex{k+1}_i)}\\

\infer[(Co\text{-}RD)]{i:\mathbf{relDis}(j:\psi)\cap i:\flex{k+1}_i}
{j:\mathbf{comp}(i:\flex{k}_i)\cap j:\psi &  (\flex{k+1}_i \rightarrow \trad{i}{j}(\psi))\wedge\neg(\trad{i}{j}(\psi)\rightarrow\flex{k+1}_i)}
\\
\infer[(Co\text{-}Ag)]{i:\mathbf{agree}(j:\psi)\cap i:\trad{i}{j}(\psi)}
{j:\mathbf{comp}(i:\flex{k}_i)\cap j:\psi & (\flex{k+1}_i \rightarrow \trad{i}{j}(\psi))}\\

\infer[(RD\text{-}Ag)]{i:\mathbf{agree}(j:\psi) \cap i:\trad{i}{j}(\psi)}
{j:\mathbf{relDis}(i:\varphi)\cap j:\psi}
\end{array}
\end{displaymath}
\caption{Rules for the following proposing agents.}
\label{tab:1-1net2}
\end{table}
\renewcommand\arraystretch{1}

\subsubsection{MN Rules: 1-1 MN}\label{subsubsec:1-1MNRules}

We give the transition rules\fix{Luca}{Le regole vanno spiegate tutte, meglio se a gruppi. {\bf Matteo} No, qui non intervengo, ci vuole troppo tempo. Preferisco che andiamo avanti e nella seconda versione per Arxiv e poi per JOLLI sia tutto in ordine, se ci resta solo questo dovrebbe essere fattibile.} the agents use to negotiate depending on
the mutual negotiation position they test and on their flexibility;
these rules are coupled with those in Table~\ref{proposal}. No rules are
needed for the first proposing agent because she only makes a proposal;
conversely the second and the following proposing agents make proposals
and assert the evaluation of the received offer.

There are different rules for the second proposing agent and the
following ones. The second proposing agent has to check if the first
proposal is too general or too restrictive and thus if the negotiation
has to end immediately in the former case (call-away relation), or to
re-initiate with a new proposal in the latter case ($(I)$ rule).
Therefore, a following proposing agent receives proposals that are not
restrictions of her initial one and she has only to test if a received
proposal is not too general.

Consider the scenario in Figure~\ref{mp}: Alice ($A$) makes the proposal
$\varphi$ and Bob ($B$) evaluates it, where $B$'s reasoning is based
upon two tests:
\begin{enumerate}
\item The relation between his CAF and $\varphi$. $B$'s CAF may be in
agreement ($\varphi \leftrightarrow \flex{k}_{B}$) or not with
$\varphi$, and $B$ recognizes it by testing the condition listed above.
\item His stubbornness condition, i.e., if his CAF is $\stub_{B}$
($\flex{k}_{B}\leftrightarrow\stub_{B}$) or not. Whenever $B$ is
stubborn, he performs the same counterproposal, otherwise he may relax
his CAF by the $(W)$ rule of Table \ref{proposal}
($\flex{k}_{B}\rightarrow\flex{k+1}_{B}$) or change his theory by the
$(C)$ rule of Table \ref{proposal} ($\neg( \stub_i \leftrightarrow
\flex{k}_i)$ and $ \neg( \flex{k}_i \rightarrow \flex{k+1}_i)$ and
$\neg( \flex{k+1}_i \rightarrow \flex{k}_i)$).
\end{enumerate}

At the end of his evaluation, $B$ replies to $A$ with a
counterproposal $\psi$. When $A$ evaluates $\psi$ she has to consider
the relation between her CAF and $\psi$, her stubbornness condition
($\stub_{A}\leftrightarrow\flex{k}_{A}$) and $B$'s evaluation.
The evaluation of the opponent agent helps agents in choosing the new
proposal.
The choice of the action, weakening or changing theory, and of the next proposal depends on the agent's attitude: a \emph{collaborative agent} chooses the proposal that improves the negotiation relation with the opponent, whereas a \emph{competitive agent} chooses the proposal that 
changes the least the relation with the opponent. For instance, if $B$ says to $A$ that when $A$ proposes $\varphi$ they are in essence disagreement, and $B$ makes the proposal $\psi$, $A$ 
will propose $\varphi_1$ or $\varphi_2$, both inferred from $\varphi$ by applying (W) or (C). When $A$ is collaborative, she will propose $\varphi_1$ because she knows that they will be in agreement. Conversely, $A$ 
will propose $\varphi_2$, if $A$ is competitive, because she knows that they will remain in essence disagreement.

Suppose $B$ says to $A$ that when $A$ proposes
$\varphi$ they are in relative disagreement ($\psi \rightarrow \varphi$)
and $B$ makes the proposal $\psi$, then $A$ knows that they are in
agreement when she proposes $\psi$.

\renewcommand\arraystretch{1.5}
\begin{table}[t]\centering
\begin{displaymath}
\renewcommand{\arraystretch}{3}
\begin{array}{c}
\infer[(D)]{\mathit{Disagreement}(i,j)}
{\ast(i,j) & i:\varphi & j:\mathbf{na}(i:\varphi) & j:\psi & \stub_i\leftrightarrow\varphi & \stub_j\leftrightarrow\psi}
\\
\infer[(A)]{\mathit{Agreement}(i,j)}{\ast(i,j) & i:\varphi & j:\mathbf{agree}(i:\varphi)}
\\
\infer[(N)]{\mathit{Negotiate}(i,j)}{\ast(i,j) & i:\varphi & j:\mathbf{na}(i:\varphi) & j:\psi}
\end{array}
\end{displaymath}
\caption{1-1 MN system transition rules.}
    \label{1-1sts}
\end{table}
\renewcommand\arraystretch{1}

To support the interaction sketched above, we define the system $\MND$
to consist of the standard introduction and elimination rules for the
connectives of $\L_{i}$ and $\L$, and of two sets of rules: one set for
the second proposing agent (Table~\ref{tab:net1}) and another set for
the following proposing agents (Table~\ref{tab:1-1net2}). For the sake
of space, we omit the assumption of non call-away conditions in
negotiation rules and explain only some of the rules by example.

Assume that $A$ begins a MN by proposing $\flex{0}_{A}$ to $B$. $B$
evaluates $\trad{B}{A}(\flex{0}_{A})$ with respect to his initial angle
$\flex{0}_{B}$ and suppose $B$ thinks that $\trad{B}{A}(\flex{0}_{A})$
is too strict, i.e., $\trad{B}{A}(\flex{0}_{A})
\rightarrow\flex{0}_{B}$. Thus, $B$ cannot accept
$\trad{B}{A}(\flex{0}_{A})$ and re-initiates the MN by the rule $(I)$
and proposes $\flex{0}_{B}$ by $B:\flex{0}_{B}$. Otherwise, suppose $B$
thinks that $\trad{B}{A}(\flex{0}_{A})$ is entailed by his initial angle
$\flex{0}_{B}$ and that $\trad{B}{A}(\flex{0}_{A})$ is not too general,
i.e., it is not entailed by $\stub_{B}$. In this case, $B$ knows that
$A$ cannot accept $\flex{0}_{B}$ because it is too strict with respect
to her viewpoint (explained in the beginning of
Section~\ref{sec:process}), thus if $B$ accepts
$\trad{B}{A}(\flex{0}_{A})$ by $(Ag)$ because it satisfies the
precondition $(\flex{0}_{B}\rightarrow \trad{B}{A}(\flex{0}_A))$, and
says $B:\mathbf{agree}(A:\flex{0}_{A})$. This is the reason why there is
no rule $(RD)$ in Table~\ref{tab:net1} for relative disagreement
relation. Consider the case in which $B$ thinks that the proposal of
$A$, $\flex{0}_{A}$, is consistent to his initial angle $\flex{0}_{B}$
by $(Co)$. $B$ says to $A$ that they are in the compatibility relation
by $B:\mathbf{comp}(A:\flex{0}_{A})$ and makes a new proposal
$B:\flex{1}_{B}$ such that $\flex{0}_{B}\rightarrow \flex{1}_{B}$ (rule
$(W)$). Now $A$ thinks that $\trad{A}{B}(\flex{1}_{B})$ is an acceptable
angle of her initial viewpoint, i.e.
$\flex{1}_{A}\leftrightarrow\trad{A}{B}(\flex{1}_{B})$. Thus $A$ agrees
with $B$ and says $A:\mathbf{agree}(B:\flex{1}_{B})$ by
$(Co\text{-}Ag)$. It may be the case that agents make proposals that
become inconsistent with the received one. This inconsistency is tested
by the opponent agent, not by the bidding one, because in $\MND$ agents
choose the new proposal only with respect to their angles and not with
respect to the opponent's one.

Consider now the scenario in Figure~\ref{mpViolation}. $B$ evaluates the
proposal of $A$, tests the compatibility relation, and makes the
counterproposal. $A$ evaluates it and finds they are inconsistent. In
situations like this, agents perform proposals that violate the MN
relation among agents; we call such a proposal a \emph{violation} and
the rule causing it a \emph{violation rule}. In Table~\ref{tab:1-1net2},
the violation rules are $(ED\text{-}AD)$ and $(ED\text{-}Co)$.

The MN develops by agents making proposals and asserting if they are in
agreement or not. The entire process is controlled by a supervisor, an
external viewpoint, which tests if the MN ends and if the outcome is
positive or negative. Table~\ref{1-1sts} shows the transition rules for the
system
. We write 
\begin{itemize}
\item $j:\mathbf{na}(i:\varphi)$ to say that agent $j$ thinks she is not
in agreement with $i:\varphi$, and

\item $\ast(i,j)$ to say \emph{whatever the system state is} different from
the final ones (\textit{Agreement} and \textit{Disagreement}), i.e.,
whether the system is in \textit{Init} or
\textit{Negotiate}. 
\end{itemize}

The MN begins when agents make proposals in turns
($i:\varphi$, $j:\psi$) and they are not in agreement
($j:\mathbf{na}(i:\varphi)$) by $(N)$. The MN ends with a positive
outcome ($\varphi$) when each agent agrees on a proposal
($j:\mathbf{agree}(i:\varphi)$), otherwise the MN ends with a negative
outcome if there are no more proposals to perform
($\stub_i\leftrightarrow \varphi$ and $\stub_j\leftrightarrow \psi$) and
agents do not agree on a common acceptable angle
($j:\mathbf{na}(i:\varphi)$).

\begin{example}\label{es:4}
Let Alice and Bob be two  agents negotiating the definition of the term ``vehicle'' as in Example~\ref{es:1}. Suppose that the initial viewpoint of Alice is
$$\flex{0}_A=\mathit{has2wheels} \wedge \mathit{hasSteeringWheel} \wedge (\mathit{hasMotor}\vee \mathit{has2bicyclePedals})$$
and her stubbornness knowledge is as in Example~\ref{es:3}. Suppose that Bob's initial viewpoint is
$$\flex{0}_B=\mathit{has2wheels} \wedge \mathit{hasHandlebar} \wedge \mathit{has2bicyclePedals}$$
and his stubbornness knowledge is
\begin{eqnarray*}
\stub_{B}&=&(\mathit{has2wheels} \vee \mathit{has3wheels} \vee \mathit{has4wheels}) \wedge \\
        &\ &(\mathit{hasHandlebar} \vee \mathit{hasSteeringWheel})\wedge\\
        &\ &(\mathit{hasMotor} \vee \mathit{has2bicyclePedals}\vee \mathit{has4bicyclePedals})
\end{eqnarray*}
Alice is the first bidding agent and she proposes $\flex{0}_A$ to Bob, who receives the proposal and evaluates it. Bob tests that they are in compatibility because $(\flex{0}_B\vee\trad{B}{A}(\flex{0}_A))\wedge\neg(\flex{0}_B\rightarrow \trad{B}{A}(\flex{0}_A))\wedge \neg( \trad{B}{A}(\flex{0}_A )\rightarrow\flex{0}_B)$. Bob chooses the new CAF by a weakening action ($W$) in
\begin{eqnarray*}
\flex{1}_B &=&(\mathit{has2wheels} \vee \mathit{has4wheels})\wedge \\
 & \ &(\mathit{hasHandlebar} \vee \mathit{hasSteeringWheel})\wedge \mathit{has2bicyclePedals}
\end{eqnarray*}
Bob uses the ($Co$) rule and sends his CAF to Alice:
$$
\scriptsize
\infer[(Co)]{B:\mathbf{comp}(A:\flex{0}_A)\cap B:\flex{1}_B}
{A:\flex{0}_A & (\flex{0}_B\vee \trad{B}{A}(\flex{0}_A))\wedge \neg(\flex{0}_B\rightarrow \trad{B}{A}(\flex{0}_A))\wedge \neg( \trad{B}{A}(\flex{0}_A)\rightarrow\flex{0}_B)}
$$
The system continues the MN by:
$$
\scriptsize
\infer[(N)]{\mathit{Negotiate}(A,B)}{\ast(A,B) & A:\flex{0}_A & B:\mathbf{comp}(A:\flex{0}_A) & B:\flex{1}_B}
$$
Alice receives $\flex{1}_B$ and she has to make a weakening or a changing theory action because Bob did not say they were in agreement nor in relative disagreement. Alice performs a changing theory action by the rule ($C$) and her CAF is
$$
\flex{1}_A=\mathit{has2wheels}\wedge (\mathit{hasHandlebar} \vee \mathit{hasSteeringWheel})\wedge \mathit{has2bicyclePedals}
$$
Alice thinks they are in relative disagreement since $(\flex{1}_A\rightarrow\trad{A}{B}(\flex{1}_B))\wedge \neg (\trad{A}{B}(\flex{1}_B)\rightarrow\flex{1}_A)$, and she uses the rule ($Co\text{-}RD$) to inform Bob that they are in relative disagreement:
$$
\scriptsize
\infer[(Co\text{-}RD)]{A:\mathbf{relDis}(B:\flex{1}_B)\cap A:\flex{1}_A}
{B:\mathbf{comp}(A:\flex{0}_A)\cap B:\flex{1}_B & (\flex{1}_A\rightarrow \trad{A}{B}(\flex{1}_B))\wedge \neg (\trad{A}{B}(\flex{1}_B)\rightarrow\flex{1}_A)}
$$
The system continues the MN by:
$$
\scriptsize
\infer[(N)]{\mathit{Negotiate}(B,A)}{\ast(B,A) & B:\flex{1}_B & A:\mathbf{relDis}(B:\flex{1}_B) & A:\flex{1}_A}
$$
Bob receives $\flex{1}_A$ and he accepts it because Alice said they are in relative disagreement.
$$
\scriptsize
\infer[(RD\text{-}Ag)]{B:\mathbf{agree}(A:\flex{1}_A)\cap B:\trad{B}{A}(\flex{1}_A)}
{A:\mathbf{relDis}(B:\flex{1}_B)\cap A:\flex{1}_A}
$$
The system closes the MN by:
$$
\scriptsize
\infer[(A)]{\mathit{Agreement}(A,B)}{\ast(A,B) & A:\flex{1}_A & B:\mathbf{agree}(A:\flex{1}_A)}
$$
with a positive outcome, $\flex{1}_A$.

Fig.~\ref{es:PP} shows the message flow between Alice and Bob
(Fig.~\ref{es:PP-mp}), and the changes of their EGG/YOLK configurations
(Fig.~\ref{es:PP-ey}).
\hfill$\Box$
\end{example}

\begin{figure}[t]
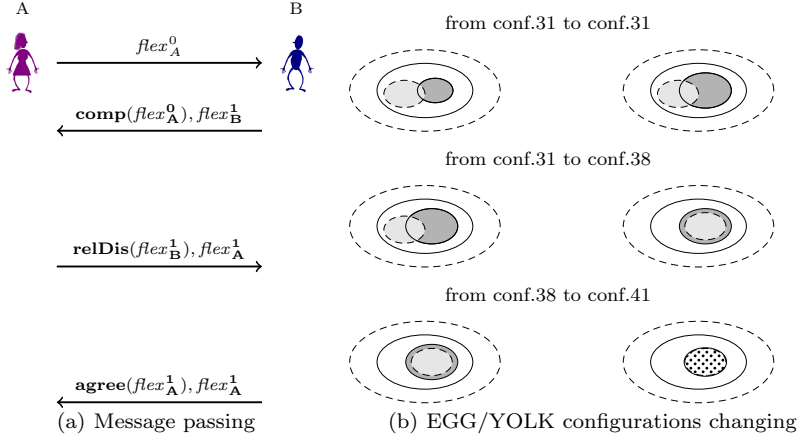
\centering
\subfigure[Message passing]{
\label{es:PP-mp}
\scalebox{0.9}{
\begin{tikzpicture}[level distance=4mm]\scriptsize
\draw (0,22.8) node (alice) {A};
\draw (0,22) node (a) {\includegraphics[width=0.5cm]{alice.jpg}};
\draw (4,22.8) node (bob) {B};
\draw (4,22) node (b) {\includegraphics[width=0.5cm]{bob.jpg}};
\draw[->,thick] (0.5,22) .. controls +(right:0.5cm) and +(left:0.5cm).. node[above,sloped] {$\flex{0}_A$} (3.5,22);
\draw[<-,thick] (0.5,21) .. controls +(right:0.5cm) and +(left:0.5cm).. node[above,sloped] {$\mathbf{comp(\flex{0}_A),\flex{1}_B}$} (3.5,21);
\draw[->,thick] (0.5,19) .. controls +(right:0.5cm) and +(left:0.5cm).. node[above,sloped] {$\mathbf{relDis(\flex{1}_B),\flex{1}_A}$} (3.5,19);
\draw[<-,thick] (0.5,17) .. controls +(right:0.5cm) and +(left:0.5cm).. node[above,sloped] {$\mathbf{agree(\flex{1}_A),\flex{1}_A}$} (3.5,17);
\end{tikzpicture}}
}
\subfigure[{EGG/YOLK configurations changing}]{
\label{es:PP-ey}
\scalebox{0.9}{
\begin{tikzpicture}[level distance=4mm]
\draw (2.8,21) node (31-31) {from conf.31 to conf.31};
\draw[densely dashed] (1,20) ellipse (11mm and 6mm);
\draw[fill = gray!60] (1.15,20) ellipse (2.6mm and 1.8mm);
\draw (1,20) ellipse (7mm and 4mm);
\draw[densely dashed, fill = gray!20] (.7,19.95) ellipse (3mm and 2mm);
\draw (1.15,20) ellipse (2.6mm and 1.8mm);
\draw[densely dashed] (5,20) ellipse (11mm and 6mm);
\draw (5,20) ellipse (7mm and 4mm);
\draw[fill = gray!60] (5.1,20) ellipse (3.8mm and 2.6mm);
\draw[densely dashed, fill = gray!20] (4.7,19.95) ellipse (3mm and 2mm);
\draw (5.1,20) ellipse (3.8mm and 2.6mm);
\draw (2.8,19) node (31-38) {from conf.31 to conf.38};
\draw[densely dashed] (1,18) ellipse (11mm and 6mm);
\draw (1,18) ellipse (7mm and 4mm);
\draw[fill = gray!60] (1.1,18) ellipse (3.8mm and 2.6mm);
\draw[densely dashed, fill = gray!20] (0.7,17.95) ellipse (3mm and 2mm);
\draw (1.1,18) ellipse (3.8mm and 2.6mm);
\draw[densely dashed] (5,18) ellipse (11mm and 6mm);
\draw (5,18) ellipse (7mm and 4mm);
\draw[fill = gray!60] (5.1,18) ellipse (3.8mm and 2.6mm);
\draw[densely dashed, fill = gray!20] (5.1,18) ellipse (3mm and 2mm);
\draw (2.8,17) node (38-41) {from conf.38 to conf.41};
\draw[densely dashed] (1,16) ellipse (11mm and 6mm);
\draw (1,16) ellipse (7mm and 4mm);
\draw[fill = gray!60] (1.1,16) ellipse (3.8mm and 2.6mm);
\draw[densely dashed, fill = gray!20] (1.1,16) ellipse (3mm and 2mm);
\draw[densely dashed] (5,16) ellipse (11mm and 6mm);
\draw (5,16) ellipse (7mm and 4mm);
\draw (5.1,16) ellipse (3.1mm and 2.1mm);
\draw[densely dashed, pattern = crosshatch dots] (5.1,16) ellipse (3mm and 2mm);
\end{tikzpicture}\hspace{1cm}}
}
\caption{The MN scenario of Example~\ref{es:4}: the message passing flow (a) and the changes of the EGG/YOLKs of the agents (b). Alice is identified by plain lines and Bob by dashed lines. White yolks represent the precedent proposal of the agent and the dotted gray yolk is the positive outcome of the scenario.\label{es-fig:PP}}
\end{figure}

The classification of the agreement conditions provided above is
complete, in the sense that there is no other possible configuration of
EGG/YOLKs, as shown by Lehmann and Cohn in (\citeyear{lehmann94}). Based
on the completeness of that analysis, we can show the following results.

\begin{theorem}
\label{consistent}
$\MND$ for bilateral MN is consistent.
\end{theorem}
\begin{proof}
Consider two agents represented in the $\MND$ system with sets $\LS{1}$ and
$\LS{2}$ of stubbornness formulas and sets $\LF{1}$ and $\LF{2}$ of flexible formulas.
To prove that $\MND$ is consistent, we show that if a  $\sign_i$ formula
$\xi$ is inferred using the $\MND$
rules, or, in other terms, is deduced as a theorem in the system, then $\xi$ represents a proposal that is acceptable by both
agents. In other words, we aim at proving that when the rules yield
$\xi$ then $\xi$ generalizes both $\LF{1}$ and $\LF{2}$ and is generalized by both $\LS{1}$ and $\LS{2}$.
To prove this, we need to show that:

\emph{(i)} The rules for making new proposals yield a relation that is
acceptable from the viewpoint of the agent who made the proposal before
and infer a new proposal again still acceptable. In other terms, if an
agent makes a proposal that is generalized by the set of stubbornness
formulas $\LS{i}$, and is a generalization of the set of flexible
formulas $\LF{i}$, for one agent, the rules infer a new proposal that is
in the same relationships with $\LS{i}$ and $\LF{i}$.

\emph{(ii)} The rules for the second proposing agent infer the relation
between the agents at that step of the negotiation.

\emph{(iii)} The rules for the following proposing agent do the same as
the rules for the second proposing agent, taking into account that this
step takes place after the step of the second proposing agent.

\emph{(iv)} The system transition rules close the negotiation only when
the proposal is acceptable by both agents, namely generalizes both
$\LF{i}$ and is generalized by both $\LS{i}$ sets.

Let us now consider a formula $\xi$ that is acceptable by the two agents, and let us consider the
rules that produce transitions in the system. In particular, if $\xi$ is
inferred by means of one of the rules $(AD)$, $(ED)$, $(I)$, $(Co)$,
$(Ag)$ for the second proposing agent, or by means of one of the rules
given in Fig.~\ref{tab:1-1net2} for the following proposing agent, then
the possible results of the step described above are given by the
application of the system transition rules. Evidently, if $\xi$ is
inferred, then the rule $(D)$ does not apply. If $(N)$ applies, and one
more inference is performed, then the rules $(W)$, $(C)$, $(S)$ allow us to infer a different formula.
Suppose now, by contradiction, that the new formula $\xi$ is
not acceptable by one of the agents (in the sense that either is not a generalization of her set of flexible formulas or it is not generalized by the set of stubbornness formulas. As a consequence, one
agent has called herself away, as we stated above. This, however, is
impossible, by construction of the rules for the second and following
proposals. Conversely, if the transition rule $(D)$ applies and,
therefore, the agents have incompatible viewpoints, then $\xi$ is not
inferred through the system, because it is not a generalization of both flexible sets of formulas and generalizes by both stubbornness sets of formulas. Clearly, by means of the full set of rules,
it is not possible to do so when the agents have consistent
viewpoints.
\hfill $\Box$
\end{proof}

We say that a deductive system is adequate to represent a MN between two agents when it infers an outcome iff an agreement is reachable between the agents, otherwise it does not produce any result.
\begin{definition}\label{def:adequate:1-1}
A deductive system $\mathcal{R}$ is \emph{adequate} to represent the MN process between two agents, $i$ and $j$, when
$$
\mathcal{R}\mbox{ infers }
\left\{
\begin{array}{lr}
\varphi & \mbox{ iff for all }x\in\agset\mbox{ there exists } k\in\mathds{N}\mbox{ s.t. }\agstruct{x}{k} \models(\flex{k}_x\rightarrow \varphi) \wedge (\varphi\rightarrow\stub_x)  \\
\perp & \mbox{ otherwise}
\end{array}
\right.
$$
where $\agset = \{i,j\}$ and $\varphi\in\bigcup_{x\in\agset}\L_x$.
\end{definition}

\begin{theorem}
$\MND$ is adequate to represent the MN of two agents.
\end{theorem}
\begin{proof}
We consider two agents that have 
consistent viewpoints, namely such that there exists a possible common angle. Their
stubbornness sets and their flexible sets of formulas are in one of the
EGG/YOLK configurations 
except number 1. Suppose now that the $\MND$ system
infers a $\sign_i$ formula $\xi$. Then, $\xi$ is a common angle.
Conversely, suppose that $\MND$ does not infer any $\sign_i$ formula. Then, the agents are in call-away situation.
Suppose now that the two agents have  inconsistent viewpoints (configuration 1). The
relation established is 
absolute disagreement. The
result is that no formula can be inferred through the system, which is
consistent by Theorem~\ref{consistent}. Hence, overall, the system is
adequate.
\hfill $\Box$
\end{proof}


For MN processes that are built on finite signature theories, we obtain
the following decidability result:
\begin{corollary}
MN is decidable for theories with finite signature under the assumption
of competitive agents.
\end{corollary}
\begin{proof}
Consider an MN between competitive agents on a language with finite
signature. The number of possible proposals the agents can exchange
during a MN process is formed by the possible formulas that can
be built on the signature, which is finite. Since the rules of $\MND$
are finite and the new possible proposals are finite, and the number of applications of each rule is limited to the number of proposals the other negotiator can perform, the number of
steps that will be performed, in any algorithmic solution to the
problem, is finite as well.
\hfill $\Box$
\end{proof}

In the following section, we extend the MND for 1-n MN in which one
agent is the referee.

\subsubsection{MN Rules: 1-n MN}\label{subsubsec:1-nMNRules}

When the Meaning Negotiation involves more than two agents, it may be
viewed as an English Auction Game. An agent behaves differently if she
is the auctioneer or not. The referee is the agent who receives all the
proposals of the others and finds which one is shared by the agents. The
auctioneer is a player himself; he makes a proposal at each new bid. The
auctioneer replicates the same proposal to each of the negotiating
agents. As a 1-1 MN player, the auctioneer evaluates each received
proposal by testing the validity of the conditions listed above: he
checks the relation between each received proposal and his stubbornness
knowledge and his flexible one. An auctioneer differs from the other
negotiating agents in the number of the evaluations he has to do.
Moreover, when the auctioneer infers the next proposal to perform by the
weakening or the changing rules, the proposal may be related in more
ways than that of the proposal made by other agents. In fact, in 1-1 MN
it is not possible to reach the absolute disagreement by a relative
disagreement situation because, when an agent $i$ informs her opponent
$j$ that they are in relative disagreement, then $j$ knows that $i$
proposed one of $i$'s CAF that is a restriction of her CAF and $j$
accepts it. Instead, in 1-n MN, the previous situations may raise: the
auctioneer may not accept the proposal of one of the negotiating agents
who said that they are in relative disagreement, because the proposal is
not shared by the other agents. The set of deductive rules for the
auctioneer is an extension of those in Table~\ref{tab:1-1net2} with
those in Table~\ref{tab:1-nnet2}. In particular, all the added rules are
violations and they represent the changing of the negotiation situation
from relative disagreement to the relations of absolute disagreement,
essence disagreement and compatibility, and from agreement to all the
possible relations between agents: absolute disagreement, essence
disagreement, compatibility, relative disagreement and agreement. The
rule $(Ag\text{-}Ag)$ is not only for the auctioneer but also for
negotiating agents and it is used by them whenever the auctioneer
proposes an acceptable angle.

\begin{table}[t]
\begin{displaymath}	
	\renewcommand{\arraystretch}{3}
\begin{array}{c}
\infer[(Co\text{-}AD)]{i:\mathbf{absDis}(j:\psi)\cap i:\flex{k+1}_i}
{j:\mathbf{comp}(i:\flex{k}_i)\cap j:\psi & \neg(\stub_i \wedge \trad{i}{j}(\psi))}
\\
\infer[(RD\text{-}ED)]{i:\mathbf{essDis}(j:\psi)\cap i:\flex{k+1}_i}
{j:\mathbf{relDis}(i:\flex{k}_i)\cap j:\psi & (\stub_i \vee \trad{i}{j}(\varphi))\wedge\neg(\flex{k+1}_i \wedge \trad{i}{j}(\psi))}\\

\infer[(RD\text{-}Co)]{i:\mathbf{comp}(j:\psi)\cap i:\flex{k+1}_i}
{j:\mathbf{relDis}(i:\flex{k}_i)\cap j:\psi & (\flex{k+1}_i \vee \trad{i}{j}(\psi))\wedge\neg(\flex{k+1}_i\rightarrow\trad{i}{j}(\psi))\wedge\neg(\trad{i}{j}(\psi)\rightarrow\flex{k+1}_i)}
\\
\infer[(RD\text{-}RD)]{i:\mathbf{relDis}(j:\psi)\cap i:\flex{k+1}_i}
{j:\mathbf{relDis}(i:\flex{k}_i)\cap j:\psi & (\flex{k+1}_i \rightarrow \trad{i}{j}(\psi))\wedge\neg(\trad{i}{j}(\psi)\rightarrow\flex{k+1}_i)}\\

\infer[(Ag\text{-}AD)]{i:\mathbf{absDis}(j:\psi)\cap i:\flex{k+1}_i}
{j:\mathbf{agree}(i:\flex{k}_i)\cap j:\psi &  \neg(\stub_i \wedge \trad{i}{j}(\psi))}\\

\infer[(Ag\text{-}ED)]{i:\mathbf{essDis}(j:\psi)\cap i:\flex{k+1}_i}
{j:\mathbf{agree}(i:\flex{k}_i)\cap j:\psi & (\stub_i \vee \trad{i}{j}(\psi))\wedge\neg(\flex{k+1}_i \wedge \trad{i}{j}(\psi))}
\\
\infer[(Ag\text{-}Co)]{i:\mathbf{comp}(j:v)\cap i:\flex{k+1}_i}
{j:\mathbf{agree}(i:\flex{k}_i)\cap j:\varphi & (\flex{k+1}_i \vee \trad{i}{j}(\psi))\wedge\neg(\flex{k+1}_i\rightarrow\trad{i}{j}(\psi))\wedge\neg(\trad{i}{j}(\psi)\rightarrow\flex{k+1}_i)}\\

\infer[(Ag\text{-}RD)]{i:\mathbf{relDis}(j:\psi)\cap i:\flex{k+1}_i}
{j:\mathbf{agree}(i:\flex{k}_i)\cap j:\psi & (\flex{k+1}_i \rightarrow \trad{i}{j}(\psi))\wedge\neg(\trad{i}{j}(\psi)\rightarrow\flex{k+1}_i)}
\\
\infer[(Ag\text{-}Ag)]{i:\mathbf{agree}(j:\psi)\cap i:\trad{i}{j}(\psi)}
{j:\mathbf{agree}(i:\flex{k}_i)\cap j:\psi & (\flex{k+1}_i \rightarrow \trad{i}{j}(\psi))}
\end{array}
\end{displaymath}
\caption[Extension of the rules in Table~\ref{tab:1-1net2} for the auctioneer]{Extension of the rules in Table~\ref{tab:1-1net2} for the auctioneer. All these rules are violations and the rule $(Ag\text{-}Ag)$ may be used also by negotiating agents.}
\label{tab:1-nnet2}
\end{table}
\renewcommand\arraystretch{1}

\renewcommand\arraystretch{1.5}
\begin{table}[t]\centering
\begin{displaymath}
\renewcommand{\arraystretch}{3}
\begin{array}{c}
\infer[(DD)]{\mathit{Disagreement}(a,i_1,\dots,i_n)}
{\ast(a,i_1,\dots,i_n) & {a:\varphi} & \mbox{for all }{i \in\agset_1.i:\mathbf{agree}(a:\varphi)} &  \mbox{for all }j \in \agset_2.j:\mathbf{na}(a:\varphi) \\ & \mid\agset_1\mid\leq\alpha & \mbox{ for all }{ i \in\agset.i:\psi \mbox{ and }\stub_i\leftrightarrow\psi}}
\\
\infer[(AA)]{\mathit{Agreement}(a,i_1,\dots,i_n)}{\ast(a,i_1,\dots,i_n) & {a:\varphi} & \mbox{for all }{i \in\agset_1.i:\mathbf{agree}(a:\varphi)} & \\ \mbox{for all }{j \in \agset_2.j:\mathbf{na}(a:\varphi)} & \mid\agset_1\mid\geq\alpha}
\\
\infer[(NN)]{\mathit{Negotiate}(a,i_1,\dots,i_n)}{\ast(a,i_1,\dots,i_n) & {a:\varphi} & \mbox{for all }{i \in\agset_1.i:\mathbf{agree}(a:\varphi)}& \mbox{for all }{j \in \agset_2.j:\mathbf{na}(a:\varphi)} & \mid\agset_1\mid\leq\alpha}
\end{array}
\end{displaymath}
\caption{1-n MN system transition rules.}
    \label{1-nsts}
\end{table}
\renewcommand\arraystretch{1}

Moreover, the system transition rules are different from the 1-1 MN ones
(Table~\ref{1-1sts}) because the agreement and disagreement conditions
are different. In particular, the test of the agreement condition needs
to count the number of agreeing agents. The 1-n MN ends in:
\begin{itemize}
\item \emph{disagreement} when all the agents involved are in
stubbornness and no agreement is found yet;
\item \emph{agreement} when all the agents or an acceptable part of
them, i.e. $\alpha$ agents where $\alpha$ is the degree of sharing,
agree about a common angle.
\end{itemize}

In all the other cases, the MN continues. The system transition rules
for 1-n MN are in Table~\ref{1-nsts}.\fix{Luca}{regole da spiegare
meglio. Sono anche da riscrivere, magari in landscape, in modo che si
capiscano bene le premesse. Bisogna anche spiegare che non \`e un
sistema infinitario, visto che gli insiemi di agenti sono (dovrebbero
essere?) finiti. Matteo, per favore spiega il limite ad $Ag_2$. {\bf Matteo} Anche qui, faccio la prox settimana} 
The
following example shows a simple 1-n MN ending negatively or
positively depending on the sharing degree decided in front of the
beginning of the MN.

\begin{example}\label{es:5}
Let Alice, Bob and Charles be three agents negotiating the definition of the term ``vehicle''.
Suppose that the initial viewpoint of Alice is
$$\flex{0}_A=\mathit{has3wheels} \wedge \mathit{hasSteeringWheel} \wedge \mathit{hasMotor}$$
and her stubbornness knowledge is
\begin{eqnarray*}
\stub_A&=&(\mathit{has3wheels} \vee \mathit{has4wheels})\wedge \mathit{hasSteeringWheel} \\
&\ & \wedge (\mathit{hasMotor}\vee \mathit{has2bicyclePedals} \vee \mathit{has4bicyclePedals})
\end{eqnarray*}
Suppose that Bob's initial viewpoint is
$$\flex{0}_B=\mathit{has2wheels} \wedge \mathit{hasHandlebar}\wedge \mathit{has2bicyclePedals}$$
 and that his stubbornness knowledge is
\begin{eqnarray*}
\stub_B&=&(\mathit{has2wheels} \vee \mathit{has4wheels})\wedge (\mathit{hasSteeringWheel} \vee \mathit{hasHandlebar})\wedge\\
&\ & (\mathit{hasMotor}\vee \mathit{has2bicyclePedals} \vee \mathit{has4bicyclePedals})
\end{eqnarray*}
and that the initial viewpoint of Charles is
$$\flex{0}_C=\mathit{has4wheels} \wedge \mathit{hasHandlebar} \wedge \mathit{hasTowBar}$$
and his stubbornness knowledge is
\begin{eqnarray*}
\stub_{C}&=&\mathit{has4wheels} \wedge (\mathit{hasHandlebar} \vee \mathit{hasSteeringWheel})\wedge\\
        &\ &(\mathit{hasTowBar} \vee \mathit{hasMotor})
\end{eqnarray*}

Moreover, suppose that the MN is considered positive iff all the agents agree with a common angle, i.e there are $\alpha=3$ agreeing agents.
Alice is the first bidding agent, and thus she is the referee, and she proposes $\flex{0}_A$ to Bob and Charles.
Bob and Charles receive the proposal and evaluate it. Bob tests that they are in essence disagreement because $(\stub_B\vee \trad{B}{A}(\flex{0}_A))\wedge \neg(\flex{0}_B \wedge \trad{B}{A}(\flex{0}_A ))$. Bob chooses the new CAF by a changing theory action ($C$) in
\begin{eqnarray*}
\flex{1}_B &=&(\mathit{has2wheels} \vee \mathit{has4wheels})\wedge \\
 & \ &\mathit{hasSteeringWheel} \wedge (\mathit{has2bicyclePedals}\vee\mathit{hasMotor})
\end{eqnarray*}
Bob uses the ($ED$) rule and sends his CAF to Alice:
$$
\scriptsize
\infer[(ED)]{B:\mathbf{essDis}(A:\flex{0}_A)\cap B:\flex{1}_B}
{A:\flex{0}_A & (\stub_B \vee \trad{B}{A}(\flex{0}_A))\wedge \neg(\flex{0}_B\wedge \trad{B}{A}(\flex{0}_A))}
$$
Charles evaluates $\flex{0}_A$ and  tests that they are in essence disagreement because $(\stub_C\vee\trad{C}{A}(\flex{0}_A))\wedge\neg(\flex{0}_C \wedge \trad{C}{A}(\flex{0}_A))$. Charles chooses the new CAF by a changing theory action ($C$) in
$$
\flex{1}_C = \mathit{has4wheels} \wedge \mathit{hasSteeringWheel} \wedge \mathit{hasTowBar}
$$
Charles uses the ($ED$) rule and sends his CAF to Alice:
$$
\scriptsize
\infer[(ED)]{C:\mathbf{essDis}(A:\flex{0}_A)\cap C:\flex{1}_C}
{A:\flex{0}_A & (\stub_C\vee \trad{C}{A}(\flex{0}_A))\wedge \neg(\flex{0}_C\wedge \trad{C}{A}(\flex{0}_A))}
$$

The system continues the MN by:
$$
\scriptsize
\infer[(NN)]{\mathit{Negotiate}(A,B,C)}{\ast(A,B,C) & A:\flex{0}_A & B:\mathbf{essDis}(A:\flex{0}_A) &  C:\mathbf{essDis}(A:\flex{0}_A) & |\{A\}|\leq\alpha}
$$
Alice receives $\flex{1}_B$ and $\flex{1}_C$, and she has to make a weakening or a changing theory action because Bob and Charles did not say they were in agreement nor in relative disagreement. Alice performs a changing theory action by the rule ($C$) and her CAF is
$$
\flex{1}_A = \mathit{has4wheels} \wedge \mathit{hasSteeringWheel}\wedge (\mathit{hasMotor} \vee \mathit{has2bicyclePedals})
$$
Alice thinks she is in relative disagreement relation with Bob since $(\flex{1}_A\rightarrow\trad{A}{B}(\flex{1}_B))\wedge \neg (\trad{A}{B}(\flex{1}_B)\rightarrow\flex{1}_A)$.

Alice thinks she is in compatibility relation with Charles since $(\flex{1}_A\vee\trad{A}{C}(\flex{1}_C))\wedge(\flex{1}_A\rightarrow\trad{A}{C}(\flex{1}_C))\wedge \neg (\trad{A}{C}(\flex{1}_C)\rightarrow\flex{1}_A)$.

Alice uses the rule ($ED\text{-}RD$) to inform Bob they are in relative disagreement and ($ED\text{-}Co$) to inform Charles that they are in compatibility:
$$
\scriptsize
\infer[(ED\text{-}RD)]{A:\mathbf{relDis}(B:\flex{1}_B)\cap A:\flex{1}_A}
{\begin{array}{c}B:\mathbf{essDis}(A:\flex{0}_A)\cap B:\flex{1}_B\\(\flex{1}_A\rightarrow \trad{A}{B}(\flex{1}_B))\wedge \neg (\trad{A}{B}(\flex{1}_B)\rightarrow\flex{1}_A)\end{array}}
$$
and
$$
\scriptsize
\infer[(ED\text{-}Co)]{A:\mathbf{essDis}(C:\flex{1}_C)\cap A:\flex{1}_A}
{\begin{array}{c}C:\mathbf{comp}(A:\flex{0}_A)\cap C:\stub_C \\ (\flex{1}_A\vee \trad{A}{C}(\flex{1}_C))\wedge\neg(\flex{1}_A\rightarrow \trad{A}{C}(\flex{1}_C))\wedge \neg (\trad{A}{C}(\flex{1}_C)\rightarrow\flex{1}_A)\end{array}}
$$
Bob receives $\flex{1}_A$ and he is in agreement with Alice by ($RD\text{-}Ag$):
$$
\scriptsize
\infer[(RD\text{-}Ag)]{B:\mathbf{agree}(A:\flex{1}_A)\cap B:\trad{B}{A}(\flex{1}_A)}
{A:\mathbf{relDis}(B:\flex{1}_B)\cap A:\flex{1}_A }
$$
Charles receives the proposal and evaluates it. Charles chooses the new CAF by a weakening action ($C$ rule) in
$$
\flex{2}_C = \mathit{has4wheels} \wedge (\mathit{hasHandlebar}\vee\mathit{hasSteeringWheel})\wedge (\mathit{hasMotor} \vee \mathit{hasTowbar})
$$
and $\flex{2}_C = \stub_C$. Charles tests that he is in compatibility relation with Alice since $(\stub_C\vee \trad{C}{A}(\flex{1}_A))\wedge \neg(\stub_C\rightarrow \trad{C}{A}(\flex{1}_A))\wedge \neg( \trad{C}{A}(\flex{1}_A)\rightarrow\stub_C)$ and uses the ($Co\text{-}Co$) rule to send his CAF to Alice:
$$
\scriptsize
\infer[(Co\text{-}Co)]{C:\mathbf{comp}(A:\flex{1}_A)\cap C:\stub_C}
{\begin{array}{c}A:\mathbf{comp}(C:\stub_C)\cap A:\flex{1}_A \\ (\stub_C\vee \trad{C}{A}(\flex{1}_A))\wedge \neg(\stub_C\rightarrow \trad{C}{A}(\flex{1}_A))\wedge \neg( \trad{C}{A}(\flex{1}_A)\rightarrow\stub_C)
\end{array}}
$$
The system continues the MN by:
$$
\scriptsize
\infer[(NN)]{\mathit{Negotiate}(A,B,C)}{\ast(A,B,C) & A:\flex{0}_A & B:\mathbf{agree}(A:\flex{1}_A) & C:\mathbf{comp}(A:\flex{1}_A) & |\{A,B\}|\leq\alpha}
$$
Alice receives $\trad{B}{A}(\flex{1}_A)$ and $\stub_C$ from Bob and Charles respectively and she has to make a weakening or a changing theory action because Charles did not say he was in agreement nor in relative disagreement with her. Alice performs a changing theory action by the rule ($C$) and her CAF is
$$
\flex{2}_A = \mathit{has4wheels}\wedge \mathit{hasSteeringWheel}\wedge (\mathit{has4bicyclePedals}\vee\mathit{hasMotor})
$$
Alice thinks she is in compatibility relation with Bob and Charles since $(\flex{2}_A\vee\trad{A}{B}(\trad{B}{A}(\flex{1}_A)))\wedge \neg(\flex{2}_A\rightarrow \trad{A}{B}(\trad{B}{A}(\flex{1}_A))) \wedge \neg (\trad{A}{B}(\trad{B}{A}(\flex{1}_A)) \rightarrow \flex{2}_A)$ and
$(\flex{2}_A\vee \trad{A}{C}(\stub_C))\wedge\neg(\flex{2}_A\rightarrow \trad{A}{C}(\stub_C))\wedge \neg (\trad{A}{C}(\stub_C)\rightarrow \flex{2}_A)$.

Alice uses the rule ($Ag\text{-}Co$) to inform Bob she is in compatibility with him:
$$
\scriptsize
\infer[(Ag\text{-}Co)]{A:\mathbf{comp}(B:\trad{B}{A}(\flex{1}_A))\cap A:\trad{A}{B}(\trad{B}{A}(\flex{1}_A))}
{\begin{array}{c}B:\mathbf{agree}(A:\flex{1}_A)\cap B:\trad{B}{A}(\flex{1}_A) \\ (\flex{2}_A\vee\trad{A}{B}(\trad{B}{A}(\flex{1}_A)))\wedge \neg(\flex{2}_A\rightarrow \trad{A}{B}(\trad{B}{A}(\flex{1}_A))) \wedge \neg (\trad{A}{B}(\trad{B}{A}(\flex{1}_A)) \rightarrow \flex{2}_A)\end{array}}
$$
and the rule ($Co\text{-}Co$) to inform Charles she is in compatibility with him:
$$
\scriptsize
\infer[(Co\text{-}Co)]{A:\mathbf{comp}(C:\stub_C)\cap A:\flex{2}_A}
{\begin{array}{c}C:\mathbf{comp}(A:\flex{1}_A)\cap C:\stub_C  \\ (\flex{2}_A\vee \trad{A}{C}(\stub_C))\wedge\neg(\flex{2}_A\rightarrow \trad{A}{C}(\stub_C))\wedge \neg (\trad{A}{C}(\stub_C)\rightarrow \flex{2}_A)\end{array}}
$$
Bob receives $\flex{2}_A$ and he makes a weakening action by the rule $W$ and his CAF is:
\begin{eqnarray*}
\flex{3}_B&=&(\mathit{has2wheels} \vee \mathit{has4wheels})\wedge (\mathit{hasSteeringWheel} \vee \mathit{hasHandlebar})\wedge\\
&\ & (\mathit{hasMotor}\vee \mathit{has2bicyclePedals} \vee \mathit{has4bicyclePedals})
\end{eqnarray*}
and $\flex{3}_B=\stub_B$. Bob uses the rule ($Co\text{-}Co$) and sends $\stub_B$ to Alice.
$$
\scriptsize
\infer[(Co\text{-}Co)]{B:\mathbf{comp}(A:\flex{2}_A)\cap B:\stub_B}
{\begin{array}{c}A:\mathbf{comp}(B:\flex{2}_B)\cap A:\flex{2}_A  \\ (\stub_B\vee \trad{B}{A}(\flex{2}_A))\wedge\neg(\stub_B\rightarrow \trad{B}{A}(\flex{2}_A))\wedge \neg (\trad{B}{A}(\flex{2}_A)\rightarrow \stub_B)\end{array}}
$$

Charles receives the proposal and evaluates it. Charles is in stubbornness thus he applies the ($S$) rule and proposes $\stub_C$ to Alice.

Charles uses the ($Co\text{-}Co$) rule and sends his CAF to Alice:
$$
\scriptsize
\infer[(Co\text{-}Co)]{C:\mathbf{comp}(A:\flex{2}_A)\cap C:\stub_C}
{\begin{array}{c}A:\mathbf{comp}(C:\stub_C)\cap A:\flex{2}_A \\ (\stub_C\vee \trad{C}{A}(\flex{2}_A))\wedge \neg(\stub_C\rightarrow \trad{C}{A}(\flex{2}_A))\wedge \neg( \trad{C}{A}(\flex{2}_A)\rightarrow\stub_C)
\end{array}}
$$
The system continues the MN by:
$$
\scriptsize
\infer[(NN)]{\mathit{Negotiate}(A,B,C)}{\ast(A,B,C) & A:\flex{2}_A & B:\mathbf{comp}(A:\flex{2}_A) & \\ C:\mathbf{comp}(A:\flex{2}_A) & |\{A\}|\leq\alpha}
$$
Alice receives $\stub_B$ and $\stub_C$ and she has to make a weakening or a changing theory action because Charles did not say he was in agreement nor in relative disagreement with her. Alice performs a weakening action by the rule ($W$) and her CAF is
\begin{eqnarray*}
\flex{3}_A&=&(\mathit{has3wheels} \vee \mathit{has4wheels})\wedge \mathit{hasSteeringWheel} \\
&\ & \wedge (\mathit{hasMotor}\vee \mathit{has2bicyclePedals} \vee \mathit{has4bicyclePedals})
\end{eqnarray*}
and $\flex{3}_A=\stub_A$.

Alice thinks she is in compatibility relation with Bob since $(\stub_A\vee\trad{A}{B}(\stub_B))\wedge(\stub_A\rightarrow\trad{A}{B}(\stub_B))\wedge \neg (\trad{A}{B}(\stub_B)\rightarrow\stub_A)$.

Moreover, Alice thinks she is in compatibility relation with Charles
since
$(\stub_A\vee\trad{A}{C}(\stub_C))\wedge(\stub_A\rightarrow\trad{A}{C}(\stub_C))\wedge
\neg (\trad{A}{C}(\stub_C)\rightarrow\stub_A)$.

Alice uses the ($Co\text{-}Co$) rule to inform Bob and Charles that they are in compatibility relation:
$$
\scriptsize
\infer[(Co\text{-}Co)]{A:\mathbf{comp}(B:\stub_B)\cap A:\stub_A}
{\begin{array}{c}B:\mathbf{comp}(A:\flex{2}_A)\cap B:\stub_B \\ (\stub_A\vee\trad{A}{B}(\stub_B))\wedge \neg(\stub_A\rightarrow \trad{A}{B}(\stub_B)) \wedge \neg (\trad{A}{B}(\stub_B) \rightarrow \stub_A)\end{array}}
$$
and
$$
\scriptsize
\infer[(Co\text{-}Co)]{A:\mathbf{comp}(C:\stub_C)\cap A:\stub_A}
{\begin{array}{c}C:\mathbf{comp}(A:\flex{2}_A)\cap C:\stub_C \\ (\stub_A\vee\trad{A}{C}(\stub_C))\wedge \neg(\stub_A\rightarrow \trad{A}{C}(\stub_C)) \wedge \neg (\trad{A}{C}(\stub_C) \rightarrow \stub_A)\end{array}}
$$
Bob and Charles receive the proposal and evaluate it. They are in stubbornness thus they apply the ($S$) rule and they propose $\stub_B$ and $\stub_C$ to Alice respectively.

They use the ($Co\text{-}Co$) rule and send their CAFs to Alice:
$$
\scriptsize
\infer[(Co\text{-}Co)]{B:\mathbf{comp}(A:\stub_A)\cap B:\stub_B}
{\begin{array}{c}A:\mathbf{comp}(B:\stub_B)\cap A:\stub_A \\ (\stub_B\vee \trad{B}{A}(\stub_A))\wedge \neg(\stub_B\rightarrow \trad{B}{A}(\stub_A))\wedge \neg( \trad{B}{A}(\stub_A)\rightarrow\stub_B)
\end{array}}
$$
and
$$
\scriptsize
\infer[(Co\text{-}Co)]{C:\mathbf{comp}(A:\stub_A)\cap C:\stub_C}
{\begin{array}{c}A:\mathbf{comp}(C:\stub_C)\cap A:\stub_A \\ (\stub_C\vee \trad{C}{A}(\stub_A))\wedge \neg(\stub_C\rightarrow \trad{C}{A}(\stub_A))\wedge \neg( \trad{C}{A}(\stub_A)\rightarrow\stub_C)
\end{array}}
$$
The system closes the MN by:
$$
\scriptsize
\infer[(DD)]{\mathit{Disagreement}(A,B,C)}{\begin{array}{cccc} \ast(A,B,C) & \ A:\stub_A & B:\mathbf{comp}(A:stub_A) & C:\mathbf{agree}(A:\stub_A) \\ \ & |\{A\}|\leq\alpha & \mbox{ for all }{i \in\agset.i:\psi \mbox{ and }\stub_i\leftrightarrow\psi} & \ \end{array}}
$$
with a negative outcome. The negotiation would be positively ending by rule $(AA)$ whether the sharing degree were $\alpha=2$ when Alice proposes $A:\flex{1}_A$ to Bob and Charles.
\hfill$\Box$
\end{example}

As for the rules for bilateral MN, we can show the consistency and adequateness of the rules for $1\text{-}n$ MN.\fix{Luca}{Uniformare: o $1\text{-}n$ MN o multiparty MN. {\bf Matteo} Fatto, preferisco 1-n}
\begin{theorem}
\label{1-n:consistent}
$\MND$ for 1-n MN is consistent.
\end{theorem}
\begin{proof}
Consider $n$ agents represented in the $\MND$ system with sets $\LS{1},\dots,\LS{n}$ of stubbornness formulas and sets $\LF{1},\dots,\LF{n}$ of flexible formulas.
The proof of the consistency of $\MND$ for 1-n MN is similar to the proof of Theorem~\ref{consistent}. We show that if a  $\sign_i$ formula
$\xi$ is inferred using the $\MND$
$1\text{-}n$ rules, or, in other terms, is deduced as a theorem in the system, then $\xi$ represents a proposal that is acceptable by all or at least $\alpha$ agents. In other words, we aim at proving that when the rules yield
$\xi$ then $\xi$ generalizes at least $\alpha$ languages among $\LF{1},\dots,\LF{n}$ and is generalized by at least $\alpha$ languages among  $\LS{1},\dots,\LS{n}$.
Let us now consider a formula $\xi$ that is acceptable by at least $\alpha$ agents, and let us consider the
rules that produce transitions in the system. In particular, if $\xi$ is
inferred by means of one of the rules $(AD)$, $(ED)$, $(I)$, $(Co)$,
$(Ag)$ for the second proposing agent, or by means of one of the rules
given in Fig.~\ref{tab:1-1net2} or in Fig.~\ref{tab:1-nnet2} for the following proposing agent and the auctioneer, then
the possible results of the step described above are given by the
application of the system transition rules. Evidently, if $\xi$ is
inferred, then the rule $(DD)$ does not apply. If $(NN)$ applies, and one
more inference is performed, then the rules $(W)$, $(C)$, $(S)$ allow us to infer a different formula.
Suppose now, by contradiction, that the new formula $\xi$ is
not acceptable by more than $n-\alpha$ agents (in the sense that either it is not a generalization of their set of flexible formulas or it is not generalized by their sets of stubbornness formulas). As a consequence, these
agents have called themselves away, as we stated above. This, however, is
impossible, by construction of the rules for the second and following
proposals, and for the auctioneer. Conversely, if the transition rule $(DD)$ applies and,
therefore, more than $n-\alpha$ agents have incompatible viewpoints, then $\xi$ is not
inferred through the system, because it is not a generalization of the flexible sets of formulas and generalizes by the stubbornness sets of formulas of at least $\alpha$ agents. Clearly, by means of the full set of rules,
it is not possible to do so when the agents have consistent
viewpoints.
\hfill $\Box$
\end{proof}

We say that a deductive system is adequate to represent a MN between $n$ agents when it infers an outcome iff an agreement is reachable among the agents otherwise it does not produce any result.
\begin{definition}\label{def:adequate:1-n}
A deductive system $\mathcal{R}$ is \emph{adequate} to represent the MN process among $n$ agents when
$$
\mathcal{R}\mbox{ infers }
\left\{
\begin{array}{lr}
\varphi & \mbox{ iff there exists }\agset'\subseteq\agset\mbox{ s.t. }\mid\agset'\mid = m\geq\alpha \mbox{ and for all }i\in\agset'\\
\ & \mbox{there exists } k\in\mathds{N}\mbox{ s.t. }\agstruct{i}{k} \models(\flex{k}_i\rightarrow \varphi) \wedge (\varphi\rightarrow\stub_i)  \\
\perp & \mbox{ otherwise}
\end{array}
\right.
$$
where $\agset$ is the set of the agents with $\mid\agset\mid=n$, $\varphi\in\bigcup_{i\in\agset}\L_i$ and $\alpha$ is the minimum number of agreeing agents to consider positive the outcome of the MN.
\end{definition}

\begin{theorem}\label{1-n:adequateness}
$\MND$ is adequate to represent the MN of $n$ agents.
\end{theorem}
\begin{proof}
We consider $n$ agents of which $m$ agents have 
consistent viewpoints where $n\geq m \geq \alpha$, namely such that there exists a possible common angle. Their
stubbornness sets and their flexible sets of formulas are pairwise in one of the
EGG/YOLK configurations except number 1. Suppose now that the $\MND$ system
infers a $\sign_i$ formula $\xi$. Then, $\xi$ is a common angle for at least $\alpha$ agents.
Conversely, suppose that $\MND$ does not infer any $\sign_i$ formula. Then, some of the $m$ agents are in call-away situation with the received proposal.
Suppose now that there are more than $n\text{-}\alpha$ agents having pairwise inconsistent viewpoints (configuration 1). The relation established is absolute disagreement. The
result is that no formula can be inferred through the system, which is
consistent by Theorem~\ref{1-n:consistent}. Hence, overall, the system is
adequate.
\hfill $\Box$
\end{proof}

\subsection{MN Process Development}\label{subsec:MNdevelop}

In this section, we show how the MN process develops in terms of the
changing of the relations between the EGG/YOLKs of the agents. We model
the multiparty MN among $n$ agents as an English Auction Game in which
the auctioneer negotiates simultaneously with $n-1$ agents so that it
can be considered as $n-1$ bilateral MNs. For this reason we show here
how the relation of the EGG/YOLKs changes during a MN process only for
bilateral MN. As said above, the stubborn knowledge of the agents never
changes during the negotiation but only the flexible part may differ
from one step to the next one of a MN. The evolution of the relations
between the EGG/YOLKs is different when the stubborn knowledge of the
agents are inconsistent, or in a generalization relation, or just
consistent or equivalent. We show below the MN development in all the
cases.

Agents in MN make offers $\flex{}$\fix{Luca}{``make bid $\flex{}$'' suona un po' strano} such that for each agent $i$:
$$
\flex{}_i = \stub_i \wedge \varphi
$$
where $\stub_i$ is the stubbornness knowledge formula and $\varphi$ is the flexible part of $\flex{}_i$. Whenever an agent receive the opponent proposal, she does not know which is its flexible part and which the stubbornness one. She only knows which is the relation of the received proposal with respect to her own stubbornness and flexible knowledge.

Only the supervisor system knows the stubbornness knowledge of all the
agents. As said above, the stubbornness knowledge is unquestionable and
it never changes during the negotiation. Therefore, the MN process can
be represented as a path of a graph in which nodes are the EGG/YOLK
configurations and edges are the result of the usage of a bidding rule
(Table~\ref{tab:net1} and Table~\ref{tab:1-1net2}).

Suppose that the agents have inconsistent stubbornness knowledge;
whatever the deductive rules they use, they remain related as in
configuration number 1 and the knowledge\fix{Luca}{Inglese. A cosa si riferisce
``it''?} is described by a logical formula involving the stubbornness
formulas of the agents (see Table~\ref{tab:confDR}).

\begin{table}[t]\centering
\begin{tabular}{l l}
\hline
\textbf{Configuration} & \textbf{Formula} \\
\hline
\begin{minipage}{3cm}
\ \\
\scalebox{0.8}{
\normalsize{1} \\
\begin{tikzpicture}[level distance=4mm]
\draw (0,0) ellipse (7mm and 4mm);
\draw[fill = gray!60] (0,0) ellipse (3mm and 2mm);
\draw[densely dashed] (1.5,0) ellipse (7mm and 4mm);
\draw[densely dashed, fill = gray!20] (1.5,0) ellipse (3mm and 2mm);
\end{tikzpicture}
}
\\
\end{minipage} &
\begin{minipage}{8cm}
$\neg(\stub_i \wedge \stub_j)$
\end{minipage}\\
\end{tabular}
\caption{Configurations for inconsistent stubbornness knowledge. Agent $i$ is identified by plain lines and agent $j$ by dashed lines.}\label{tab:confDR}
\end{table}

In the following subsections, we describe all the MN situations with
respect to the relations of the stubbornness knowledge of the agents.

\subsubsection{Equivalent Stubbornness Knowledge}
Suppose that the agents have equivalent stubbornness sets. Then
$$
\stub_i \leftrightarrow \stub_j
$$

The relations between two sets are defined by means of the intersection of their interiors and their exteriors. Given two sets, there are, in theory, two relations between their interiors (either the set intersection is empty or not), two between their exteriors and two relations between the interior of one of them and the exterior of the other one. The possible configurations are then eight, but some of them are absurd (for instance when the intersection of the interiors in empty, the intersection of the exteriors cannot be empty, and vice versa. This presentation issues have been studied deeply in the past and summarised in the spatial reasoning framework known as the Region Connection Calculus (RCC-5). This calculus provides five relations for the cases in which the sets coincide (EQ), two order relations of proper part (PP and PP$^{-1}$), the relation of proper overlapping (PO) and the relation of dijointness (DR).
The equivalence relation between stubbornness theories  relations (RCC5) by
EQ.\fix{Luca}{Improvvisamente, si parla di RCC senza alcuna
introduzione o spiegazione. Possiamo togliere RCC ovunque? {\bf Matteo} Ho provato, ma non si capivano le figure. Ho preferito questo piccolo lavoretto di inserzione.}

In Table~\ref{confEQ}, we show the possible yolk configurations and we
give a statement representing the configuration, i.e. the negotiation
state.
\begin{center}

\begin{longtable}{ll}
\hline \textbf{Configuration} & \textbf{Formula} \\ \hline
\endfirsthead

\multicolumn{2}{c}%
{{\bfseries \tablename\ \thetable{} -- continued from previous page}} \\
\hline \textbf{Configuration} & \textbf{Formula} \\ \hline
\endhead

\hline \multicolumn{2}{r}{{Continued on next page}} \\ \hline
\endfoot

\endlastfoot

\hline
\begin{minipage}{3cm}
\ \\
\normalsize{42a} \\
\scalebox{0.6}{
\begin{tikzpicture}[level distance=4mm]
\draw[densely dashed] (0.5,0) ellipse (12mm and 7mm);
\draw (0.5,0) ellipse (11.8mm and 6.8mm);
\draw[densely dashed, fill = gray!20] (0.1,0) ellipse (3mm and 2mm);
\draw[fill = gray!60] (0.9,0) ellipse (3mm and 2mm);
\end{tikzpicture}
}
\\
\end{minipage} &
\begin{minipage}{8cm}
$(\stub_i \leftrightarrow \stub_j) \wedge \neg(\flex{}_i \wedge \flex{}_j)\wedge(\flex{}_i \rightarrow \stub_j)\wedge (\flex{}_j \rightarrow \stub_i)$
\end{minipage}\\
\begin{minipage}{3cm}
\ \\
\normalsize{42b}\\
\scalebox{0.6}{
\begin{tikzpicture}[level distance=4mm]
\draw[densely dashed] (0.5,0) ellipse (12mm and 7mm);
\draw (0.5,0) ellipse (11.8mm and 6.8mm);
\draw[densely dashed, fill = gray!20] (0.3,0) ellipse (3mm and 2mm);
\draw[fill = gray!60] (0.7,0) ellipse (3mm and 2mm);
\draw[densely dashed] (0.3,0) ellipse (3mm and 2mm);
\end{tikzpicture}
}
\\
\end{minipage} &
\begin{minipage}{8cm}
$(\stub_i \leftrightarrow \stub_j) \wedge (\flex{}_i \vee \flex{}_j)\wedge\neg(\flex{}_i \rightarrow \flex{}_j)\wedge \neg(\flex{}_j \rightarrow \flex{}_i)\wedge(\flex{}_i \rightarrow \stub_j)\wedge (\flex{}_j \rightarrow \stub_i)$
\end{minipage}\\
\begin{minipage}{3cm}
\ \\
\normalsize{42c} \\
\scalebox{0.6}{
\begin{tikzpicture}[level distance=4mm]
\draw[densely dashed] (0.5,0) ellipse (12mm and 7mm);
\draw (0.5,0) ellipse (11.8mm and 6.8mm);
\draw[densely dashed, fill = gray!20] (0.5,0) ellipse (7mm and 4mm);
\draw[fill = gray!60] (0.5,0) ellipse (3mm and 2mm);
\end{tikzpicture}
}
\\
\end{minipage} &
\begin{minipage}{8cm}
$(\stub_i \leftrightarrow \stub_j) \wedge (\flex{}_i \rightarrow \flex{}_j)\wedge(\flex{}_i \rightarrow \stub_j)\wedge (\flex{}_j \rightarrow \stub_i)$
\end{minipage}\\
\begin{minipage}{3cm}
\ \\
\normalsize{42d}\\
\scalebox{0.6}{
\begin{tikzpicture}[level distance=4mm]
\draw[densely dashed] (0.5,0) ellipse (12mm and 7mm);
\draw (0.5,0) ellipse (11.8mm and 6.8mm);
\draw[fill = gray!60] (0.5,0) ellipse (7mm and 4mm);
\draw[densely dashed, fill = gray!20]  (0.5,0) ellipse (3mm and 2mm);
\end{tikzpicture}
}
\\
\end{minipage} &
\begin{minipage}{8cm}
$(\stub_i \leftrightarrow \stub_j) \wedge (\flex{}_j \rightarrow \flex{}_i)\wedge(\flex{}_i \rightarrow \stub_j)\wedge (\flex{}_j \rightarrow \stub_i)$
\end{minipage}\\
\begin{minipage}{3cm}
\ \\
\normalsize{42e} \\
\scalebox{0.6}{
\begin{tikzpicture}[level distance=4mm]
\draw[densely dashed] (0.5,0) ellipse (12mm and 7mm);
\draw (0.5,0) ellipse (11.8mm and 6.8mm);
\draw[densely dashed, fill = gray!20] (0.5,0) ellipse (7mm and 4mm);
\draw[fill = gray!60] (0.5,0) ellipse (6.8mm and 3.8mm);
\end{tikzpicture}
}
\end{minipage} &
\begin{minipage}{8cm}
$(\stub_i \leftrightarrow \stub_j) \wedge (\flex{}_i \leftrightarrow \flex{}_j)\wedge(\flex{}_i \rightarrow \stub_j)\wedge (\flex{}_j \rightarrow \stub_i)$
\end{minipage}\\
\caption{Configurations for equivalent stubbornness sets. Agent $i$ is identified by plain lines and agent $j$ by dashed lines.}\label{confEQ}
\end{longtable}
\end{center}

Figure~\ref{grafoEQ} depicts the graph of the possible negotiation
relations of the agents during the negotiation. The nodes are the
EGG/YOLK configurations and the edges are coloured by the agent who
makes the next bid. The gray node identifies the positive outcome of the
negotiation.

\begin{figure}[t]\centering
\scalebox{0.8}{
\begin{tikzpicture}[level distance=4mm]\centering
\draw (2,0) node[draw, shape = circle, fill = black!10] (e) {42e};
\draw (4,2) node[draw, shape = circle] (d) {42d};
\draw (0,2) node[draw, shape = circle] (c) {42c};
\draw (2,3.5) node[draw, shape = circle] (b) {42b};
\draw (2,6) node[draw, shape = circle] (a) {42a};
\path[->,thick] (a) edge [loop above] node {} (a);
\path[->,thick,densely dashed] (a) edge [loop left] node {} (a);
\path[->,thick] (b) edge [loop left] node {} (b);
\path[->,thick,densely dashed] (b) edge [loop right] node {} (b);
\path[->,thick] (a) edge [bend right = 7] node {} (b);
\path[->,thick,densely dashed] (a) edge [bend left = 7] node {} (b);
\path[->,thick] (a) edge [bend right = 5] node {} (c);
\path[->,thick,densely dashed] (a) edge [bend left = 5] node {} (c);
\path[->,thick] (a) edge [bend right = 5] node {} (d);
\path[->,thick,densely dashed] (a) edge [bend left = 5] node {} (d);
\path[->,thick] (a) edge [bend right = 15] node {} (e);
\path[->,thick,densely dashed](a) edge [bend left = 15] node {} (e);
\path[->,thick] (b) edge [bend right = 5] node {} (c);
\path[->,thick,densely dashed] (b) edge [bend left = 5] node {} (c);
\path[->,thick] (b) edge [bend right = 5] node {} (d);
\path[->,thick,densely dashed] (b) edge [bend left = 5] node {} (d);
\path[->,thick] (d) edge (e);
\path[->,thick,densely dashed] (c) edge (e);
\path[->,thick] (b) edge [bend right = 5] node {} (e);
\path[->,thick,densely dashed] (b) edge [bend left = 5] node {} (e);
\path[->,very thick]  (b) edge (a);
\path[->, thick,densely dashed] (b) edge (a);
\draw[snake=coil,segment aspect=0, black!60] (b) -- (a);
\end{tikzpicture}
}
\caption[Transition graph for equivalent stubbornness knowledge]{Transition graph for equivalent stubbornness knowledge. Nodes are coloured: the gray node is the configuration of the positive outcome of the negotiation process.}\label{grafoEQ}
\end{figure}
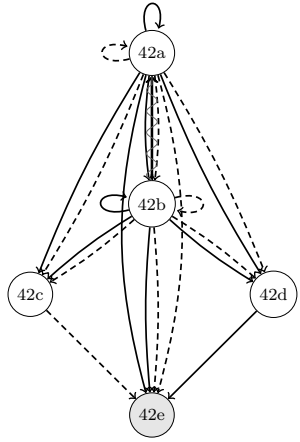

All the rules the agents use when their stubbornness knowledge are
equivalent are \emph{legitimate}.\fix{Luca}{La nozione ``legitimate'' non \`e
definita. Si pu\`o togliere o rimpiazzare con qualcosa d'altro? Oppure
definire la nozione? Lo stesso per ``violation'' dopo. Da quello che
capisco, la terminologia \`e diversa (ad esempio ``call-away'')} A rule is legitimate when it can be used in a specific configuration. In the
following example, we show how deductive rules of MND are used and their
effects in the EGG/YOLK configurations when the stubbornness knowledge
of agents are equivalent.

\begin{example}\label{es:EQS1S2}
Suppose Alice and Bob are related as in configuration 42a. Alice, $A$, is the first bidding agent and she proposes $\flex{0}_{A}$ to Bob, $B$. Bob receives the proposal and evaluates it.
Bob tests that they are in essence disagreement and generalizes his initial viewpoint $\flex{0}_B$ by:
$$
\infer[(W)]{\flex{1}_B}
{
\flex{0}_B\rightarrow\flex{1}_B &
\neg( \stub_B \leftrightarrow \flex{0}_B)
}
$$
and he checks the provisional negotiation situation by:
$$
\infer[(ED)]{B:\mathbf{essDis}(A:\flex{0}_A)\cap B:\flex{1}_B}
{A:\flex{0}_A & \neg(\flex{0}_B\wedge \trad{B}{A}(\flex{0}_A))\wedge (\stub_B\vee \trad{B}{A}(\flex{0}_A)}
$$
Bob says to Alice that they are in essence disagreement and makes a proposal $\flex{1}_B$.

The system continues the MN by:
$$
\infer[(N)]{\mathit{Negotiate}(A,B)}{\ast(A,B) & A:\flex{0}_A & B:\mathbf{essDis}(A:\flex{0}_A) & B:\flex{1}_B}
$$

Alice receives $\flex{1}_B$ and she makes a weakening or a changing action because Bob said they are not in agreement nor in relative disagreement.
Suppose Alice generalizes her CAF by:
$$
\infer[(W)]{\flex{1}_A}
{
\flex{0}_A\rightarrow\flex{1}_A &
\neg( \stub_A \leftrightarrow \flex{0}_A)
}
$$
Alice tests the negotiation relation by:
$$
\infer[(ED\text{-}Co)]{A:\mathbf{comp}(B:\flex{1}_B)\cap A:\flex{1}_A}
{\begin{array}{c}B:\mathbf{essDis}(A:\flex{0}_A)\cap B:\flex{1}_B \\ (\flex{1}_A\vee \trad{A}{B}(\flex{1}_B))\wedge\neg(\flex{1}_A\rightarrow\trad{A}{B}(\flex{1}_B))\wedge\neg(\trad{A}{B}(\flex{1}_B)\rightarrow\flex{1}_A)\end{array}}
$$
Alice says to Bob that they are in compatibility and makes a proposal $\flex{1}_A$.

The system continues the MN by:
$$
\infer[(N)]{\mathit{Negotiate}(B,A)}{\ast(B,A) & B:\flex{1}_B & A:\mathbf{comp}(B:\flex{1}_B) & A:\flex{1}_A}
$$

Bob receives $\flex{1}_A$ and he makes a weakening or a changing action because Alice said they are not in agreement nor in relative disagreement.
Suppose Bob changes his CAF by:
$$
\infer[(C)]{\flex{2}_B}{\flex{1}_B & \neg( \stub_B \leftrightarrow \flex{1}_B) & \neg(\flex{1}_B\rightarrow\flex{2}_B) & \neg(\flex{2}_B \rightarrow \flex{1}_B)}
$$
Bob tests the negotiation relation by:
$$
\infer[(Co\text{-}Co)]{B:\mathbf{comp}(A:\flex{1}_A)\cap B:\flex{2}_B}
{\begin{array}{c}A:\mathbf{comp}(B:\flex{1}_B)\cap A:\flex{1}_A \\ (\flex{2}_A\vee \trad{B}{A}(\flex{1}_A))\wedge\neg(\flex{2}_A\rightarrow \trad{B}{A}(\flex{1}_A))\wedge\neg(\flex{2}_A\leftarrow \trad{B}{A}(\flex{1}_A))\end{array}}
$$
Bob says to Alice that they are in compatibility and makes a proposal $\flex{2}_B$.

The system continues the MN by:
$$
\infer[(N)]{\mathit{Negotiate}(A,B)}{\ast(A,B) & A:\flex{1}_A & B:\mathbf{comp}(A:\flex{1}_A) & B:\flex{2}_B}
$$

Alice receives $\flex{2}_B$ and she makes a weakening or a changing action because Bob said they are not in agreement nor in relative disagreement.
Suppose Alice changes her CAF by:
$$
\infer[(C)]{\flex{3}_A}{\flex{2}_A & \neg( \stub_A \leftrightarrow \flex{2}_A) & \neg(\flex{2}_A\rightarrow\flex{3}_A) & \neg(\flex{3}_A\rightarrow\flex{2}_A)}
$$
Alice tests the negotiation relation by:
$$
\infer[(Co\text{-}Ag)]{A:\mathbf{agree}(B:\flex{2}_A)\cap A:\trad{A}{B}(\flex{2}_B)}
{B:\mathbf{comp}(A:\flex{2}_A)\cap B:\flex{2}_B & (\flex{3}_A \rightarrow \trad{A}{B}(\flex{3}_B))}
$$
Alice says to Bob that they are in agreement and that they have a common angle that is $\flex{2}_B$.

The system closes the MN by:
$$
\infer[(A)]{\mathit{Agreement}(B,A)}{\ast(B,A) & B:\flex{2}_B & A:\mathbf{agree}(B:\flex{2}_B)}
$$
with a positive outcome, $\flex{2}_B$.

\begin{figure}
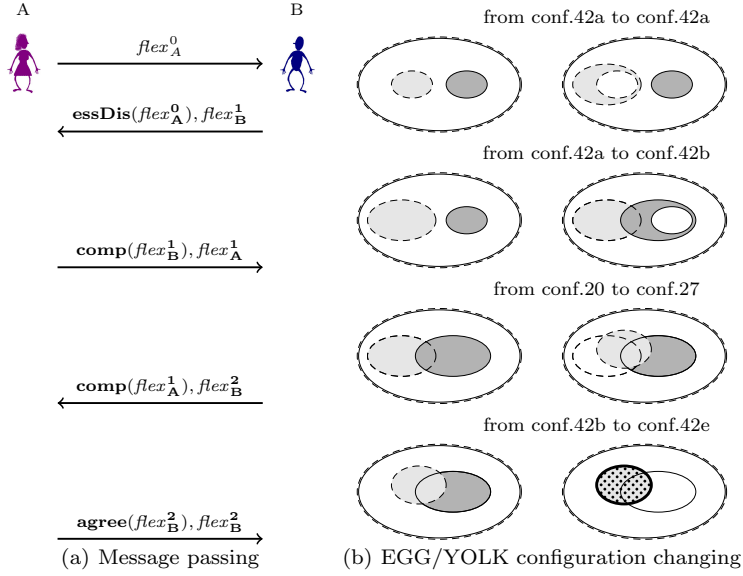
\centering
\subfigure[Message passing]{
\label{es:EQ-mp}
\scalebox{0.9}{
\begin{tikzpicture}[level distance=4mm]\scriptsize
\draw (0,22.8) node (alice) {A};
\draw (0,22) node (a) {\includegraphics[width=0.5cm]{alice.jpg}};
\draw (4,22.8) node (bob) {B};
\draw (4,22) node (b) {\includegraphics[width=0.5cm]{bob.jpg}};
\draw[->,thick] (0.5,22) .. controls +(right:0.5cm) and +(left:0.5cm).. node[above,sloped] {$\flex{0}_A$} (3.5,22);
\draw[<-,thick] (0.5,21) .. controls +(right:0.5cm) and +(left:0.5cm).. node[above,sloped] {$\mathbf{essDis(\flex{0}_A),\flex{1}_B}$} (3.5,21);
\draw[->,thick] (0.5,19) .. controls +(right:0.5cm) and +(left:0.5cm).. node[above,sloped] {$\mathbf{comp(\flex{1}_B),\flex{1}_A}$} (3.5,19);
\draw[<-,thick] (0.5,17) .. controls +(right:0.5cm) and +(left:0.5cm).. node[above,sloped] {$\mathbf{comp(\flex{1}_A),\flex{2}_B}$} (3.5,17);
\draw[->,thick] (0.5,15) .. controls +(right:0.5cm) and +(left:0.5cm).. node[above,sloped] {$\mathbf{agree(\flex{2}_B),\flex{2}_B}$} (3.5,15);
\end{tikzpicture}
}
}
\subfigure[EGG/YOLK configuration changing]{
\label{es:EQ-ey}
\scalebox{0.9}{
\begin{tikzpicture}[level distance=4mm]
\draw (2.8,21) node (8-12) {from conf.42a to conf.42a};
\draw[densely dashed] (0.5,20) ellipse (12mm and 7mm);
\draw (0.5,20) ellipse (11.8mm and 6.8mm);
\draw[densely dashed, fill = gray!20] (0.1,20) ellipse (3mm and 2mm);
\draw[fill = gray!60] (0.9,20) ellipse (3mm and 2mm);
\draw[densely dashed] (3.5,20) ellipse (12mm and 7mm);
\draw (3.5,20) ellipse (11.8mm and 6.8mm);
\draw[densely dashed, fill = gray!20] (2.95,20) ellipse (5mm and 3mm);
\draw[densely dashed, fill = white] (3.1,20) ellipse (3mm and 2mm);
\draw[fill = gray!60] (3.9,20) ellipse (3mm and 2mm);
\draw (2.8,19) node (13-20) {from conf.42a to conf.42b};
\draw[densely dashed] (0.5,18) ellipse (12mm and 7mm);
\draw (0.5,18) ellipse (11.8mm and 6.8mm);
\draw[densely dashed, fill = gray!20] (-0.05,18) ellipse (5mm and 3mm);
\draw[fill = gray!60] (0.9,18) ellipse (3mm and 2mm);
\draw[densely dashed] (3.5,18) ellipse (12mm and 7mm);
\draw (3.5,18) ellipse (11.8mm and 6.8mm);
\draw[densely dashed, fill = gray!20] (2.95,18) ellipse (5mm and 3mm);
\draw[fill = gray!60] (3.7,18) ellipse (5.5mm and 3mm);
\draw[fill = white] (3.9,18) ellipse (3mm and 2mm);
\draw[densely dashed] (2.95,18) ellipse (5mm and 3mm);
\draw (2.8,17) node (20-27) {from conf.20 to conf.27};
\draw[densely dashed] (0.5,16) ellipse (12mm and 7mm);
\draw (0.5,16) ellipse (11.8mm and 6.8mm);
\draw[densely dashed, fill = gray!20] (-0.05,16) ellipse (5mm and 3mm);
\draw[fill = gray!60] (0.7,16) ellipse (5.5mm and 3mm);
\draw[densely dashed] (-0.05,16) ellipse (5mm and 3mm);
\draw[densely dashed] (3.5,16) ellipse (12mm and 7mm);
\draw (3.5,16) ellipse (11.8mm and 6.8mm);
\draw[fill = gray!60] (3.7,16) ellipse (5.5mm and 3mm);
\draw[densely dashed] (2.95,16) ellipse (5mm and 3mm);
\draw[densely dashed, fill = gray!20] (3.2,16.1) ellipse (4mm and 2.8mm);
\draw[densely dashed] (2.95,16) ellipse (5mm and 3mm);
\draw (3.7,16) ellipse (5.5mm and 3mm);
\draw (2.8,15) node (27-32) {from conf.42b to conf.42e};
\draw[densely dashed] (0.5,14) ellipse (12mm and 7mm);
\draw (0.5,14) ellipse (11.8mm and 6.8mm);
\draw[fill = gray!60] (0.7,14) ellipse (5.5mm and 3mm);
\draw[densely dashed, fill = gray!20] (0.2,14.1) ellipse (4mm and 2.8mm);
\draw (0.7,14) ellipse (5.5mm and 3mm);
\draw[densely dashed] (3.5,14) ellipse (12mm and 7mm);
\draw (3.5,14) ellipse (11.8mm and 6.8mm);
\draw[very thick, fill = gray!20] (3.2,14.1) ellipse (4mm and 2.8mm);
\draw[densely dashed, pattern= crosshatch dots] (3.2,14.1) ellipse (4mm and 2.8mm);
\draw (3.7,14) ellipse (5.5mm and 3mm);
\end{tikzpicture}
}
}
\caption[A MN scenario between agents with equivalent stubbornness knowledge]{A MN scenario between Alice and Bob with equivalent stubbornness knowledge: the message passing flow (a) and the changes of their CAFs (b). Alice is identified by plain lines and Bob by dashed lines. White yolks represent the precedent proposal of the agent and the dotted yolk is the positive outcome of the scenario.\label{es:EQ}}
\end{figure}

In Figure~\ref{es:EQ}, we show the message passing flow between Alice and Bob 
and the changes of the EGG/YOLK configurations.
The MN results in a path, shown in Figure~\ref{fig:grafoEQes}, from
node 8 to node 41 of the graph in Figure~\ref{grafoEQ}.

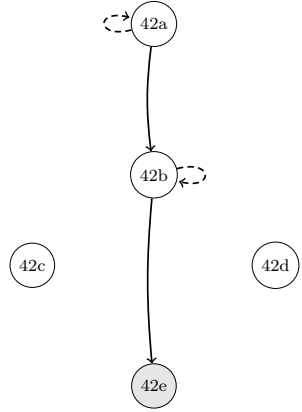
\begin{figure}\centering
\scalebox{0.8}{
\begin{tikzpicture}[level distance=4mm]\centering
\draw (2,0) node[draw, shape = circle, fill = black!10] (e) {42e};
\draw (4,2) node[draw, shape = circle] (d) {42d};
\draw (0,2) node[draw, shape = circle] (c) {42c};
\draw (2,3.5) node[draw, shape = circle] (b) {42b};
\draw (2,6) node[draw, shape = circle] (a) {42a};
\path[->,thick,densely dashed] (a) edge [loop left] node {} (a);
\path[->,thick,densely dashed] (b) edge [loop right] node {} (b);
\path[->,thick] (a) edge [bend right = 7] node {} (b);
\path[->,thick] (b) edge [bend right = 5] node {} (e);
\end{tikzpicture}
}
\caption[The MN path of the message passing in Figure \ref{es:EQ}]{The MN path of the Alice and Bob message passing in Figure \ref{es:EQ}.}\label{fig:grafoEQes}
\end{figure}
\hfill$\Box$
\end{example}

\subsubsection{Generalized Stubbornness Knowledge}

Suppose that one agent's stubbornness set is a \emph{generalization} of the theory of the opponent, i.e. they are consistent and one is a restriction of the other. Then
$$
\stub_i \rightarrow \stub_j
$$

The\fix{Luca}{This whole description should be either anticipated to
before or deleted} generalization (weakening, relaxing, etc.) relation
between stubbornness theories is represented in RCC5 as the partial
proper part relation between eggs. We assumed that the stubbornness part
of the agent theory never changes, then the models satisfying it are
fixed at the beginning of the negotiation process.

On the other hand, the flexible sets are relaxed or changed during the
negotiation process so that the models satisfying them change during the
negotiation. The \emph{flexible models} are the yolks of the RCC theory.

In Table~\ref{confPP}, we show the possible yolk configurations and we
give a statement representing the configuration, i.e. the negotiation
state.

\begin{center}
\begin{longtable}{ll}
\hline \textbf{Configuration} & \textbf{Formula} \\ \hline
\endfirsthead

\multicolumn{2}{c}%
{{\bfseries \tablename\ \thetable{} -- continued from previous page}} \\
\hline \textbf{Configuration} & \textbf{Formula} \\ \hline
\endhead

\hline \multicolumn{2}{r}{{Continued on next page}} \\ \hline
\endfoot

\endlastfoot

\hline
\begin{minipage}{3cm}
\ \\
\normalsize{8} \\
\scalebox{0.6}{
\begin{tikzpicture}[level distance=4mm]
\draw[densely dashed] (0.5,0) ellipse (13mm and 8mm);
\draw[densely dashed, fill = gray!20] (-0.2,0) ellipse (3mm and 2mm);
\draw (1,0) ellipse (7mm and 4mm);
\draw[fill = gray!60] (1.3,0) ellipse (3mm and 2mm);
\end{tikzpicture}}
\\
\end{minipage} &
\begin{minipage}{8cm}
$(\stub_i \rightarrow \stub_j) \wedge \neg(\flex{}_i \wedge \flex{}_j)\wedge(\flex{}_i \rightarrow \stub_j)\wedge \neg(\flex{}_j \wedge \stub_i)$
\end{minipage}\\
\begin{minipage}{3cm}
\ \\
\normalsize{13}\\
\scalebox{0.6}{
\begin{tikzpicture}[level distance=4mm]
\draw[densely dashed] (0.5,0) ellipse (13mm and 8mm);
\draw[densely dashed, fill = gray!20] (0.2,0) ellipse (3mm and 2mm);
\draw (1,0) ellipse (7mm and 4mm);
\draw[fill = gray!60] (1.3,0) ellipse (3mm and 2mm);
\end{tikzpicture}}
\\
\end{minipage} &
\begin{minipage}{8cm}
$(\stub_i \rightarrow \stub_j) \wedge \neg(\flex{}_i \wedge \flex{}_j)\wedge(\flex{}_i \rightarrow \stub_j)\wedge (\flex{}_j \vee \stub_i) \wedge \neg (\flex{}_j\rightarrow\stub_i)\wedge \neg(\stub_i\rightarrow\flex{}_j)$
\end{minipage}\\
\begin{minipage}{3cm}
\ \\
\normalsize{20} \\
\scalebox{0.6}{
\begin{tikzpicture}[level distance=4mm]
\draw[densely dashed] (0.5,0) ellipse (13mm and 8mm);
\draw[densely dashed, fill = gray!20] (0.4,0) ellipse (3.5mm and 2.5mm);
\draw (1,0) ellipse (7mm and 4mm);
\draw[fill = gray!60] (0.9,0) ellipse (3mm and 2mm);
\draw[densely dashed] (0.4,0) ellipse (3.5mm and 2.5mm);
\end{tikzpicture}}
\\
\end{minipage} &
\begin{minipage}{8cm}
$(\stub_i \rightarrow \stub_j) \wedge (\flex{}_i \vee \flex{}_j)\wedge\neg(\flex{}_i \rightarrow \flex{}_j)\wedge\neg(\flex{}_i \rightarrow \flex{}_j)\wedge(\flex{}_i \rightarrow \stub_j)\wedge (\flex{}_j \vee \stub_i)\wedge\neg(\flex{}_j \rightarrow \stub_i)\wedge\neg(\stub_i\rightarrow\flex{}_j)$
\end{minipage}\\
\begin{minipage}{3cm}
\ \\
\normalsize{22}\\
\scalebox{0.6}{
\begin{tikzpicture}[level distance=4mm]
\draw[densely dashed] (0.5,0) ellipse (13mm and 8mm);
\draw[densely dashed, fill = gray!20] (0.5,0) ellipse (10mm and 5mm);
\draw (0.5,0) ellipse (7mm and 4mm);
\draw[fill = gray!60] (0.5,0) ellipse (3mm and 2mm);
\end{tikzpicture}}
\\
\end{minipage} &
\begin{minipage}{8cm}
$(\stub_i \rightarrow \stub_j) \wedge (\flex{}_i \rightarrow \flex{}_j)\wedge(\flex{}_i \rightarrow \stub_j)\wedge (\stub_i \rightarrow \flex{}_j)$
\end{minipage}\\
\begin{minipage}{3cm}
\ \\
\normalsize{24} \\
\scalebox{0.6}{
\begin{tikzpicture}[level distance=4mm]
\draw[densely dashed] (0.5,0) ellipse (13mm and 8mm);
\draw[densely dashed, fill = gray!20] (0.4,0) ellipse (7.3mm and 3mm);
\draw (1,0) ellipse (7mm and 4mm);
\draw[fill = gray!60] (0.7,0) ellipse (3mm and 2mm);
\end{tikzpicture}}
\\
\end{minipage} &
\begin{minipage}{8cm}
$(\stub_i \rightarrow \stub_j) \wedge (\flex{}_i \rightarrow \flex{}_j)\wedge(\flex{}_i \rightarrow \stub_j)\wedge (\flex{}_j \vee \stub_i)\wedge\neg(\flex{}_j\rightarrow\stub_i)\wedge\neg(\stub_i\rightarrow\flex{}_j)$
\end{minipage}\\
\begin{minipage}{3cm}
\ \\
\normalsize{27}\\
\scalebox{0.6}{
\begin{tikzpicture}[level distance=4mm]
\draw[densely dashed] (0.5,0) ellipse (13mm and 8mm);
\draw[densely dashed, fill = gray!20] (0.1,0.1) ellipse (3mm and 2mm);
\draw (0.5,0) ellipse (10mm and 5mm);
\draw[fill = gray!60] (0.9,-0.1) ellipse (3mm and 2mm);
\end{tikzpicture}}
\\
\end{minipage} &
\begin{minipage}{8cm}
$(\stub_i \rightarrow \stub_j) \wedge \neg(\flex{}_i \wedge \flex{}_j)\wedge(\flex{}_i \rightarrow \stub_j)\wedge (\flex{}_j \rightarrow \stub_i)$
\end{minipage}\\
\begin{minipage}{3cm}
\ \\
\normalsize{32} \\
\scalebox{0.6}{
\begin{tikzpicture}[level distance=4mm]
\draw[densely dashed] (0.5,0) ellipse (13mm and 8mm);
\draw[densely dashed, fill = gray!20] (0.3,0.05) ellipse (3.5mm and 2.5mm);
\draw (0.5,0) ellipse (10mm and 5mm);
\draw[fill = gray!60] (0.7,-0.05) ellipse (3.5mm and 2.5mm);
\draw[densely dashed] (0.3,0.05) ellipse (3.5mm and 2.5mm);
\end{tikzpicture}}
\\
\end{minipage} &
\begin{minipage}{8cm}
$(\stub_i \rightarrow \stub_j) \wedge (\flex{}_i \vee \flex{}_j)\wedge\neg(\flex{}_i\rightarrow\flex{}_j)\wedge\neg(\flex{}_j\rightarrow\flex{}_i)\wedge(\flex{}_i \rightarrow \stub_j)\wedge (\flex{}_j \rightarrow \stub_i)$
\end{minipage}\\
\begin{minipage}{3cm}
\ \\
\normalsize{34}\\
\scalebox{0.6}{
\begin{tikzpicture}[level distance=4mm]
\draw[densely dashed] (0.5,0) ellipse (13mm and 8mm);
\draw[densely dashed, fill = gray!20] (0.5,0) ellipse (8mm and 4mm);
\draw (0.5,0) ellipse (8.20mm and 4.20mm);
\draw[fill = gray!60] (0.5,0) ellipse (3mm and 2mm);
\end{tikzpicture}}
\\
\end{minipage} &
\begin{minipage}{8cm}
$(\stub_i \rightarrow \stub_j) \wedge (\flex{}_i \rightarrow \flex{}_j)\wedge(\flex{}_j \leftrightarrow \stub_i)\wedge (\flex{}_i \rightarrow \stub_j)$
\end{minipage}\\
\begin{minipage}{3cm}
\ \\
\normalsize{37} \\
\scalebox{0.6}{
\begin{tikzpicture}[level distance=4mm]
\draw[densely dashed] (0.5,0) ellipse (13mm and 8mm);
\draw[densely dashed, fill = gray!20] (0.5,0) ellipse (6mm and 3mm);
\draw (0.5,0) ellipse (10mm and 5mm);
\draw[fill = gray!60] (0.5,0) ellipse (3mm and 2mm);
\end{tikzpicture}}
\\
\end{minipage} &
\begin{minipage}{8cm}
$(\stub_i \rightarrow \stub_j) \wedge (\flex{}_i \rightarrow \flex{}_j)\wedge(\flex{}_i \rightarrow \stub_j)\wedge (\flex{}_j \rightarrow \stub_i)$
\end{minipage}\\
\begin{minipage}{3cm}
\ \\
\normalsize{38}\\
\scalebox{0.6}{
\begin{tikzpicture}[level distance=4mm]
\draw[densely dashed] (0.5,0) ellipse (13mm and 8mm);
\draw[fill = gray!60] (0.5,0) ellipse (6mm and 3mm);
\draw (0.5,0) ellipse (10mm and 5mm);
\draw[densely dashed, fill = gray!20] (0.5,0) ellipse (3mm and 2mm);
\end{tikzpicture}}
\\
\end{minipage} &
\begin{minipage}{8cm}
$(\stub_i \rightarrow \stub_j) \wedge (\flex{}_j \rightarrow \flex{}_i)\wedge(\flex{}_i \rightarrow \stub_j)\wedge (\flex{}_j \rightarrow \stub_i)$
\end{minipage}\\
\begin{minipage}{3cm}
\ \\
\normalsize{41} \\
\scalebox{0.6}{
\begin{tikzpicture}[level distance=4mm]
\draw[densely dashed] (0.5,0) ellipse (13mm and 8mm);
\draw[fill = gray!60] (0.5,0) ellipse (4mm and 2mm);
\draw[densely dashed, fill = gray!20] (0.5,0) ellipse (3.8mm and 1.8mm);
\draw (0.5,0) ellipse (10mm and 5mm);
\end{tikzpicture}}
\\
\end{minipage} &
\begin{minipage}{8cm}
$(\stub_i \rightarrow \stub_j) \wedge (\flex{}_i \leftrightarrow \flex{}_j)\wedge(\flex{}_i \rightarrow \stub_j)\wedge (\flex{}_j \rightarrow \stub_i)$
\end{minipage}\\
\caption[Egg/yolk configurations for generalized/restricted stubbornness sets]{Configurations for generalized/restricted stubbornness sets. The stubbornness knowledge of agent $i$, identified by plain lines, is generalized by the stubbornness knowledge of the agent $j$, identified by dashed lines.}\label{confPP}
\end{longtable}
\end{center}

Figure \ref{grafoPP} depicts the graph of the possible negotiation
relations of the agents during the negotiation. The nodes are the
EGG/YOLK configurations and the edges are colored by the agent who makes
the next bid. The gray node identifies the positive outcome of the
negotiation.

All the rules of agent $i$, identified by plain lines, are legitimate.
The violation rules are used only by the agent $j$, identified by dashed
lines.

\begin{figure}\centering
\scalebox{0.8}{
\begin{tikzpicture}[level distance=4mm]\centering
\draw (5,0) node[draw, shape = circle] (34) {34};
\draw (5,1.5) node[draw, shape = circle,fill = black!10] (41) {41};
\draw (5,4) node[draw, shape = circle] (32) {32};
\draw (3,3) node[draw, shape = circle] (37) {37};
\draw (7,3) node[draw, shape = circle] (38) {38};
\draw (5,6) node[draw, shape = circle] (27) {27};
\draw (8,7) node[draw, shape = circle] (22) {22};
\draw (6,9) node[draw, shape = circle] (24) {24};
\draw (10,10.5) node[draw, shape = circle] (20) {20};
\draw (8,12) node[draw, shape = circle] (13) {13};
\draw (3,9) node[draw, shape = circle] (8) {8};
\path[->,thick, densely dashed] (20) edge  (8);
\draw[snake=coil,segment aspect=0, black!60] (8) -- (20);
\path[->,thick, densely dashed] (8) edge [bend left = 10] (20);
\draw[->,thick, densely dashed] (8) -- (24);
\draw[->,thick, densely dashed] (8) -- (22);
\draw[->,thick, densely dashed] (8) -- (27);
\draw[->,thick, densely dashed] (8) -- (32);
\draw[->,thick, densely dashed] (8) -- (37);
\path[->,thick, densely dashed] (8) edge [bend left = 20] (38);
\path[->,thick, densely dashed] (8) edge [bend right = 10] (34);
\path[->,thick, densely dashed] (8) edge [bend right = 5] (41);
\path[->,thick, densely dashed] (8) edge [bend left = 10] (13);
\path[->,thick, densely dashed] (13) edge  (8);
\draw[snake=coil,segment aspect=0, black!60] (8) -- (13);
\path[->,thick, densely dashed] (8) edge [loop above] node {} (8);
\path[->,thick] (8) edge [loop left] node {} (8);
\path[->,thick, densely dashed] (13) edge [loop above] node {} (13);
\path[->,thick] (13) edge [loop left] node {} (13);
\path[->,thick, densely dashed] (13) edge [bend left = 5] (24);
\path[->,thick] (13) edge [bend right = 5] (24);
\path[->,thick, densely dashed] (13) edge [bend left = 10] (20);
\draw[snake=coil,segment aspect=0, black!60] (20) -- (13);
\path[->,very thick] (20) edge (13);
\path[->,thick, densely dashed] (20) edge (13);
\path[->,thick] (13) edge [bend right = 10] (20);
\path[->,thick, densely dashed] (13) edge [bend left = 20] (38);
\path[->,thick, densely dashed] (13) edge [bend left = 40] (34);
\path[->,thick, densely dashed] (13) edge (41);
\path[->,thick, densely dashed] (13) edge [bend right = 15] (37);
\path[->,thick, densely dashed] (13) edge (27);
\path[->,thick, densely dashed] (13) edge (32);
\path[->,thick, densely dashed] (20) edge [bend left = 5] (24);
\path[->,thick] (20) edge [bend right = 5] (24);
\path[->,thick, densely dashed] (20) edge [loop above] node {} (20);
\path[->,thick] (20) edge [loop right] node {} (20);
\path[->,thick, densely dashed] (20) edge (22);
\path[->,thick, densely dashed] (20) edge (27);
\draw[snake=coil,segment aspect=0, black!60] (20) -- (27);
\path[->,thick, densely dashed] (20) edge (32);
\path[->,thick, densely dashed] (20) edge (37);
\path[->,thick, densely dashed] (20) edge [bend left = 5] (38);
\path[->,thick, densely dashed] (20) edge [bend left = 10] (41);
\path[->,thick, densely dashed] (20) edge [bend left = 30] (34);
\path[->,thick, densely dashed] (22) edge (41);
\path[->,thick, densely dashed] (24) edge [bend right = 30] (41);
\path[->,thick, densely dashed] (27) edge [bend right = 20] (41);
\path[->,thick] (27) edge [bend left = 20] (41);
\path[->,thick, densely dashed] (27) edge [bend right = 5] (37);
\path[->,thick] (27) edge [bend left = 5] (37);
\path[->,thick, densely dashed] (27) edge [bend right = 5] (38);
\path[->,thick] (27) edge [bend left = 5] (38);
\path[->,thick, densely dashed] (27) edge [bend right = 5] (32);
\path[->,thick] (27) edge [bend left = 5] (32);
\path[->,thick, densely dashed] (32) edge [bend right = 5] (37);
\path[->,thick] (32) edge [bend left = 5] (37);
\path[->,thick, densely dashed] (32) edge [bend right = 5] (38);
\path[->,thick] (32) edge [bend left = 5] (38);
\path[->,thick] (32) edge [bend right = 5] (41);
\path[->,thick, densely dashed] (32) edge [bend left = 5] (41);
\path[->,thick, densely dashed] (37) edge (41);
\path[->,thick] (38) edge (41);
\path[->,thick, densely dashed] (34) edge (41);
\path[->,thick] (34) edge [loop below] node {} (34);
\path[->,thick] (38) edge [loop below] node {} (38);
\path[->,thick] (37) edge [loop below] node {} (37);
\path[->,thick] (27) edge [loop above] node {} (27);
\end{tikzpicture}
}
\caption[Transition graph for generalized (or restricted) stubbornness knowledge]{Transition graph for generalized (or restricted) stubbornness knowledge: the stubbornness knowledge of agent $i$, identified by plain lines, is a restriction of the stubbornness knowledge of agent $j$, identified by dashed lines. The nodes are colored: the gray node is the configuration of the positive outcome of the negotiation process.}\label{grafoPP}
\end{figure}
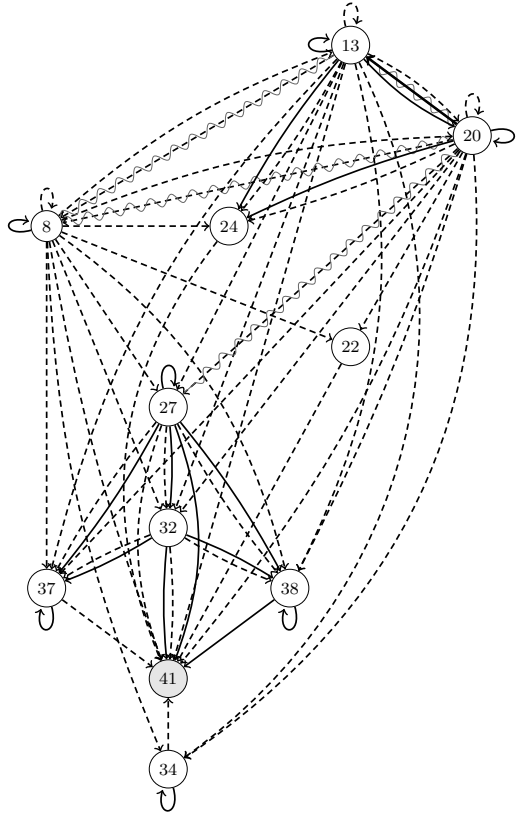

In the following example, we show how deductive rules of MND are used
and their effects in the EGG/YOLK configurations when the stubbornness
knowledge of agents are in a generalization/restriction relation.

\begin{example}\label{es:PPS1S2}
Suppose Alice and Bob are related as in configuration 8. Alice, $A$, is the first bidding agent and she proposes $\flex{0}_{A}$ to Bob, $B$. Bob receives the proposal and evaluates it.
Bob tests that they are in essence disagreement. Bob generalizes his initial viewpoint $\flex{0}_B$ by:
$$
\infer[(W)]{\flex{1}_B}
{
\flex{0}_B\rightarrow\flex{1}_B &
\neg( \stub_B \leftrightarrow \flex{0}_B)
}
$$
and he checks the provisional negotiation situation by:
$$
\infer[(ED)]{B:\mathbf{essDis}(A:\flex{0}_A)\cap B:\flex{1}_B}
{A:\flex{0}_A & \neg(\flex{0}_B\wedge \trad{B}{A}(\flex{0}_A))\wedge (\stub_B\vee \trad{B}{A}(\flex{0}_A)}
$$
Bob says to Alice that they are in essence disagreement and makes a proposal $\flex{1}_B$.

The system continues the MN by:
$$
\infer[(N)]{\mathit{Negotiate}(A,B)}{\ast(A,B) & A:\flex{0}_A & B:\mathbf{essDis}(A:\flex{0}_A) & B:\flex{1}_B}
$$

Alice receives $\flex{1}_B$ and she makes a weakening or a changing action because Bob said they are not in agreement nor in relative disagreement.
Suppose Alice changes her CAF by:
$$
\infer[(C)]{\flex{1}_A}{\flex{0}_A & \neg( \stub_A \leftrightarrow \flex{0}_A) & \neg(\flex{0}_A \rightarrow \flex{1}_A) & \neg(\flex{1}_A \rightarrow \flex{0}_A)}
$$
Alice tests the negotiation relation by:
$$
\infer[(ED\text{-}Co)]{A:\mathbf{comp}(B:\flex{1}_B)\cap A:\flex{1}_A}
{\begin{array}{c}B:\mathbf{essDis}(A:\flex{0}_A)\cap B:\flex{1}_B \\ (\flex{1}_A\vee \trad{A}{B}(\flex{1}_B))\wedge\neg(\flex{1}_A\rightarrow \trad{A}{B}(\flex{1}_B))\wedge\neg(\flex{1}_A\leftarrow \trad{A}{B}(\flex{1}_B))\end{array}}
$$
Alice says to Bob that they are in compatibility and makes a proposal $\flex{1}_A$.

The system continues the MN by:
$$
\infer[(N)]{\mathit{Negotiate}(B,A)}{\ast(B,A) & B:\flex{1}_B & A:\mathbf{comp}(B:\flex{1}_B) & A:\flex{1}_A}
$$

Bob receives $\flex{1}_A$ and he makes a weakening or a changing action because Alice said they are not in agreement nor in relative disagreement.
Suppose Bob changes his CAF by:
$$
\infer[(C)]{\flex{2}_B}{\flex{1}_B & \neg( \stub_B \leftrightarrow \flex{1}_B) & \neg(\flex{1}_B \rightarrow\flex{2}_B) & \neg(\flex{2}_B\rightarrow\flex{1}_B)}
$$
Bob tests the negotiation relation and makes a violation by:
$$
\infer[(Co\text{-}ED)]{B:\mathbf{essDis}(A:\flex{1}_A)\cap B:\flex{2}_B}
{\begin{array}{c}A:\mathbf{comp}(B:\flex{1}_B)\cap A:\flex{1}_A \\ (\stub_B \vee \trad{B}{A}(\flex{1}_A))\wedge\neg(\flex{2}_B \wedge \trad{B}{A}(\flex{1}_A))\end{array}}
$$
Bob says to Alice that they are in essence disagreement and makes a proposal $\flex{2}_B$.

The system continues the MN by:
$$
\infer[(N)]{\mathit{Negotiate}(A,B)}{\ast(A,B) & A:\flex{1}_A & B:\mathbf{essDis}(A:\flex{1}_A) & B:\flex{2}_B}
$$

Alice receives $\flex{2}_B$ and she makes a weakening or a changing action because Bob said they are not in agreement nor in relative disagreement.
Suppose Alice changes her CAF by:
$$
\infer[(C)]{\flex{2}_A}{\flex{1}_A & \neg( \stub_A \leftrightarrow \flex{1}_A) & \neg(\flex{1}_A\rightarrow\flex{2}_A) & \neg(\flex{2}_A\rightarrow\flex{1}_A)}
$$
Alice tests the negotiation relation by:
$$
\infer[(ED\text{-}Co)]{A:\mathbf{comp}(B:\flex{2}_B)\cap A:\flex{2}_A}
{\begin{array}{c}B:\mathbf{essDis}(A:\flex{1}_A)\cap B:\flex{2}_B \\ (\flex{2}_A\vee \trad{A}{B}(\flex{2}_B))\wedge \neg(\flex{2}_A\rightarrow \trad{A}{B}(\flex{2}_B))\wedge\neg(\flex{2}_A\leftarrow \trad{A}{B}(\flex{2}_B))\end{array}}
$$
Alice says to Bob that they are in compatibility and makes a proposal $\flex{2}_A$.

$$
\infer[(N)]{\mathit{Negotiate}(B,A)}{\ast(B,A) & B:\flex{2}_B & A:\mathbf{comp}(B:\flex{2}_B) & A:\flex{2}_A}
$$

Bob receives $\flex{2}_A$ and he makes a weakening or a changing action because Alice said they are not in agreement nor in relative disagreement.
Suppose Bob changes his CAF by:
$$
\infer[(C)]{\flex{3}_B}{\flex{2}_B & \neg( \stub_B \leftrightarrow \flex{2}_B) & \neg(\flex{2}_B\rightarrow\flex{3}_B) &\neg(\flex{3}_B\rightarrow\flex{2}_B)}
$$
Bob tests the negotiation relation by:
$$
\infer[(Co\text{-}Ag)]{B:\mathbf{agree}(A:\flex{2}_A)\cap B:\trad{B}{A}(\flex{2}_A)}
{A:\mathbf{comp}(B:\flex{2}_B)\cap A:\flex{2}_A & (\flex{3}_B \rightarrow \trad{B}{A}(\flex{2}_A))}
$$

Bob says to Alice that they are in agreement and that they have a common angle that is $\flex{2}_A$.

The system closes the MN by:
$$
\infer[(A)]{\mathit{Agreement}(A,B)}{\ast(A,B) & A:\flex{2}_A & B:\mathbf{agree}(A:\flex{2}_A)}
$$
with a positive outcome, $\flex{2}_A$.

In Figure \ref{es:PP} we show the message passing flow between Alice and Bob 
and the changes of the EGG/YOLK configurations.

\begin{figure}
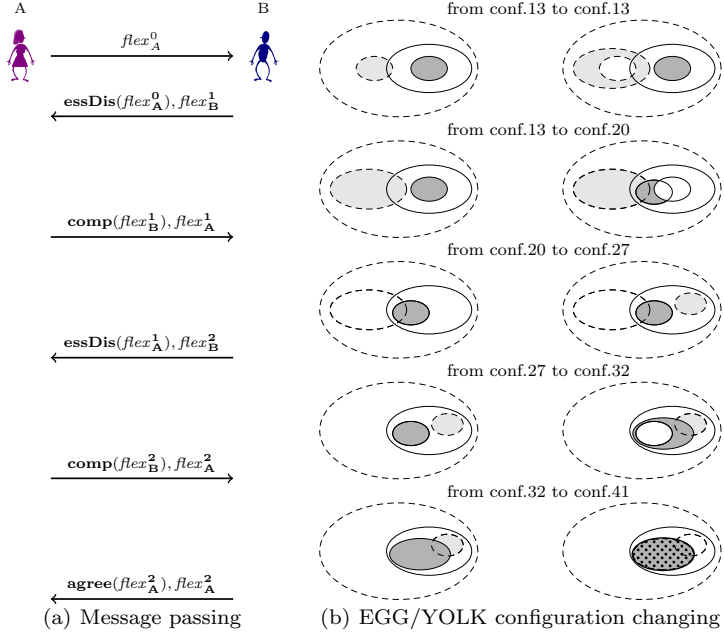
\centering
\subfigure[Message passing]{
\label{es-fig:PP-mp}
\scalebox{0.8}{
\begin{tikzpicture}[level distance=4mm]\scriptsize
\draw (0,22.8) node (alice) {A};
\draw (0,22) node (a) {\includegraphics[width=0.5cm]{alice.jpg}};
\draw (4,22.8) node (bob) {B};
\draw (4,22) node (b) {\includegraphics[width=0.5cm]{bob.jpg}};
\draw[->,thick] (0.5,22) .. controls +(right:0.5cm) and +(left:0.5cm).. node[above,sloped] {$\flex{0}_A$} (3.5,22);
\draw[<-,thick] (0.5,21) .. controls +(right:0.5cm) and +(left:0.5cm).. node[above,sloped] {$\mathbf{essDis(\flex{0}_A),\flex{1}_B}$} (3.5,21);
\draw[->,thick] (0.5,19) .. controls +(right:0.5cm) and +(left:0.5cm).. node[above,sloped] {$\mathbf{comp(\flex{1}_B),\flex{1}_A}$} (3.5,19);
\draw[<-,thick] (0.5,17) .. controls +(right:0.5cm) and +(left:0.5cm).. node[above,sloped] {$\mathbf{essDis(\flex{1}_A),\flex{2}_B}$} (3.5,17);
\draw[->,thick] (0.5,15) .. controls +(right:0.5cm) and +(left:0.5cm).. node[above,sloped] {$\mathbf{comp(\flex{2}_B),\flex{2}_A}$} (3.5,15);
\draw[<-,thick] (0.5,13) .. controls +(right:0.5cm) and +(left:0.5cm).. node[above,sloped] {$\mathbf{agree(\flex{2}_A),\flex{2}_A}$} (3.5,13);
\end{tikzpicture}
}
}
\subfigure[EGG/YOLK configuration changing]{
\label{es-fig:PP-ey}
\scalebox{0.8}{
\begin{tikzpicture}[level distance=4mm]
\draw (2.8,21) node (13-13) {from conf.13 to conf.13};
\draw[densely dashed] (0.5,20) ellipse (13mm and 8mm);
\draw[densely dashed, fill = gray!20] (0.1,20) ellipse (3mm and 2mm);
\draw (1,20) ellipse (7mm and 4mm);
\draw[fill = gray!60] (1,20) ellipse (3mm and 2mm);
\draw[densely dashed] (4.5,20) ellipse (13mm and 8mm);
\draw[densely dashed, fill = gray!20] (4,20) ellipse (6.3mm and 3.3mm);
\draw[densely dashed, fill = white] (4.1,20) ellipse (3mm and 2mm);
\draw (5,20) ellipse (7mm and 4mm);
\draw[fill = gray!60] (5,20) ellipse (3mm and 2mm);
\draw (2.8,19) node (13-20) {from conf.13 to conf.20};
\draw[densely dashed] (0.5,18) ellipse (13mm and 8mm);
\draw[densely dashed, fill = gray!20] (0,18) ellipse (6.3mm and 3.3mm);
\draw (1,18) ellipse (7mm and 4mm);
\draw[fill = gray!60] (1,18) ellipse (3mm and 2mm);
\draw[densely dashed] (4.5,18) ellipse (13mm and 8mm);
\draw[densely dashed, fill = gray!20] (4,18) ellipse (6.3mm and 3.3mm);
\draw (5,18) ellipse (7mm and 4mm);
\draw[fill = gray!60] (4.7,17.95) ellipse (3mm and 2mm);
\draw[fill = white] (5,18) ellipse (3mm and 2mm);
\draw (4.7,17.95) ellipse (3mm and 2mm);
\draw[densely dashed] (4,18) ellipse (6.3mm and 3.3mm);
\draw (2.8,17) node (20-27) {from conf.20 to conf.27};
\draw[densely dashed] (0.5,16) ellipse (13mm and 8mm);
\draw[densely dashed, fill = white] (0,16) ellipse (6.3mm and 3.3mm);
\draw (1,16) ellipse (7mm and 4mm);
\draw[fill = gray!60] (0.7,15.95) ellipse (3mm and 2mm);
\draw (0.7,15.95) ellipse (3mm and 2mm);
\draw[densely dashed] (0,16) ellipse (6.3mm and 3.3mm);
\draw[densely dashed] (4.5,16) ellipse (13mm and 8mm);
\draw[densely dashed, fill = white] (4,16) ellipse (6.3mm and 3.3mm);
\draw[densely dashed, fill = gray!20] (5.3,16.1) ellipse (2.6mm and 1.8mm);
\draw (5,16) ellipse (7mm and 4mm);
\draw[fill = gray!60] (4.7,15.95) ellipse (3mm and 2mm);
\draw (4.7,15.95) ellipse (3mm and 2mm);
\draw[densely dashed] (4,16) ellipse (6.3mm and 3.3mm);
\draw (2.8,15) node (27-32) {from conf.27 to conf.32};
\draw[densely dashed] (0.5,14) ellipse (13mm and 8mm);
\draw[densely dashed, fill = gray!20] (1.3,14.1) ellipse (2.6mm and 1.8mm);
\draw (1,14) ellipse (7mm and 4mm);
\draw[fill = gray!60] (.7,13.95) ellipse (3mm and 2mm);
\draw (.7,13.95) ellipse (3mm and 2mm);
\draw[densely dashed] (4.5,14) ellipse (13mm and 8mm);
\draw[densely dashed, fill = gray!20] (5.3,14.1) ellipse (2.6mm and 1.8mm);
\draw (5,14) ellipse (7mm and 4mm);
\draw[fill = gray!60] (4.85,13.95) ellipse (5mm and 2.6mm);
\draw[fill = white] (4.7,13.95) ellipse (3mm and 2mm);
\draw (4.7,13.95) ellipse (3mm and 2mm);
\draw[densely dashed] (5.3,14.1) ellipse (2.6mm and 1.8mm);
\draw (2.8,13) node (32-41) {from conf.32 to conf.41};
\draw[densely dashed] (0.5,12) ellipse (13mm and 8mm);
\draw[densely dashed, fill = gray!20] (1.3,12.1) ellipse (2.6mm and 1.8mm);
\draw (1,12) ellipse (7mm and 4mm);
\draw[fill = gray!60] (0.85,11.95) ellipse (5mm and 2.6mm);
\draw[densely dashed] (1.3,12.1) ellipse (2.6mm and 1.8mm);
\draw[densely dashed] (4.5,12) ellipse (13mm and 8mm);
\draw[densely dashed, fill = white] (5.3,12.1) ellipse (2.6mm and 1.8mm);
\draw (5,12) ellipse (7mm and 4mm);
\draw[very thick] (4.85,11.95) ellipse (5mm and 2.6mm);
\draw[densely dashed, fill = gray!60] (4.85,11.95) ellipse (5mm and 2.6mm);
\draw[densely dashed, pattern = crosshatch dots] (4.85,11.95) ellipse (5mm and 2.6mm);
\draw[densely dashed] (5.3,12.1) ellipse (2.6mm and 1.8mm);
\end{tikzpicture}
}
}
\caption[A MN scenario between agent with stubbornness knowledge]{A MN scenario between Alice and Bob with stubbornness knowledge of Alice is a restriction of the stubbornness knowledge of Bob: the message passing flow (a) and the changes of their CAFs (b). White yolks represent the precedent proposal of the agent and the dotted gray yolk is the positive outcome of the scenario. \label{es:PP}}
\end{figure}
The MN results in a path, showed in Figure \ref{fig:grafoPPes}, from node 8 to node 41 of the graph in Figure \ref{grafoPP}.

\begin{figure}\centering
\scalebox{0.8}{
\begin{tikzpicture}[level distance=4mm]\centering
\draw (5,0) node[draw, shape = circle] (34) {34};
\draw (5,1.5) node[draw, shape = circle,fill = black!10] (41) {41};
\draw (5,4) node[draw, shape = circle] (32) {32};
\draw (3,3) node[draw, shape = circle] (37) {37};
\draw (7,3) node[draw, shape = circle] (38) {38};
\draw (5,6) node[draw, shape = circle] (27) {27};
\draw (8,7) node[draw, shape = circle] (22) {22};
\draw (6,9) node[draw, shape = circle] (24) {24};
\draw (10,10.5) node[draw, shape = circle] (20) {20};
\draw (8,12) node[draw, shape = circle] (13) {13};
\draw (3,9) node[draw, shape = circle] (8) {8};
\path[->,thick, densely dashed] (8) edge [bend left = 10] (13);
\path[->,thick] (13) edge [bend right = 10] (20);
\path[->,thick, densely dashed] (20) edge (27);
\draw[snake=coil,segment aspect=0, black!60] (20) -- (27);
\path[->,thick] (27) edge [bend left = 5] (32);
\path[->,thick, densely dashed] (32) edge [bend left = 5] (41);
\end{tikzpicture}
}
\caption[The MN path of the message passing in Figure \ref{es:PP}]{The MN path of the Alice and Bob message passing in Figure \ref{es:PP}.}\label{fig:grafoPPes}
\end{figure}
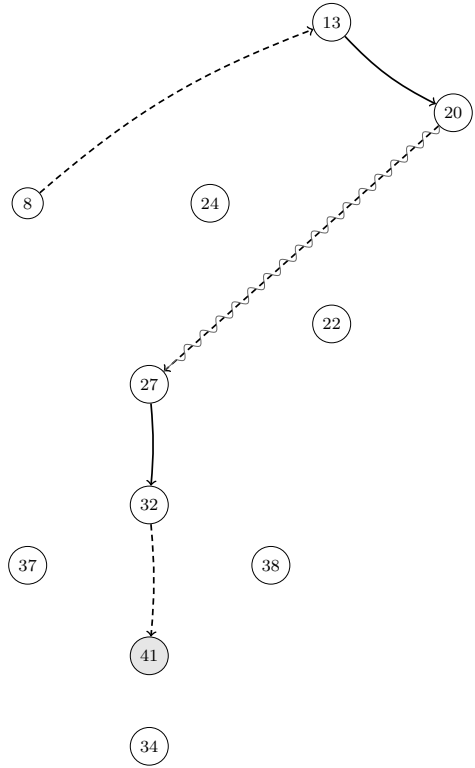
\hfill$\Box$
\end{example}

\subsubsection{Consistent Stubbornness Knowledge}

Suppose that the agents' stubbornness knowledge are \emph{compatible}, i.e. they are consistent and no one is a restriction or a generalization of the other. Then
$$
(\stub_i \vee \stub_j)\wedge\neg(\stub_i \rightarrow \stub_j)\wedge\neg(\stub_j \rightarrow \stub_i)
$$
The compatibility relation between stubbornness theories is represented in RCC5 as the partial overlapping relation between eggs. We assumed that if the stubbornness part of the agent theory never changes, then the models satisfying it are fixed at the beginning of the negotiation process.

On the other hand, the flexible sets are relaxed or changed during the
negotiation process so the models satisfying them change during the
negotiation. The \emph{flexible models} are the yolks of the RCC theory.

In Table \ref{confPO}, we show the possible yolk configurations and we give a statement representing the configuration, i.e. the negotiation state.
\begin{center}
\begin{longtable}{ll}
\hline \textbf{Configuration} & \textbf{Formula} \\ \hline
\endfirsthead

\multicolumn{2}{c}%
{{\bfseries \tablename\ \thetable{} -- continued from previous page}} \\
\hline \textbf{Configuration} & \textbf{Formula} \\ \hline
\endhead

\hline \multicolumn{2}{r}{{Continued on next page}} \\ \hline
\endfoot

\endlastfoot

\hline
\begin{minipage}{3cm}
\ \\
\normalsize{2} \\
\scalebox{0.6}{
\begin{tikzpicture}[level distance=4mm]
\draw (1,0) ellipse (10mm and 5mm);
\draw[fill = gray!60] (0.5,0) ellipse (3mm and 2mm);
\draw[densely dashed] (2,0) ellipse (10mm and 5mm);
\draw[densely dashed, fill = gray!20] (2.5,0) ellipse (3mm and 2mm);
\end{tikzpicture}
}
\\
\end{minipage} &
\begin{minipage}{8cm}
$(\stub_i \vee \stub_j)\wedge\neg(\stub_i \rightarrow \stub_j)\wedge\neg(\stub_j \rightarrow \stub_i) \wedge \neg(\flex{}_i \wedge \flex{}_j)\wedge\neg(\flex{}_i \wedge \stub_j)\wedge \neg(\flex{}_j \wedge \stub_i)$
\end{minipage}\\
\begin{minipage}{3cm}
\ \\
\normalsize{3}\\
\scalebox{0.6}{
\begin{tikzpicture}[level distance=4mm]
\draw (1,0) ellipse (10mm and 5mm);
\draw[fill = gray!60] (1,0) ellipse (3mm and 2mm);
\draw[densely dashed] (2,0) ellipse (10mm and 5mm);
\draw[densely dashed, fill = gray!20] (2.5,0) ellipse (3mm and 2mm);
\end{tikzpicture}
}
\\
\end{minipage}&
\begin{minipage}{8cm}
$(\stub_i \vee \stub_j)\wedge\neg(\stub_i \rightarrow \stub_j)\wedge\neg(\stub_j \rightarrow \stub_i) \wedge \neg(\flex{}_i \wedge \flex{}_j)\wedge(\flex{}_i \vee \stub_j)\wedge\neg(\flex{}_i \rightarrow \stub_j)\wedge\neg(\stub_j \rightarrow \flex{}_i)\wedge \neg(\flex{}_j \wedge \stub_i)$
\end{minipage}\\
\begin{minipage}{3cm}
\ \\
\normalsize{4} \\
\scalebox{0.6}{
\begin{tikzpicture}[level distance=4mm]
\draw[densely dashed, fill = gray!20] (2,0) ellipse (3mm and 2mm);
\draw (1,0) ellipse (10mm and 5mm);
\draw[fill = gray!60] (0.5,0) ellipse (3mm and 2mm);
\draw[densely dashed] (2,0) ellipse (10mm and 5mm);
\end{tikzpicture}
}
\\
\end{minipage} &
\begin{minipage}{8cm}
$(\stub_i \vee \stub_j)\wedge\neg(\stub_i \rightarrow \stub_j)\wedge\neg(\stub_j \rightarrow \stub_i) \wedge \neg(\flex{}_i \wedge \flex{}_j)\wedge\neg(\flex{}_i \wedge \stub_j)\wedge (\flex{}_j \vee \stub_i)\wedge\neg(\flex{}_j \rightarrow \stub_i)\wedge\neg(\stub_i \vee \flex{}_j)$
\end{minipage}\\
\begin{minipage}{3cm}
\ \\
\normalsize{5}\\
\scalebox{0.6}{
\begin{tikzpicture}[level distance=4mm]
\draw (1,0) ellipse (10mm and 5mm);
\draw[fill = gray!60] (0.5,0) ellipse (3mm and 2mm);
\draw[densely dashed] (2,0) ellipse (10mm and 5mm);
\draw[densely dashed, fill = gray!20] (1.5,0) ellipse (3mm and 2mm);
\end{tikzpicture}
}
\\
\end{minipage}&
\begin{minipage}{8cm}
$(\stub_i \vee \stub_j)\wedge\neg(\stub_i \rightarrow \stub_j)\wedge\neg(\stub_j \rightarrow \stub_i) \wedge \neg(\flex{}_i \wedge \flex{}_j)\wedge\neg(\flex{}_i \wedge \stub_j)\wedge (\flex{}_j \rightarrow \stub_i)$
\end{minipage}\\
\begin{minipage}{3cm}
\ \\
\normalsize{6}\\
\scalebox{0.6}{
\begin{tikzpicture}[level distance=4mm]
\draw (1,0) ellipse (10mm and 5mm);
\draw[fill = gray!60] (1.5,0) ellipse (3mm and 2mm);
\draw[densely dashed] (2,0) ellipse (10mm and 5mm);
\draw[densely dashed, fill = gray!20] (2.5,0) ellipse (3mm and 2mm);
\end{tikzpicture}
}
\\
\end{minipage}&
\begin{minipage}{8cm}
$(\stub_i \vee \stub_j)\wedge\neg(\stub_i \rightarrow \stub_j)\wedge\neg(\stub_j \rightarrow \stub_i) \wedge \neg(\flex{}_i \wedge \flex{}_j)\wedge(\flex{}_i \rightarrow \stub_j)\wedge \neg(\flex{}_j \wedge \stub_i)$
\end{minipage}\\
\begin{minipage}{3cm}
\ \\
\normalsize{9}\\
\scalebox{0.6}{
\begin{tikzpicture}[level distance=4mm]
\draw[densely dashed, fill = gray!20] (2,0) ellipse (3mm and 2mm);
\draw (1,0) ellipse (10mm and 5mm);
\draw[fill = gray!60] (1,0) ellipse (3mm and 2mm);
\draw[densely dashed] (2,0) ellipse (10mm and 5mm);
\end{tikzpicture}
}
\\
\end{minipage} &
\begin{minipage}{8cm}
$(\stub_i \vee \stub_j)\wedge\neg(\stub_i \rightarrow \stub_j)\wedge\neg(\stub_j \rightarrow \stub_i) \wedge \neg(\flex{}_i \wedge \flex{}_j)\wedge(\flex{}_i \vee \stub_j)\wedge \neg(\flex{}_i \rightarrow \stub_j)\wedge\neg(\stub_j \rightarrow \flex{}_i)\wedge(\flex{}_j \vee \stub_i)\wedge\neg(\flex{}_j \rightarrow \stub_i)\wedge\neg(\stub_i \rightarrow \flex{}_j)$\\
\end{minipage}\\
\begin{minipage}{3cm}
\ \\
\normalsize{10}\\
\scalebox{0.6}{
\begin{tikzpicture}[level distance=4mm]
\draw (1,0) ellipse (10mm and 5mm);
\draw[fill = gray!60] (0.9,0) ellipse (3mm and 2mm);
\draw[densely dashed] (2,0) ellipse (10mm and 5mm);
\draw[densely dashed, fill = gray!20] (1.6,0) ellipse (3mm and 2mm);
\end{tikzpicture}
}
\\
\end{minipage}&
\begin{minipage}{8cm}
$(\stub_i \vee \stub_j)\wedge\neg(\stub_i \rightarrow \stub_j)\wedge\neg(\stub_j \rightarrow \stub_i) \wedge \neg(\flex{}_i \wedge \flex{}_j)\wedge(\flex{}_i \rightarrow \stub_j)\wedge (\flex{}_j \vee \stub_i)\wedge\neg(\flex{}_j \rightarrow \stub_i)\wedge\neg(\stub_i \rightarrow \flex{}_j)$\\
\end{minipage}\\
\begin{minipage}{3cm}
\ \\
\normalsize{11}\\
\scalebox{0.6}{
\begin{tikzpicture}[level distance=4mm]
\draw[densely dashed, fill = gray!20] (2.1,0) ellipse (3mm and 2mm);
\draw (1,0) ellipse (10mm and 5mm);
\draw[fill = gray!60] (1.4,0) ellipse (3mm and 2mm);
\draw[densely dashed] (2,0) ellipse (10mm and 5mm);
\end{tikzpicture}
}
\\
\end{minipage}&
\begin{minipage}{8cm}
$(\stub_i \vee \stub_j)\wedge\neg(\stub_i \rightarrow \stub_j)\wedge\neg(\stub_j \rightarrow \stub_i)\wedge \neg(\flex{}_i \wedge \flex{}_j)\wedge(\flex{}_i \vee \stub_j)\wedge \neg(\flex{}_i \rightarrow \stub_j)\wedge\neg(\stub_j \rightarrow \flex{}_i)\wedge (\flex{}_j \rightarrow \stub_i)$\\
\end{minipage}\\
\begin{minipage}{3cm}
\ \\
\normalsize{14}\\
\scalebox{0.6}{
\begin{tikzpicture}[level distance=4mm]
\draw[densely dashed, fill = gray!20] (1.8,0) ellipse (3.5mm and 2mm);
\draw (1,0) ellipse (10mm and 5mm);
\draw[fill = gray!60] (1.2,0) ellipse (3.5mm and 2mm);
\draw[densely dashed] (2,0) ellipse (10mm and 5mm);
\draw[densely dashed] (1.8,0) ellipse (3.5mm and 2mm);
\end{tikzpicture}
}
\\
\end{minipage}&
\begin{minipage}{8cm}
$(\stub_i \vee \stub_j)\wedge\neg(\stub_i \rightarrow \stub_j)\wedge\neg(\stub_j \rightarrow \stub_i) \wedge (\flex{}_i \vee \flex{}_j)\wedge\neg(\flex{}_i \rightarrow \flex{}_j)\wedge\neg(\flex{}_j \rightarrow \flex{}_i)\wedge(\flex{}_i \vee \stub_j)\wedge \neg(\flex{}_i \rightarrow \stub_j)\wedge\neg(\stub_j \rightarrow \flex{}_i)\wedge (\flex{}_j \vee \stub_i)\wedge \neg(\flex{}_j \rightarrow \stub_i)\wedge\neg(\stub_i \rightarrow \flex{}_j)$\\
\end{minipage}\\
\begin{minipage}{3cm}
\ \\
\normalsize{15}\\
\scalebox{0.6}{
\begin{tikzpicture}[level distance=4mm]
\draw[densely dashed, fill = gray!20] (2,0) ellipse (3mm and 2mm);
\draw (1,0) ellipse (10mm and 5mm);
\draw[fill = gray!60] (1.5,0) ellipse (3mm and 2mm);
\draw[densely dashed] (2,0) ellipse (10mm and 5mm);
\draw[densely dashed] (2,0) ellipse (3mm and 2mm);
\end{tikzpicture}
}
\\
\end{minipage} &
\begin{minipage}{8cm}
$(\stub_i \vee \stub_j)\wedge\neg(\stub_i \rightarrow \stub_j)\wedge\neg(\stub_j \rightarrow \stub_i) \wedge (\flex{}_i \vee \flex{}_j)\wedge\neg(\flex{}_i \rightarrow \flex{}_j)\wedge\neg(\flex{}_j \rightarrow \flex{}_i)\wedge(\flex{}_i \rightarrow \stub_j)\wedge (\flex{}_j \vee \stub_i)\wedge \neg(\flex{}_j \rightarrow \stub_i)\wedge\neg(\stub_i \rightarrow \flex{}_j)$\\
\end{minipage}\\
\begin{minipage}{3cm}
\ \\
\normalsize{16}\\
\scalebox{0.6}{
\begin{tikzpicture}[level distance=4mm]
\draw[densely dashed, fill = gray!20] (1.5,0) ellipse (3mm and 2mm);
\draw (1,0) ellipse (10mm and 5mm);
\draw[fill = gray!60] (1,0) ellipse (3mm and 2mm);
\draw[densely dashed] (2,0) ellipse (10mm and 5mm);
\draw[densely dashed] (1.5,0) ellipse (3mm and 2mm);
\end{tikzpicture}
}
\\
\end{minipage}&
\begin{minipage}{8cm}
$(\stub_i \vee \stub_j)\wedge\neg(\stub_i \rightarrow \stub_j)\wedge\neg(\stub_j \rightarrow \stub_i) \wedge (\flex{}_i \vee \flex{}_j)\wedge\neg(\flex{}_i \rightarrow \flex{}_j)\wedge\neg(\flex{}_j \rightarrow \flex{}_i)\wedge(\flex{}_i \vee \stub_j)\wedge \neg(\flex{}_i \rightarrow \stub_j)\wedge\neg(\stub_j \rightarrow \flex{}_i)\wedge(\flex{}_j \rightarrow \stub_i)$\\
\end{minipage}\\
\begin{minipage}{3cm}
\ \\
\normalsize{17}\\
\scalebox{0.6}{
\begin{tikzpicture}[level distance=4mm]
\draw (1,0) ellipse (10mm and 5mm);
\draw[fill = gray!60] (1.3,0) ellipse (4mm and 2mm);
\draw[densely dashed] (2,0) ellipse (10mm and 5mm);
\draw[densely dashed, fill = gray!20] (1.4,0) ellipse (2.3mm and 1.4mm);
\end{tikzpicture}
}
\\
\end{minipage}&
\begin{minipage}{8cm}
$(\stub_i \vee \stub_j)\wedge\neg(\stub_i \rightarrow \stub_j)\wedge\neg(\stub_j \rightarrow \stub_i) \wedge (\flex{}_j \rightarrow \flex{}_i)\wedge(\flex{}_i \vee \stub_j)\wedge \neg(\flex{}_i \rightarrow \stub_j)\wedge\neg(\stub_j \rightarrow \flex{}_i)\wedge (\flex{}_j \rightarrow \stub_i)$\\
\end{minipage}\\
\begin{minipage}{3cm}
\ \\
\normalsize{18}\\
\scalebox{0.6}{
\begin{tikzpicture}[level distance=4mm]
\draw[densely dashed, fill = gray!20] (1.7,0) ellipse (4mm and 2mm);
\draw (1,0) ellipse (10mm and 5mm);
\draw[densely dashed] (2,0) ellipse (10mm and 5mm);
\draw[densely dashed] (1.7,0) ellipse (4mm and 2mm);
\draw[fill = gray!60] (1.6,0) ellipse (2.3mm and 1.4mm);
\end{tikzpicture}
}
\\
\end{minipage}&
\begin{minipage}{8cm}
$(\stub_i \vee \stub_j)\wedge\neg(\stub_i \rightarrow \stub_j)\wedge\neg(\stub_j \rightarrow \stub_i) \wedge (\flex{}_i \rightarrow \flex{}_j)\wedge(\flex{}_i \rightarrow \stub_j)\wedge (\flex{}_j \vee \stub_i)\wedge \neg(\flex{}_j \rightarrow \stub_i)\wedge\neg(\stub_i \rightarrow \flex{}_j)$
\end{minipage}\\
\begin{minipage}{3cm}
\ \\
\normalsize{25}\\
\scalebox{0.6}{
\begin{tikzpicture}[level distance=4mm]
\draw (1,0) ellipse (10mm and 5mm);
\draw[fill = gray!60] (1.3,0.1) ellipse (2.2mm and 1.2mm);
\draw[densely dashed] (2,0) ellipse (10mm and 5mm);
\draw[densely dashed, fill = gray!20] (1.7,-0.1) ellipse (2.2mm and 1.2mm);
\end{tikzpicture}
}
\\
\end{minipage} &
\begin{minipage}{8cm}
$(\stub_i \vee \stub_j)\wedge\neg(\stub_i \rightarrow \stub_j)\wedge\neg(\stub_j \rightarrow \stub_i) \wedge \neg(\flex{}_i \wedge \flex{}_j)\wedge(\flex{}_i \rightarrow \stub_j)\wedge(\flex{}_j \rightarrow \stub_i)$
\end{minipage}\\
\begin{minipage}{3cm}
\ \\
\normalsize{28}\\
\scalebox{0.6}{
\begin{tikzpicture}[level distance=4mm]
\draw[densely dashed, fill = gray!20] (1.6,0) ellipse (2.5mm and 1.5mm);
\draw (1,0) ellipse (10mm and 5mm);
\draw[fill = gray!60] (1.4,0) ellipse (2.5mm and 1.5mm);
\draw[densely dashed] (2,0) ellipse (10mm and 5mm);
\draw[densely dashed] (1.6,0) ellipse (2.5mm and 1.5mm);
\end{tikzpicture}
}
\\
\end{minipage}&
\begin{minipage}{8cm}
$(\stub_i \vee \stub_j)\wedge\neg(\stub_i \rightarrow \stub_j)\wedge\neg(\stub_j \rightarrow \stub_i) \wedge (\flex{}_i \vee \flex{}_j)\wedge\neg(\flex{}_i \rightarrow \flex{}_j)\wedge\neg(\flex{}_j \rightarrow \flex{}_i)\wedge(\flex{}_i \rightarrow \stub_j)\wedge(\flex{}_j \rightarrow \stub_i)$
\end{minipage}\\
\begin{minipage}{3cm}
\ \\
\normalsize{29}\\
\scalebox{0.6}{
\begin{tikzpicture}[level distance=4mm]
\draw (1,0) ellipse (10mm and 5mm);
\draw[fill = gray!60] (1.5,0) ellipse (3mm and 2mm);
\draw[densely dashed] (2,0) ellipse (10mm and 5mm);
\draw[densely dashed, fill = gray!20] (1.5,0) ellipse (2.2mm and 1.2mm);
\end{tikzpicture}
}
\\
\end{minipage} &
\begin{minipage}{8cm}
$(\stub_i \vee \stub_j)\wedge\neg(\stub_i \rightarrow \stub_j)\wedge\neg(\stub_j \rightarrow \stub_i) \wedge (\flex{}_j \rightarrow \flex{}_i)\wedge(\flex{}_i \rightarrow \stub_j)\wedge(\flex{}_j \rightarrow \stub_i)$
\end{minipage}\\
\begin{minipage}{3cm}
\ \\
\normalsize{30}\\
\scalebox{0.6}{
\begin{tikzpicture}[level distance=4mm]
\draw (1,0) ellipse (10mm and 5mm);
\draw[densely dashed, fill = gray!20] (1.5,0) ellipse (3mm and 2mm);
\draw[densely dashed] (2,0) ellipse (10mm and 5mm);
\draw[fill = gray!60] (1.5,0) ellipse (2.2mm and 1.2mm);
\end{tikzpicture}
}
\\
\end{minipage} &
\begin{minipage}{8cm}
$(\stub_i \vee \stub_j)\wedge\neg(\stub_i \rightarrow \stub_j)\wedge\neg(\stub_j \rightarrow \stub_i) \wedge (\flex{}_i \rightarrow \flex{}_j)\wedge(\flex{}_i \rightarrow \stub_j)\wedge(\flex{}_j \rightarrow \stub_i)$
\end{minipage}\\
\begin{minipage}{3cm}
\ \\
\normalsize{39}\\
\scalebox{0.6}{
\begin{tikzpicture}[level distance=4mm]
\draw (1,0) ellipse (10mm and 5mm);
\draw (1.5,0) ellipse (3mm and 2mm);
\draw[densely dashed] (2,0) ellipse (10mm and 5mm);
\draw[densely dashed, fill = gray!20] (1.5,0) ellipse (2.9mm and 1.9mm);
\end{tikzpicture}
}
\\
\end{minipage} &
\begin{minipage}{8cm}
$(\stub_i \vee \stub_j)\wedge\neg(\stub_i \rightarrow \stub_j)\wedge\neg(\stub_j \rightarrow \stub_i) \wedge (\flex{}_i \leftrightarrow \flex{}_j)\wedge(\flex{}_i \rightarrow \stub_j)\wedge (\flex{}_j \rightarrow \stub_i)$
\end{minipage}\\
\caption[Egg/yolk configurations for consistent stubbornness sets]{Configurations for consistent stubbornness sets. The stubbornness knowledge of the agent $i$, identified by plain lines, is only consistent by the stubbornness knowledge of the agent $j$, identified by dashed lines. }\label{confPO}
\end{longtable}
\end{center}

Figure \ref{grafoPO} depicts the graph of the possible negotiation
relations of the agents during the negotiation. The nodes are the
EGG/YOLK configurations and the edges are colored by the agent who makes
the next bid. The gray node identifies the positive outcome of the
negotiation.

Both agents may make legitimate or violation actions, thus they may use or not the violation rules in Table \ref{tab:1-1net2}.

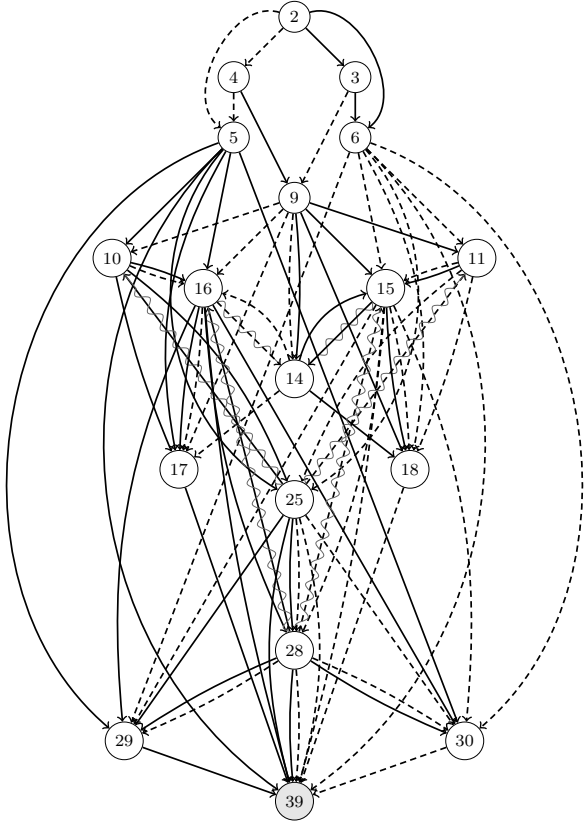
\begin{figure}\centering
\scalebox{0.8}{
\begin{tikzpicture}[level distance=4mm]
\draw (7,15) node[draw, shape = circle] (2) {2};
\draw (8,14) node[draw, shape = circle] (3) {3};
\draw (6,14) node[draw, shape = circle] (4) {4};
\draw (6,13) node[draw, shape = circle] (5) {5};
\draw (8,13) node[draw, shape = circle] (6) {6};
\draw (7,12) node[draw, shape = circle] (9) {9};
\draw (4,11) node[draw, shape = circle] (10) {10};
\draw (10,11) node[draw, shape = circle] (11) {11};
\draw (8.5,10.5) node[draw, shape = circle] (15) {15};
\draw (5.5,10.5) node[draw, shape = circle] (16) {16};
\draw (7,9) node[draw, shape = circle] (14) {14};
\draw (5.1,7.5) node[draw, shape = circle] (17) {17};
\draw (8.9,7.5) node[draw, shape = circle] (18) {18};
\draw (7,7) node[draw, shape = circle] (25) {25};
\draw (7,4.5) node[draw, shape = circle] (28) {28};
\draw (4.2,3) node[draw, shape = circle] (29) {29};
\draw (9.8,3) node[draw, shape = circle] (30) {30};
\draw (7,2) node[draw, shape = circle,fill = black!10] (39) {39};
\draw[->,thick] (2) -- (3);
\draw[->,thick,densely dashed] (2) -- (4);
\path[->,thick, densely dashed] (2) edge [bend right=80] node {} (5);
\path[->,thick] (2) edge [bend left=80] node {} (6);
\draw[->,thick,densely dashed] (4) -- (5);
\draw[->,thick] (3) -- (6);
\draw[->,thick] (4) -- (9);
\draw[->,thick,densely dashed] (3) -- (9);
\draw[->,thick] (5) -- (10);
\draw[->,thick,densely dashed] (9) -- (10);
\draw[->,thick] (9) -- (11);
\draw[->,thick,densely dashed] (6) -- (11);
\path[->,thick, densely dashed] (10) edge [bend right =5] node {} (16);
\path[->,thick] (10) edge [bend left =5] node {} (16);
\path[->,thick, densely dashed] (11) edge [bend right =5] node {} (15);
\path[->,thick] (11) edge [bend left =5] node {} (15);
\draw[snake=coil,segment aspect=0, black!60] (15) -- (14);
\draw[snake=coil,segment aspect=0, black!60] (16) -- (14);
\draw[->,thick] (15) -- (14);
\draw[->,thick,densely dashed] (16) -- (14);
\path[->,thick, densely dashed] (14) edge [bend right=30] node {} (16);
\path[->,thick] (14) edge [bend left=30] node {} (15);
\path[->,thick] (5) edge [bend right=50] node {} (25);
\path[->,thick] (5) edge [bend right=60] node {} (29);
\path[->,thick] (5) edge node {} (30);
\path[->,thick] (5) edge node {} (16);
\path[->,thick] (5) edge [bend right = 20] node {} (17);
\path[->,thick] (5) edge [bend right=50] node {} (39);
\path[->,thick, densely dashed] (6) edge [bend left=50] node {} (25);
\path[->,thick, densely dashed] (6) edge [bend left=60] node {} (30);
\path[->,thick, densely dashed] (6) edge node {} (29);
\path[->,thick, densely dashed] (6) edge node {} (15);
\path[->,thick, densely dashed] (6) edge [bend left = 20] node {} (18);
\path[->,thick, densely dashed] (6) edge [bend left=50] node {} (39);
\path[->,thick, densely dashed] (9) edge node {} (16);
\path[->,thick] (9) edge node {} (15);
\path[->,thick, densely dashed] (9) edge node {} (17);
\path[->,thick] (9) edge node {} (18);
\path[->,thick, densely dashed] (9) edge [bend right =5] node {} (14);
\path[->,thick] (9) edge [bend left =5] node {} (14);
\path[->,thick] (10) edge [bend right =5] node {} (17);
\path[->,thick] (10) edge [bend left = 15] node {} (25);
\path[->,thick] (25) edge  node {} (10);
\draw[snake=coil,segment aspect=0, black!60] (25) -- (10);
\path[->,thick, densely dashed] (11) edge [bend left =5] node {} (18);
\path[->,thick, densely dashed] (11) edge [bend right = 15] node {} (25);
\path[->,thick, densely dashed] (25) edge  node {} (11);
\draw[snake=coil,segment aspect=0, black!60] (25) -- (11);
\path[->,thick] (17) edge node {} (39);
\path[->,thick, densely dashed] (18) edge node {} (39);
\path[->,thick] (14) edge node {} (18);
\path[->,thick, densely dashed] (14) edge node {} (17);
\path[->,thick] (16) edge node {} (30);
\path[->,thick, densely dashed] (15) edge node {} (29);
\path[->,thick] (16) edge [bend right = 10] node {} (28);
\path[->,thick] (28) edge node {} (16);
\draw[snake=coil,segment aspect=0, black!60] (28) -- (16);
\path[->,thick] (16) edge [bend right = 15] node {} (29);
\path[->,thick] (16) edge [bend right = 5] node {} (39);
\path[->,thick, densely dashed] (16) edge [bend left = 5] node {} (17);
\path[->,thick] (16) edge [bend right = 5] node {} (17);
\path[->,thick, densely dashed] (15) edge [bend left = 10] node {} (28);
\path[->,thick, densely dashed] (28) edge node {} (15);
\draw[snake=coil,segment aspect=0, black!60] (28) -- (15);
\path[->,thick, densely dashed] (15) edge [bend left = 15] node {} (30);
\path[->,thick, densely dashed] (15) edge [bend left = 5] node {} (39);
\path[->,thick, densely dashed] (15) edge [bend left = 5] node {} (18);
\path[->,thick] (15) edge [bend right = 5] node {} (18);
\path[->,thick, densely dashed] (25) edge [bend left = 5] node {} (28);
\path[->,thick] (25) edge [bend right = 5] node {} (28);
\path[->,thick, densely dashed] (28) edge [bend left = 5] node {} (39);
\path[->,thick] (28) edge [bend right = 5] node {} (39);
\path[->,thick, densely dashed] (28) edge [bend left = 5] node {} (29);
\path[->,thick] (28) edge [bend right = 5] node {} (29);
\path[->,thick, densely dashed] (28) edge [bend left = 5] node {} (30);
\path[->,thick] (28) edge [bend right = 5] node {} (30);
\path[->,thick, densely dashed] (30) edge  node {} (39);
\path[->,thick] (29) edge  node {} (39);
\path[->,thick, densely dashed] (25) edge  node {} (30);
\path[->,thick] (25) edge  node {} (29);
\path[->,thick, densely dashed] (25) edge [bend left = 15] node {} (39);
\path[->,thick] (25) edge [bend right = 15] node {} (39);
\end{tikzpicture}
}
\caption[Transition graph for consistent and not generalized/restricted stubbornness knowledge]{Transition graph for consistent and not generalized/restricted stubbornness knowledge: the stubbornness knowledge of agent $i$, identified by plain lines, is not a restriction of the stubbornness knowledge of agent $j$, identified by dashed lines, and vice versa but they have shared semantical structures. The nodes are coloured: the gray node is the configuration of the positive outcome of the negotiation process.}\label{grafoPO}
\end{figure}

In the following example, we show how deductive rules of MND are used and their effects in the EGG/YOLK configurations when the stubbornness knowledge of agents are consistent and no generalization or restriction relation exist between them.

\begin{example}\label{es:POS1S2}
Suppose Alice and Bob are related as in configuration 2. Alice, $A$, is the first bidding agent and she proposes $\flex{0}_{A}$ to Bob, $B$. Bob receives the proposal and evaluates it.
Bob tests that they are in absolute disagreement. Bob generalizes his initial viewpoint $\flex{0}_B$ by:
$$
\infer[(W)]{\flex{1}_B}
{
\flex{0}_B\rightarrow\flex{1}_B &
\neg( \stub_B \leftrightarrow \flex{0}_B)
}
$$
and he checks the provisional negotiation situation by:
$$
\infer[(AD)]{B:\mathbf{absDis}(A:\flex{0}_A)\cap B:\flex{1}_B}
{A:\flex{0}_A & \neg(\stub_B \wedge \trad{B}{A}(\flex{0}_A))}
$$
Bob says to Alice that they are in absolute disagreement and makes a proposal $\flex{1}_B$.

The system continues the MN by:
$$
\infer[(N)]{\mathit{Negotiate}(A,B)}{\ast(A,B) & A:\flex{0}_A & B:\mathbf{absDis}(A:\flex{0}_A) & B:\flex{1}_B}
$$

Alice receives $\flex{1}_B$ and she makes a weakening or a changing action because Bob said they are not in agreement nor in relative disagreement.
Suppose Alice changes her CAF by:
$$
\infer[(C)]{\flex{1}_A}{\flex{0}_A & \neg( \stub_A \leftrightarrow \flex{0}_A) & \neg(\flex{0}_A\rightarrow\flex{1}_A) & \neg(\flex{1}_A\rightarrow\flex{0}_A)}
$$
Alice tests the negotiation relation by:
$$
\infer[(AD\text{-}ED)]{A:\mathbf{essDis}(B:\flex{1}_B)\cap A:\flex{1}_A}
{\begin{array}{c}B:\mathbf{absDis}(A:\flex{0}_A)\cap B:\flex{1}_B \\ (\stub_A \vee \trad{A}{B}(\flex{1}_B))\wedge\neg(\flex{1}_A \wedge \trad{A}{B}(\flex{1}_B))\end{array}}
$$
Alice says to Bob that they are in essence disagreement and makes a proposal $\flex{1}_A$.

The system continues the MN by:
$$
\infer[(N)]{\mathit{Negotiate}(B,A)}{\ast(B,A) & B:\flex{1}_B & A:\mathbf{essDis}(B:\flex{1}_B) & A:\flex{1}_A}
$$

Bob receives $\flex{1}_A$ and he makes a weakening or a changing action because Alice said they are not in agreement nor in relative disagreement.
Suppose Bob changes her CAF by:
$$
\infer[(C)]{\flex{2}_B}{\flex{1}_B & \neg( \stub_B \leftrightarrow \flex{1}_B) & \neg(\flex{1}_B\rightarrow\flex{2}_B) & \neg(\flex{2}_B\rightarrow\flex{1}_B)}
$$
Bob tests the negotiation relation by:
$$
\infer[(ED\text{-}ED)]{B:\mathbf{essDis}(A:\flex{1}_A)\cap B:\flex{2}_A}
{\begin{array}{c}A:\mathbf{essDis}(B:\flex{1}_B)\cap A:\flex{1}_A \\ (\stub_B \vee \trad{B}{A}(\flex{1}_A))\wedge\neg(\flex{2}_B \wedge \trad{B}{A}(\flex{1}_A))\end{array}}
$$
Bob says to Alice that they are in essence disagreement and makes a proposal $\flex{2}_B$.

The system continues the MN by:
$$
\infer[(N)]{\mathit{Negotiate}(A,B)}{\ast(A,B) & A:\flex{1}_A & B:\mathbf{essDis}(A:\flex{1}_A) & B:\flex{2}_B}
$$

Alice receives $\flex{2}_B$ and she makes a weakening or a changing action because Bob said they are not in agreement nor in relative disagreement.
Suppose Alice generalizes her CAF by:
$$
\infer[(W)]{\flex{2}_A}
{
\flex{1}_A\rightarrow\flex{2}_A &
\neg( \stub_A \leftrightarrow \flex{1}_A)
}
$$
Alice tests the negotiation relation by:
$$
\infer[(ED\text{-}Co)]{A:\mathbf{comp}(B:\flex{2}_B)\cap A:\flex{2}_A}
{\begin{array}{c}B:\mathbf{essDis}(A:\flex{1}_A)\cap B:\flex{2}_B \\ (\flex{2}_A\vee \trad{A}{B}(\flex{2}_B))\wedge\neg(\flex{2}_A\rightarrow \trad{A}{B}(\flex{2}_B))\wedge\neg(\flex{2}_A\leftarrow \trad{A}{B}(\flex{2}_B))\end{array}}
$$
Alice says to Bob that they are in compatibility and makes a proposal $\flex{2}_A$.

$$
\infer[(N)]{\mathit{Negotiate}(B,A)}{\ast(B,A) & B:\flex{2}_B & A:\mathbf{comp}(B:\flex{2}_B) & A:\flex{2}_A}
$$

Bob receives $\flex{2}_A$ and he makes a weakening or a changing action because Alice said they are not in agreement nor in relative disagreement.
Suppose Bob changes her CAF by:
$$
\infer[(C)]{\flex{3}_B}{\flex{2}_B & \neg( \stub_B \leftrightarrow \flex{2}_B) & \neg(\flex{2}_B\rightarrow\flex{3}_B) & \neg(\flex{3}_B\rightarrow\flex{2}_B)}
$$
Bob tests the negotiation relation by:
$$
\infer[(Co\text{-}RD)]{B:\mathbf{relDis}(A:\flex{2}_A)\cap B:\flex{3}_B}
{\begin{array}{c}A:\mathbf{comp}(B:\flex{2}_B)\cap A:\flex{2}_A \\  (\flex{3}_B \rightarrow \trad{B}{A}(\flex{2}_A))\wedge\neg(\flex{3}_B \leftarrow \trad{B}{A}(\flex{2}_A))\end{array}}
$$
Bob says to Alice that they are in relative disagreement and makes a proposal $\flex{3}_B$.

The system continues the MN by:
$$
\infer[(N)]{\mathit{Negotiate}(A,B)}{\ast(A,B) & A:\flex{2}_A & B:\mathbf{relDis}(A:\flex{2}_A) & B:\flex{3}_B}
$$

Alice receives $\flex{2}_B$ and she cannot\fix{Luca}{``has not'' non \`e chiaro. Meglio ``cannot''?} to make a weakening or a changing action because Bob said they are in relative disagreement.
Alice accepts the proposal of Bob by:
$$
\infer[(RD\text{-}Ag)]{A:\mathbf{agree}(B:\flex{3}_B) \cap A:\trad{A}{B}(\flex{2}_A)}
{B:\mathbf{relDis}(A:\flex{2}_A)\cap B:\flex{3}_B}
$$
Alice says to Bob that they are in agreement and that they have a common angle that is $\flex{3}_B$.

The system closes the MN by:
$$
\infer[(A)]{\mathit{Agreement}(B,A)}{\ast(B,A) & B:\flex{3}_B & A:\mathbf{agree}(B:\flex{3}_B)}
$$
with a positive outcome, $\flex{3}_B$.

In Figure~\ref{es:PO} we show the message passing flow between Alice and Bob 
and the changes of the EGG/YOLK configurations. 

\begin{figure}
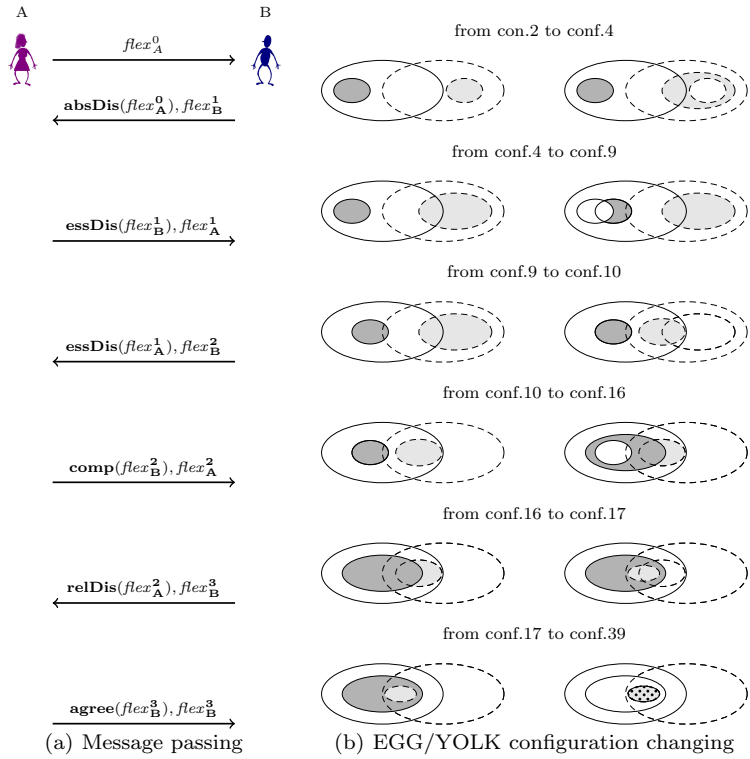
\centering
\subfigure[Message passing]{
\label{es:PO-mp}
\scalebox{0.8}{
\begin{tikzpicture}[level distance=4mm]\scriptsize
\draw (0,22.8) node (alice) {A};
\draw (0,22) node (a) {\includegraphics[width=0.5cm]{alice.jpg}};
\draw (4,22.8) node (bob) {B};
\draw (4,22) node (b) {\includegraphics[width=0.5cm]{bob.jpg}};
\draw[->,thick] (0.5,22) .. controls +(right:0.5cm) and +(left:0.5cm).. node[above,sloped] {$\flex{0}_A$} (3.5,22);
\draw[<-,thick] (0.5,21) .. controls +(right:0.5cm) and +(left:0.5cm).. node[above,sloped] {$\mathbf{absDis(\flex{0}_A),\flex{1}_B}$} (3.5,21);
\draw[->,thick] (0.5,19) .. controls +(right:0.5cm) and +(left:0.5cm).. node[above,sloped] {$\mathbf{essDis(\flex{1}_B),\flex{1}_A}$} (3.5,19);
\draw[<-,thick] (0.5,17) .. controls +(right:0.5cm) and +(left:0.5cm).. node[above,sloped] {$\mathbf{essDis(\flex{1}_A),\flex{2}_B}$} (3.5,17);
\draw[->,thick] (0.5,15) .. controls +(right:0.5cm) and +(left:0.5cm).. node[above,sloped] {$\mathbf{comp(\flex{2}_B),\flex{2}_A}$} (3.5,15);
\draw[<-,thick] (0.5,13) .. controls +(right:0.5cm) and +(left:0.5cm).. node[above,sloped] {$\mathbf{relDis(\flex{2}_A),\flex{3}_B}$} (3.5,13);
\draw[->,thick] (0.5,11) .. controls +(right:0.5cm) and +(left:0.5cm).. node[above,sloped] {$\mathbf{agree(\flex{3}_B),\flex{3}_B}$} (3.5,11);
\end{tikzpicture}
}
}
\subfigure[EGG/YOLK configuration changing]{
\label{es:PO-ey}
\scalebox{0.8}{
\begin{tikzpicture}[level distance=4mm]
\draw (3.5,21) node (2-4) {from con.2 to conf.4};
\draw (1,20) ellipse (10mm and 5mm);
\draw[fill = gray!60] (0.5,20) ellipse (3mm and 2mm);
\draw[densely dashed] (2,20) ellipse (10mm and 5mm);
\draw[densely dashed, fill = gray!20] (2.35,20) ellipse (3mm and 2mm);
\draw[densely dashed, fill = gray!20] (6.2,20) ellipse (6mm and 3mm);
\draw (5,20) ellipse (10mm and 5mm);
\draw[fill = gray!60] (4.5,20) ellipse (3mm and 2mm);
\draw[densely dashed] (6,20) ellipse (10mm and 5mm);
\draw[densely dashed, fill = white] (6.35,20) ellipse (3mm and 2mm);
\draw (3.5,19) node (4-9) {from conf.4 to conf.9};
\draw[densely dashed, fill = gray!20] (2.2,18) ellipse (6mm and 3mm);
\draw (1,18) ellipse (10mm and 5mm);
\draw[fill = gray!60] (0.5,18) ellipse (3mm and 2mm);
\draw[densely dashed] (2,18) ellipse (10mm and 5mm);
\draw[densely dashed, fill = gray!20] (6.2,18) ellipse (6mm and 3mm);
\draw (5,18) ellipse (10mm and 5mm);
\draw[fill = gray!60] (4.8,18) ellipse (3mm and 2mm);
\draw[densely dashed] (6,18) ellipse (10mm and 5mm);
\draw[fill = white] (4.5,18) ellipse (3mm and 2mm);
\draw (4.8,18) ellipse (3mm and 2mm);
\draw (3.5,17) node (9-10) {from conf.9 to conf.10};
\draw[densely dashed, fill = gray!20] (2.2,16) ellipse (6mm and 3mm);
\draw (1,16) ellipse (10mm and 5mm);
\draw[fill = gray!60] (0.8,16) ellipse (3mm and 2mm);
\draw[densely dashed] (2,16) ellipse (10mm and 5mm);
\draw[densely dashed, fill = white] (6.2,16) ellipse (6mm and 3mm);
\draw (5,16) ellipse (10mm and 5mm);
\draw[fill = gray!60] (4.8,16) ellipse (3mm and 2mm);
\draw[densely dashed] (6,16) ellipse (10mm and 5mm);
\draw[densely dashed, fill = gray!20] (5.6,16) ellipse (3.8mm and 2.2mm);
\draw (4.8,16) ellipse (3mm and 2mm);
\draw[densely dashed] (6.2,16) ellipse (6mm and 3mm);
\draw (3.5,15) node (10-14) {from conf.10 to conf.16};
\draw (1,14) ellipse (10mm and 5mm);
\draw[fill = gray!60] (0.8,14) ellipse (3mm and 2mm);
\draw[densely dashed] (2,14) ellipse (10mm and 5mm);
\draw[densely dashed, fill = gray!20] (1.6,14) ellipse (3.8mm and 2.2mm);
\draw (0.8,14) ellipse (3mm and 2mm);
\draw (5,14) ellipse (10mm and 5mm);
\draw[densely dashed] (6,14) ellipse (10mm and 5mm);
\draw[densely dashed, fill = gray!20] (5.6,14) ellipse (3.8mm and 2.2mm);
\draw[fill = gray!60] (5,14) ellipse (6.6mm and 3mm);
\draw[fill = white] (4.8,14) ellipse (3mm and 2mm);
\draw[densely dashed] (6,14) ellipse (10mm and 5mm);
\draw[densely dashed] (5.6,14) ellipse (3.8mm and 2.2mm);
\draw (3.5,13) node (16-17) {from conf.16 to conf.17};
\draw (1,12) ellipse (10mm and 5mm);
\draw[densely dashed] (2,12) ellipse (10mm and 5mm);
\draw[densely dashed, fill = gray!20] (1.6,12) ellipse (3.8mm and 2.2mm);
\draw[fill = gray!60] (1,12) ellipse (6.6mm and 3mm);
\draw[densely dashed] (2,12) ellipse (10mm and 5mm);
\draw[densely dashed] (1.6,12) ellipse (3.8mm and 2.2mm);
\draw (5,12) ellipse (10mm and 5mm);
\draw[densely dashed] (6,12) ellipse (10mm and 5mm);
\draw[densely dashed, fill = white] (5.6,12) ellipse (3.8mm and 2.2mm);
\draw[fill = gray!60] (5,12) ellipse (6.6mm and 3mm);
\draw[densely dashed, fill = gray!20] (5.3,12) ellipse (2.6mm and 1.3mm);
\draw[densely dashed] (6,12) ellipse (10mm and 5mm);
\draw[densely dashed] (5.6,12) ellipse (3.8mm and 2.2mm);
\draw (3.5,11) node (17-39) {from conf.17 to conf.39};
\draw (1,10) ellipse (10mm and 5mm);
\draw[densely dashed] (2,10) ellipse (10mm and 5mm);
\draw[fill = gray!60] (1,10) ellipse (6.6mm and 3mm);
\draw[densely dashed, fill = gray!20] (1.3,10) ellipse (2.6mm and 1.3mm);
\draw[densely dashed] (2,10) ellipse (10mm and 5mm);
\draw (5,10) ellipse (10mm and 5mm);
\draw[densely dashed] (6,10) ellipse (10mm and 5mm);
\draw[fill = white] (5,10) ellipse (6.6mm and 3mm);
\draw[densely dashed, fill = gray!20] (5.3,10) ellipse (2.6mm and 1.3mm);
\draw[densely dashed] (6,10) ellipse (10mm and 5mm);
\draw[pattern = crosshatch dots] (5.3,10) ellipse (2.6mm and 1.3mm);
\draw[densely dashed] (5.3,10) ellipse (2.6mm and 1.3mm);
\end{tikzpicture}
}
}
\caption[A MN scenario between agents with consistent stubbornness knowledge]{A MN scenario between Alice and Bob with consistent stubbornness knowledge: the message passing flow (a) and the changes of their CAFs (b). White yolks represent the precedent proposal of the agent and the dotted gray yolk is the positive outcome of the scenario.\label{es:PO}}
\end{figure}

The MN results in a path, showed in Figure \ref{fig:grafoPOes}, from node 2 to node 39 of the graph in Figure \ref{grafoPO}.
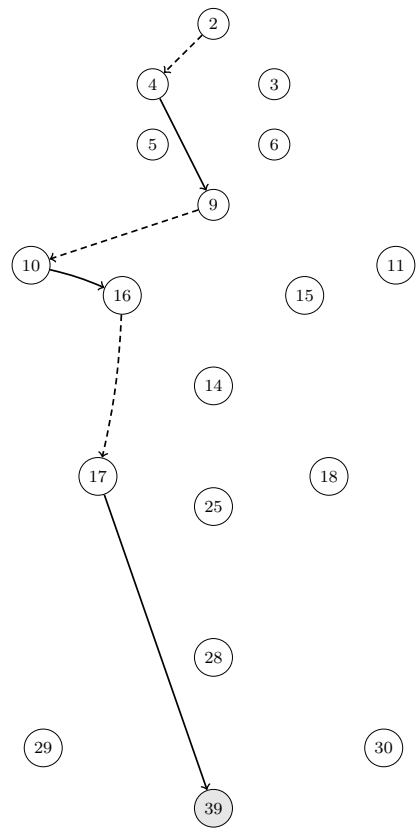
\begin{figure}\centering
\scalebox{0.8}{
\begin{tikzpicture}[level distance=4mm]
\draw (7,15) node[draw, shape = circle] (2) {2};
\draw (8,14) node[draw, shape = circle] (3) {3};
\draw (6,14) node[draw, shape = circle] (4) {4};
\draw (6,13) node[draw, shape = circle] (5) {5};
\draw (8,13) node[draw, shape = circle] (6) {6};
\draw (7,12) node[draw, shape = circle] (9) {9};
\draw (4,11) node[draw, shape = circle] (10) {10};
\draw (10,11) node[draw, shape = circle] (11) {11};
\draw (8.5,10.5) node[draw, shape = circle] (15) {15};
\draw (5.5,10.5) node[draw, shape = circle] (16) {16};
\draw (7,9) node[draw, shape = circle] (14) {14};
\draw (5.1,7.5) node[draw, shape = circle] (17) {17};
\draw (8.9,7.5) node[draw, shape = circle] (18) {18};
\draw (7,7) node[draw, shape = circle] (25) {25};
\draw (7,4.5) node[draw, shape = circle] (28) {28};
\draw (4.2,3) node[draw, shape = circle] (29) {29};
\draw (9.8,3) node[draw, shape = circle] (30) {30};
\draw (7,2) node[draw, shape = circle,fill = black!10] (39) {39};
\draw[->,thick,densely dashed] (2) -- (4);
\draw[->,thick] (4) -- (9);
\draw[->,thick,densely dashed] (9) -- (10);
\path[->,thick] (10) edge [bend left =5] node {} (16);
\path[->,thick] (17) edge node {} (39);
\path[->,thick, densely dashed] (16) edge [bend left = 5] node {} (17);
\end{tikzpicture}
}
\caption[The MN path of the message passing in Figure \ref{es:PO}]{The MN path of the Alice and Bob message passing in Figure \ref{es:PO}.}\label{fig:grafoPOes}
\end{figure}
\hfill$\Box$
\end{example}

\section{Related Work}\label{sec:relWork}

\fix{Luca}{Questa section \`e da riscrivere completamente. Ora come ora,
\`e solo un elenco di related work, senza alcun comparison. Va riscritta
partendo dal nostro lavoro (magari facendo merge con la sezione di
conclusioni), confrontandolo e giustificandolo rispetto al related work.
Molte citazioni, probabilmente, spariranno, ma va bene cos\`i.}

The Meaning Negotiation problem has reached large attention in the Artificial Intelligence community. Two are the most general approaches to the problem of finding a shared knowledge from many different and possibly inconsistent ones. The first way to model the MN process is by viewing it as a conflict resolution. The participants of a negotiation litigate about how to share something and they may disagree in many ways by Hunter and Summerton (\citeyear{hunter06}).


%
Argumentation theory, or argumentation, is the interdisciplinary study of how humans should, can, and do reach conclusions through logical reasoning, that is, claims based, soundly or not, on premises. It includes the arts and sciences of civil debate, dialogue, conversation, and persuasion. It studies rules of inference, logic, and procedural rules in both artificial and real world settings.

Argumentation includes debate and negotiation which are concerned with reaching mutually acceptable conclusions \citep{kraus,parsons,atkinson05,schroeder02}. It also encompasses eristic dialog, the branch of social debate in which victory over an opponent is the primary goal. This art and science is often the means by which people protect their beliefs or self-interests in rational dialogue, in common parlance, and during the process of arguing.
%
%
The main approaches to the Argumentation theory are: the pragma-dialectical theory and the argumentative schemes.
%
%

In pragma-dialectical theory, the argumentation is viewed as a critical discussion about the resolution of a conflicts. In this ideal model of a critical discussion, four discussion stages are distinguished that the discussion parties have to go through to resolve their difference of opinion (see \citep{emeren84} pp.85-88; \citep{emeren92}, pp.34-35; \citep{emeren04}, pp.59-62):
\begin{enumerate}
\item the confrontation stage: the interlocutors establish that they have a difference of opinion;
\item opening stage: they decide to resolve this difference of opinion. The interlocutors determine their points of departure: they agree upon the rules of the discussion and establish which propositions they can use in their argumentation;
\item argumentation stage: the protagonist defends his/her standpoint by putting forward arguments to counter the antagonist's objections or doubt;
\item concluding stage: the discussion parties evaluate to what extent their initial difference of opinion has been resolved and in whose favor.
\end{enumerate}

The ideal model stipulates ten rules (see \citep{emeren02}, pp.182-183) that apply to an argumentative discussion. Violations of the discussion rules are said to frustrate the reasonable resolution of the difference of opinion and they are therefore considered as fallacies.

The representation of \emph{Argumentative schemes} constitutes one of the central topics in current argumentation theory and they represent common patterns of reasoning used in everyday
conversational discourse. Important contributions to the study of argument schemes have been made by Douglas Walton \citep{walton96book,prakken03,walton05book,walton08,prakken03}. As considered by him, argument schemes technically have the form of an inference rule: an argument scheme has a set of premises and a conclusion.

The argumentation schemes approach is based upon the Toulmin model of the argumentation \citep{toulmin03}. 

The process of resolving conflicts between agents by argumentation involves not only a negotiation dialogue, but also a \emph{persuasion} one \citep{krabbe}. The participants in a negotiation by argumentation propose arguments to the opponents and make counterproposals in two way: by rebutting and or by undercutting the proposals of the opponents. Rebuttal of a rule claiming $c$, is made by a rule in which the claim is the negation of $c$. A rule $r$ undercuts a rule $r'$ if the claim of $r$ is the negation of some of the premises of $r'$.

When no undercut and rebuttal rules are available, an agent can accept the argument posted by someone else in the system in two ways \citep{dung07}:
\begin{itemize}
\item \emph{skeptical}: the argument is acceptable until somebody else claims the contrary;
\item \emph{credulous}: the argument is wholeheartedly accepted.
\end{itemize}
In \citep{dung95} the author explores the mechanisms humans use in argumentation to state the correctness, the appropriateness and the acceptability of arguments.

To persuade the opponents about the validity of the argument she proposes, the proponent has to \emph{justify} it \citep{pollock94,pollock01,waltonjust05,rubinelli06,atkinson04CMNA,governatori07ACAI} or to have its proof. Recent investigations have dealt with the problem about who has the burden of proving a claim and which argument produces a burden of proof \citep{farley95,walton03,prakken05,oren07,gordon07}. In \citep{chesnevar00} a complete survey of the logical models of arguments is presented.

Argumentation Theory is largely used in legal reasoning to model the interactions according to the legal debate rules \citep{sergot97,gordon09icail,bench97,kowalski96}.
In particular, in \citep{bench05}, the authors formalise an argumentation framework in order to model the definitions of \emph{objectively} and \emph{subjectively acceptable}, and \emph{indefensible} argument. The definition of the above degrees of acceptance of an argument is based upon a value given to the arguments and a form of preference between them that the agents have.


In \citep{maudet06}, the authors present a brief survey of argumentation in
multi-agent systems. It is not only brief, but rather idiosyncratic, and
focuses on the areas of research of belief revision, agent communication and reasoning.\\ \ \\

The second way to model MN is as a set of operations on the beliefs' sets of the agents involved. The scope is to construct a commonly accepted knowledge 
%
as the process of merging information becoming from different sources. The problem of how the merging has to be done was approached in two steps:
\begin{itemize}
\item how the different sources have inconsistent beliefs and how they are mutually reliable;
\item how and when beliefs causing conflicts have to be merged into the knowledge base.
\end{itemize}
The first point was studied by the \emph{information fusion} researchers and the second by the \emph{belief revision} ones.

In \citep{greg06} the author makes a survey of the contributions from the artificial intelligence research literature about logic-based information fusion. The assumption made by the early approaches were:
\begin{itemize}
\item Information sources are mutually independent;
\item All sources exhibit the same level of importance;
\item The level of information importance is also constant.
\end{itemize}
The main assumption regards the completely reliance of all the information sources as in \citep{booth06}.
More realistic approaches suppose that the information sources are not equally reliable and that some source is preferred with respect to the available ones. In \citep{greg06IFW} the reliability of the information sources is defined as a preference order. Another precedent approach assume a weight applied to the beliefs for each source by which they belong \citep{lin96}.

In the situations in which the information sources are equally reliable, the merging is said \emph{non-prioritized} otherwise a degree of certainty or plausibility is given to the belief \citep{ferme04}.

When the beliefs coming from the different sources, they have to be merged in order to \emph{minimally change} the initial knowledge base. The operation needed to add new information into a knowledge base is known as \emph{revision} and it involves only conflicting beliefs during a negotiation process. The general approach of \emph{maximal adjustment} is to remove the present belief causing the conflict and adding the new one. In \citep{benferhat04} the author present a \emph{disjunctive maximal adjustment} in which the belief are weighted and thus not always removed or simply added into the knowledge base.

The merging\footnote{One can be tempted to assume that arbitration and majority operators can be fruitfully employed to solve any admissible problem of negotiation. However, negotiation is the process of reaching agreements not the underlying semantic theory about the models. Therefore, although we can model the resulting theory by the theory of belief revision, negotiation processes are out of scope in these theories.} of beliefs was defined by two operators \citep{liberatore98}: \emph{majority} and \emph{arbitration}. Both make assumptions upon the information sources. The former revises the knowledge base by belief belonging to the majority number of information sources. The latter revises the knowledge bases by the beliefs belonging to the most reliable information sources.

In \citep{konieczny00KR} the author defines the postulates regulating the merging operators by assuming that there are \emph{integrity constraints} to assure.

Thus, in a belief merging and information fusion literature, the negotiation is modeled as a two stage processes: contraction of the beliefs causing the conflict and expansions by the new knowledge \citep{booth06}. In \citep{zhang04TR} the author define a way to formalize the negotiation process as a function and he proposes a set of postulates, similar to the AGM ones for revision for the negotiation function. 

\section{Conclusions}\label{sec:conclusions}
We presented a formalization of the MN problem by means of a deduction
system. As we remarked in many different places of the paper, the
literature has dealt with several issues of the negotiation of meaning,
but what has been only partially treated is the description of the
process of reaching agreement conditions.

Here, we focused upon the MN problem in terms of knowledge
representation and of automatic mechanism of reaching an agreement.
First, we defined a negotiating agent by two set of knowledge: stubborn
and flexible. The stubborn knowledge of the agent is the unquestionable
one and it represents the necessary properties to define the meaning of
the set of terms the agent is negotiating. Instead, the flexible
knowledge is the representation of the properties that the agent thinks
as not necessary, but can be useful, to define the negotiating terms. A
negotiating agent is willing to cede with respect to non necessary
properties. After the definition of an agent and of her knowledge, we
defined the agreement condition as the situation in which all the agent
or an acceptable part of them agree with the same proposal, i.e. when
the agents consider the proposal as an acceptable common angle.
Otherwise the agents are in disagreement.

We identified four ways in which agents are in disagreement: absolute, essence, relative or compatibility. The disagreement relation is binary because it depends upon the relation between the knowledge of the agents, thus, for instance, Alice may have inconsistent knowledge with respect to the knowledge of Bob (absolute disagreement) and she may have a consistent but not generalised or restricted knowledge with respect to the knowledge of Charles (compatibility).

Afterwards we defined rules for deriving streams of dialog between an arbitrary number of meaning negotiating agents
by assuming that in a multiparty MN the first proposing agent behaves as a referee in an English Auction Game; and
we defined a deduction system, $\MND$, based upon these
rules, which derives a stream of dialog that ends with an agreement (or
disagreement) condition.

There are several different ways in which this investigation can be taken further,
in particular by investigating the formal properties of $\MND$, such as
soundness and completeness. The proofs of consistency and adequacy do
not fix the relation to a given semantics, which is needed for a proof
of soundness and a proof of completeness. Usually, a deduction system
can be proved sound and complete against a standard interpretation of
the language, which is difficult to circumscribe in our case, because of
the presence of the relations between agents to be represented. A
standard definition of the semantics for the $\MND$ systems is therefore
needed in front of any further investigation of the soundness and
completeness properties.

We deliberately avoided to investigate the formal logical properties of
the system at this stage, for the sake of clarity and readability. It
shall be argument of another paper.

Our formalization of the MN process may be considered \emph{credulous} in the sense of the Argumentation literature (see Section~\ref{sec:relWork}). In fact, with the rule $(RD\text{-}Ag)$ an agent accepts the proposal $\varphi$ even if it is not equivalent to her current angle: the accepting agent trusts in the proposing agent. As a future work, we will study the properties of credulousness and skepticism of the rules of the deduction system. Moreover, the investigation of the trustworthiness among negotiating agents is interesting because a credulous or a skeptical deduction system may be adopted depending upon the trust relation among agents: an agent may be credulous with respect to a trustworthy agent and skeptical with a non-trustworthy one.

In this paper, we assumed that agents are truthful thus they never
inform the opponents about something wrongly. Fraudulent agents may try
to drive the MN in a way that is in some sense optimal for themselves.
It would be interesting to study the optimality and minimality of the MN
outcomes and the ways, legitimate or not, that the agents use to reach
optimal outcomes.

It would also be interesting to develop a decision making algorithm for
those cases in which the system is decidable, in particular for finite
signatures in addition to the case of competitive agents considered
here. This would foster the automation both of the subjective decision
process (i.e., the automation of the deduction system alone) and of the
whole process per se (i.e., the definition of a procedure to establish
the agreement terminal condition).


The investigation we carried out can also be extended by studying the
ways in which agents can be limited to specific strategies in choosing
the next action. Jointly with the definition of an algorithm for
negotiating a common angle, this study can also enlarge the boundary of
decidable cases. In particular, agents using some specific strategies
can apply the rules in a finite number of steps even if the signature is
infinite.

Finally, we envisage a further extensions of our approach to applications in information security, e.g., investigating the relationships between the MN process and the management of authorization
policies in security protocols and web services.

%
%


\bibliographystyle{spbasic}      
\bibliography{general}   


\end{document}